\documentclass[numbers,webpdf,imaiai]{ima-authoring-template}%

\graphicspath{{Fig/}}

\theoremstyle{thmstyletwo}%

\usepackage{soul}

\newtheorem{theorem}{Theorem}[section]
\newtheorem{proposition}[theorem]{Proposition}%

\newtheorem{lemma}[theorem]{Lemma}
\newtheorem{corollary}[theorem]{Corollary}

\newtheorem{remark}[theorem]{Remark}%

\newtheorem{definition}[theorem]{Definition}

\numberwithin{equation}{section}

\begin{document}

\DOI{DOI HERE}
\copyrightyear{2022}
\vol{00}
\pubyear{2022}
\access{Advance Access Publication Date: Day Month Year}
\appnotes{Paper}
\copyrightstatement{Published by Oxford University Press on behalf of the Institute of Mathematics and its Applications. All rights reserved.}
\firstpage{1}


\title[Non-Asymptotic Analysis of Ensemble Kalman Updates]{Non-Asymptotic Analysis of Ensemble Kalman Updates:\\ Effective Dimension and Localization}

\author{Omar Al Ghattas* and Daniel Sanz-Alonso
\address{\orgdiv{Department of Statistics}, \orgname{The University of Chicago} }}

\authormark{Omar Al Ghattas and Daniel Sanz-Alonso}

\corresp[*]{Corresponding author: \href{ghattas@uchicago.edu}{ghattas@uchicago.edu}}

\received{Date}{0}{Year}
\revised{Date}{0}{Year}
\accepted{Date}{0}{Year}


\abstract{Many modern algorithms for inverse problems and data assimilation rely on ensemble Kalman updates to blend prior predictions with observed data. Ensemble Kalman methods often perform well with a small ensemble size, which is essential in applications where generating each particle is costly. This paper develops a non-asymptotic analysis of ensemble Kalman updates that rigorously explains why a small ensemble size suffices if the prior covariance has moderate effective dimension due to fast spectrum decay or approximate sparsity. We present our theory in a unified framework, comparing several implementations of ensemble Kalman updates that use perturbed observations, square root filtering, and localization. As part of our analysis, we develop new dimension-free covariance estimation bounds for approximately sparse matrices that may be of independent interest.} 


\maketitle

\section{Introduction} \label{sec:Introduction}
Many algorithms for inverse problems and data assimilation rely on ensemble Kalman updates to blend prior predictions with observed data. The main motivation behind ensemble Kalman methods is that they often perform well with a small ensemble size $N$, which is essential in applications where generating each particle is costly. However, theoretical studies have primarily focused on large ensemble asymptotics, that is, on the limit $N\to \infty.$ While these \emph{mean-field} results are mathematically interesting and have led to significant practical improvements, they fail to explain the  empirical success of ensemble Kalman methods when deployed with a small ensemble size. The aim of this paper is to develop a \emph{non-asymptotic} analysis of ensemble Kalman updates that rigorously explains why, and under what circumstances, a small ensemble size may suffice. To that end, we establish non-asymptotic error bounds in terms of suitable notions of effective dimension  of the prior covariance model that account for spectrum decay (which may represent smoothness of a prior random field) and approximate sparsity (which may represent spatial decay of correlations). Our work complements mean-field analyses of ensemble Kalman updates and identifies scenarios where mean-field behavior holds with moderate $N$.

In addition to demystifying the practical success of ensemble Kalman methods with a small ensemble size,  our non-asymptotic perspective allows us to tell apart, on accuracy grounds, implementations of ensemble Kalman updates that use perturbed observations and square root filtering. These implementations become equivalent in the large $N$ limit, and therefore their differences in accuracy cannot be captured by asymptotic results. Furthermore, our non-asymptotic perspective provides new understanding on the importance of localization, a procedure widely used by practitioners that involves tapering or ``localizing'' empirical covariance estimates to avoid spurious correlations.

Rather than providing a complete, definite analysis of any particular ensemble Kalman method, our goal is to bring to bear a new set of tools from high-dimensional probability and statistics to the study of these algorithms. In particular, our work builds on and contributes to the theory of high-dimensional covariance estimation, which we believe is fundamental to the understanding of ensemble Kalman methods. To make the presentation accessible to a wide audience, we assume no background knowledge on covariance estimation or on ensemble Kalman methods. 

\subsection{Problem Description}
Consider the inverse problem of recovering $u \in \R^d$ from data $y \in \R^k,$ corrupted by noise $\eta,$ where
\begin{equation}\label{eq:IP}
y = \mcG(u) + \eta,
\end{equation}
$\mcG: \R^{d} \to \R^k$ is the forward model, and $\eta \sim \P_\eta = \mcN(0, \Gamma)$ is the observation error with positive-definite covariance matrix $\Gamma$. An ensemble Kalman update takes as input a \textit{prior ensemble} $\{ u_n \}_{n=1}^N$ and observed data $y$, and returns as output an \textit{updated ensemble} $\{ \upsilon_n \}_{n=1}^N$ that blends together the information in the prior ensemble and in the data. Two main types of problems will be investigated: posterior approximation and sequential optimization. In the former, ensemble Kalman updates are used to approximate a posterior distribution in a Bayesian linear setting; in the latter, they are used within optimization algorithms for nonlinear inverse problems. 

\subsubsection{Posterior Approximation}
If the forward model is linear, i.e. $\mcG(u) = Au$ for some matrix $A \in \R^{k \times d}$, and $A$ is ill-conditioned or $d \gg k$, naive inversion of the data by means of the (generalized) inverse of $A$ results in an amplification of small observation error $\eta$ into large error in the reconstruction of $u$. In such situations, regularization is needed to stabilize the solution. To this end, one may adopt a Bayesian approach and place a Gaussian prior on the unknown $u \sim \P_u = \mcN(m, C)$ with positive-definite $C$; the prior distribution then acts as a probabilistic regularizer. The Bayesian solution to the inverse problem \eqref{eq:IP} is a full characterization of the posterior distribution $\P_{u | y}$, that is, the distribution of $u$ given $y.$ A standard calculation shows that $\P_{u | y} = \mcN(\mu, \Sigma)$, with 
\begin{align}\label{eq:KFmeancovariance}
\begin{split}
 \mpost  &=  \mpr + \Cpr A^\top (A \Cpr A^\top + \Gamma)^{-1} (y - A\mpr),  \\ 
\Cpost &= \Cpr - \Cpr A^\top(A\Cpr A^\top + \Gamma)^{-1} A \Cpr, 
\end{split}
\end{align}
which require storage of $d \times d$ matrices and consequently are difficult to compute explicitly when the state dimension $d$ is large. 
A posterior-approximation ensemble Kalman update transforms a prior ensemble $\{ u_n \}_{n=1}^N$ drawn from $\P_u$ into an updated ensemble $\{ \upsilon_n \}_{n=1}^N$ whose sample mean and sample covariance approximate the mean and covariance of $\P_{u | y}.$ Ensemble Kalman updates enjoy a low computational and memory cost when the ensemble size $N$ is smaller than the state dimension $d.$  In Section \ref{sec:withoutlocalization} we establish non-asymptotic error bounds that ensure that if $N$ is larger than a suitably defined \emph{effective dimension}, then the sample mean and sample covariance of the updated ensemble approximate well the true posterior mean and covariance in \eqref{eq:KFmeancovariance}. We refer to methods that are capable of approximating well the posterior $\P_{u|y}$ in a linear-Gaussian setting as \textit{posterior-approximation} algorithms. 

\subsubsection{Sequential Optimization}
When faced with a general nonlinear model $\mcG$, exact characterization of the posterior can be challenging. One may then opt for an optimization framework and solve the inverse problem \eqref{eq:IP} by minimizing a user-chosen objective function. Starting from a prior ensemble $\{ u_n \}_{n=1}^N$ drawn from a measure $\P_u$ that encodes prior beliefs about $u$, an ensemble Kalman update returns an updated ensemble $\{ \upsilon_n \}_{n=1}^N$ whose sample mean approximates the desired minimizer. The process can  be iterated by taking the updated ensemble to be the prior ensemble of a new ensemble Kalman update. Under suitable conditions on $\mcG,$ and after a sufficient number of such updates, all particles in the ensemble collapse into the minimizer of the objective. Ensemble Kalman optimization algorithms are derivative-free methods, and are therefore particularly useful when derivatives of the model $\mcG$ are unavailable or expensive to compute. As for posterior-approximation algorithms, implementing each update has low computational and memory cost when the ensemble size $N$ is small. In Section \ref{sec:withlocalization} we will establish non-asymptotic error bounds that ensure that if $N$ is larger than a suitably defined \emph{effective dimension}, then each particle update $u_n \mapsto \upsilon_n,$ $1 \le n \le N,$ approximates well an idealized mean-field update computed with an infinite number of particles; this suggests that the evolution of particles along an ensemble-based sequential optimizer is close to an idealized mean-field evolution. We refer to methods that solve the inverse problem \eqref{eq:IP} by minimization of an objective function as \textit{sequential-optimization} algorithms.

\subsection{Summary of Contributions and Outline}
\begin{itemize}
    \item Section \ref{sec:withoutlocalization} is concerned with posterior-approximation algorithms. The main results, Theorems \ref{lem:meanwithoutlocExpectation} and \ref{lem:covwithoutlocExpectation}, give non-asymptotic bounds on the estimation of the posterior mean and covariance in terms of a standard notion of effective dimension that accounts for spectrum decay in the prior covariance model. Our analysis explains the statistical advantage of square root updates over perturbed observation ones. We also discuss the deterioration of our bounds in small noise limits where the prior and the posterior become mutually singular.
    \item Section \ref{sec:withlocalization} is concerned with sequential-optimization algorithms. The main results, Theorems \ref{thm:EKI} and \ref{thm:LEKI}, give non-asymptotic bounds on the  approximation of mean-field particle updates using ensemble Kalman updates with and without localization. Our analysis explains the advantage of localized updates if the prior covariance satisfies a soft-sparsity condition. For the study of localized updates, we show in Theorems \ref{thm:SoftSparistyCovarianceBound} and \ref{thm:SoftSparistyCrossCovarianceBound} new dimension-free covariance estimation bounds in terms of a new notion of effective dimension that simultaneously accounts for spectrum decay and approximate sparsity in the prior covariance model. 
    \item Section \ref{sec:conclusions} concludes with a summary of our work and several research directions that stem from our non-asymptotic analysis of ensemble Kalman updates. We also discuss the potential and limitations of localization in posterior-approximation algorithms. 
    \item The proofs of all our results are deferred to three appendices.
\end{itemize}

\subsection{Related Work}
Ensemble Kalman methods ---overviewed in \cite{evensen2009data,katzfuss2016understanding,houtekamer2016review,roth2017ensemble,chada2021iterative,sanzstuarttaeb}--- first appeared as filtering algorithms in the data assimilation literature \cite{evensen1994sequential,evensen1996assimilation,burgers1998analysis, houtekamer1995methods,houtekamer1998data}. The goal of data assimilation is to estimate a time-evolving state as new observations become available \cite{reich2015probabilistic,asch2016data,law2015data,majda2012filtering,van2015nonlinear,sarkka2013bayesian,sanzstuarttaeb}. Ensemble Kalman filters (EnKFs) solve an inverse problem of the form \eqref{eq:IP} every time a new observation is acquired. In that filtering context, \eqref{eq:IP} encodes the relationship between the state $u$ and observation $y$ at a given time $t,$ and  the prior on $u$ is specified by propagating a probabilistic estimate of the state at time $t-1$ through the dynamical system that governs the state evolution. To approximate this prior, EnKFs propagate an ensemble of $N$ particles through the dynamics, and subsequently update this prior  \emph{forecast} ensemble into an updated \emph{analysis} ensemble that assimilates the new observation. Thus, an ensemble Kalman update is performed every time a new observation is acquired. The goal is that the sample mean and sample covariance of the updated ensemble approximate well the mean and covariance of the filtering distribution, that is, the conditional distribution of the state at time $t$ given all observations up to time $t$. While only giving provably accurate posterior approximation in linear settings \cite{ernst2015analysis}, EnKFs are among the most popular methods for high-dimensional nonlinear filtering, in particular in numerical weather forecasting. In such applications the state dimension can be very large, but the effective dimension of the filter update is often much lower due to smoothness of the state and decay of correlations in space. Moreover, in practice the analysis step can be constrained to the subspace determined by the expanding directions of the dynamics \cite{trevisan2004assimilation}.

The papers \cite{gu2007iterative,li2007iterative,reynolds2006iterative} introduced ensemble Kalman methods for inverse problems in petroleum engineering and the geophysical sciences. Application-agnostic ensemble Kalman methods for inverse problems were developed in \cite{iglesias2013ensemble,iglesias2016regularizing}, inspired by classical regularization schemes  \cite{hanke1997regularizing}. Since then, a wide range of sequential-optimization algorithms for inverse problems have been proposed that differ in the objective function they seek to minimize and in how ensemble Kalman updates are implemented. We refer to Subsection \ref{ssec:ensembleposteriorapprox} for further background and to \cite{chada2021iterative} for a review.

Ensemble Kalman methods for inverse problems and data assimilation have been studied extensively from a large $N$ asymptotic point of view, see e.g. \cite{li2008numerical,le2009large,mandel2011convergence,kwiatkowski2015convergence,ernst2015analysis,del2018stability,herty2019kinetic,law2016deterministic,garbuno2020interacting,bishop2020mathematical,chen2021auto,ding2021ensemble}. A complementary line of work \cite{harlim2010catastrophic,gottwald2013mechanism,kelly2015concrete,tong2015nonlinear,tong2016nonlinear} has focused on challenges faced by ensemble Kalman methods, including loss of stability and catastrophic filter divergence. 
Two overarching themes that underlie large $N$ asymptotic analyses are to ensure consistency and to derive equations for the mean-field evolution of the ensemble. Related to this second theme, several works (e.g. \cite{schillings2017analysis,blomker2018strongly,blomker2019well,guth2020ensemble,chada2021iterative,tong2022localized}) set the analysis in a \emph{ continuous time limit}; the idea is to view Kalman updates as occurring over an artificial discrete-time variable, and then take the time between updates to be infinitesimally small to formally derive differential equations for the evolution of the ensemble or its density. Large $N$ asymptotics and continuous time limits have resulted in new theoretical insights and practical advancements. However, an important caveat of these results is that they cannot tell apart implementations of ensemble Kalman methods that become equivalent in large $N$ or continuous time asymptotic regimes. Moreover, several papers (e.g. \cite{bergemann2010localization,bergemann2010mollified,kelly2014well,schillings2017analysis,majda2018performance}) have noted that large $N$ asymptotic analyses fail to explain empirical results that report good performance with a moderately sized ensemble in problems with high state dimension; for instance, $d \sim 10^9$ and $N \sim 10^2$ in operational numerical weather prediction. Finally, the note \cite{nusken2019note}  shows subtle but important differences in the evolution of interacting particle systems with finite ensemble size when compared to their mean-field counterparts \cite{garbuno2020interacting}.

In this paper we adopt a non-asymptotic viewpoint to establish sufficient conditions on the ensemble size for posterior-approximation and sequential-optimization algorithms.  Empirical evidence in \cite{ott2004local} suggests that there is a sample size $N^*$ above which ensemble Kalman methods are effective.
The seminal work \cite{furrer2007estimation} conducts insightful explicit calculations that motivate our more general theory. Following the analysis of ensemble Kalman methods in  \cite{furrer2007estimation} and the study of importance sampling and particle filters in \cite{agapiou2017importance,sanz2018importance,sanz2020bayesian,bickel2008sharp,snyder2008obstacles,BBL08,snyder2011alone,chorin2013conditions,snyderberngtsson}, we focus on analyzing a single ensemble Kalman update rather than on investigating the propagation of error across multiple updates. While in practice ensemble Kalman methods for posterior approximation in data assimilation and for sequential optimization in inverse problems often perform many updates, focusing on a single update enables us to clearly demonstrate the tight connection between the sample complexity of ensemble updates and the effective dimension of the prior; additionally, for some posterior-approximation algorithms our theory generalizes in a straightforward way to multi-step implementations, as we shall demonstrate in Section \ref{sec:withoutlocalization}. 
More importantly, the focus on a single update allows us to tell apart, on accuracy grounds, perturbed observations and square root implementations of ensemble Kalman updates, as well as implementations with and without localization. Similar considerations motivate the study of sufficient sample size for importance sampling in \cite{morzfeld2017collapse,snyder2008obstacles,agapiou2017importance,chatterjee2018sample,sanz2018importance,sanz2020bayesian}, where the focus on a single update facilitates establishing clear comparisons between standard and optimal proposals, and identifying meaningful notions of dimension to characterize necessary and sufficient conditions on the required sample size. Our work builds on and develops tools from high-dimensional probability and statistics  \cite{wainwright2019high,vershynin2018high,bickel2008covariance,levina2012partial,chen2012masked,cai2012adaptive,cai2012minimax}. In particular, we bring to bear thresholded \cite{bickel2008covariance,cai2012adaptive} and masked covariance estimators \cite{levina2012partial,chen2012masked} to the understanding of localization in ensemble Kalman methods. In so doing, we
 establish new dimension-free covariance and cross-covariance estimation bounds under approximate sparsity ---see Theorems \ref{thm:SoftSparistyCovarianceBound} and \ref{thm:SoftSparistyCrossCovarianceBound}.

\subsection{Notation}
 Given two positive sequences $\{a_n\}$ and $\{ b_n\}$, the relation $a_n \lesssim b_n$ denotes that $a_n \le c b_n$ for some constant $c>0$. If the constant $c$ depends on some quantity $\tau$, then we write $a \lesssim_{\tau} b$. If both $a_n \lesssim b_n $ and $b_n \lesssim a_n$ hold simultaneously, then we write $a_n \asymp b_n$. Throughout, we denote positive universal constants by $c,c_1,c_2,c_3,c_4$, and the value of a universal constant may differ from line to line. For a vector $v \in \R^N$, $\|v\|_p^p = \sum_{n=1}^N |v_n|^p$. For a matrix $A \in \R^{n \times m}$, the operator norm is given by $\normn{A} = \sup_{\normn{v}_2=1} \|Av\|_2$. $\mcS^d_+$ denotes the set of $d\times d$ symmetric positive-semidefinite matrices, and $\mcS^d_{++}$ denotes the set of $d\times d$ symmetric positive-definite matrices. $A^\dagger$ denotes the pseudo-inverse of $A$. $1_N$ denotes the $N$-dimensional vector vector of ones, $0_d$ denotes the $d$-dimensional vector of zeroes, and $O_{d\times k}$ is the $d\times k$ matrix of zeroes.  $\indicator_B$ denotes the indicator of the set $B$. $\equiv$ denotes a definition. $\circ$ denotes the matrix Hadamard or Schur (elementwise) product. Given a non-decreasing, non-zero convex function $\psi:[0,\infty] \to [0, \infty]$ with $\psi(0)=0$, the Orlicz norm of a real random variable $X$ is $\normn{X}_{\psi} = \inf\{t > 0: \E [ \psi (t^{-1}|X|)] \le 1\}$. In particular, for the choice $\psi_p(x) \equiv e^{x^p}-1$ for $p \ge 1$, real random variables that satisfy $\normn{X}_{\psi_2} < \infty$ are referred to as sub-Gaussian. A random vector $X$ is sub-Gaussian if $\normn{v^\top X}_{\psi_2} < \infty$ for any $v$ such that $\normn{v}_2=1$. For a differentiable function $g: \R^d \to \R^k$,  $D g \in \R^{d \times k}$ denotes the Jacobian of $g.$

 All the methods we study have the same starting point of a prior ensemble 
\begin{align*}
   \upr_1, \dots, \upr_N \iid \mcN(\mpr, \Cpr),
\end{align*}
and observed data $y$ generated according to \eqref{eq:IP}, which are to be used in generating an updated ensemble $\{\upost_n\}_{n=1}^N$. 
We denote the prior sample means by
\begin{align*}
    \hatmpr &\equiv \frac{1}{N} \sum_{n=1}^N \upr_n,
    \qquad 
    \qquad 
    \barG \equiv \frac{1}{N} \sum_{n=1}^N \mcG(\upr_n),
\end{align*}
and the prior sample covariances by
\begin{align}\label{eq:samplecovariancedef}
\begin{split}
    \hatCpr \equiv \frac{1}{N-1} \sum_{n=1}^N (\upr_n- \hatmpr)
    &(\upr_n- \hatmpr)^\top,
    \qquad 
    \hatCpr^{\, pp} \equiv \frac{1}{N-1} \sum_{n=1}^N (\mcG(\upr_n)- \barG)
    (\mcG(\upr_n)- \barG)^\top,\\
    \qquad & 
    \hatCpr^{\, up} \equiv \frac{1}{N-1} \sum_{n=1}^N (\upr_n- \hatmpr)
    (\mcG(\upr_n)- \barG)^\top.
\end{split}
\end{align}
The population versions will be denoted by
\begin{align*}
    \Cpr^{\,pp} \equiv \E \insquare{ \bigl(\mcG(\upr_n) - \E [ \mcG(\upr_n)]\bigr) \bigl(\mcG(\upr_n) - \E[\mcG(\upr_n)]\bigr)^\top },
    \qquad 
    \Cpr^{\,up} \equiv \E \insquare{ \bigl(\upr_n - m \bigr) \bigl(\mcG(\upr_n) - \E[\mcG(\upr_n)]\bigr)^\top }.
\end{align*}

 \section{Ensemble Kalman Updates: Posterior Approximation Algorithms}\label{sec:withoutlocalization}

In posterior-approximation algorithms we consider the inverse problem \eqref{eq:IP} with a linear forward model, i.e.
\begin{align} \label{eq:LinearIP}
    y=Au + \eta, \qquad \eta \sim \mcN(0,\Gamma).
\end{align}
In order to establish comparisons between different posterior-approximation algorithms, as well as to streamline our analysis, we follow the exposition in \cite{kwiatkowski2015convergence} and introduce three operators that are central to the theory: the \emph{Kalman gain} operator $\msK$, the \emph{mean-update} operator $\msM,$ and the \emph{covariance-update} operator $\msC,$ defined respectively by
\begin{align}
  \msK: \mcS_+^d \to \R^{d \times k}, 
  \qquad &
  \msK(\Cpr; A ,\Gamma) = \msK(\Cpr) = \Cpr A^\top (A \Cpr A^\top + \Gamma)^{-1}, \label{eq:KalmanGainOperator}\\
  \msM: \R^d \times \mcS_+^d \to \R^{d},
  \qquad & 
  \msM(\mpr, \Cpr; A, y, \Gamma) = \msM(\mpr, \Cpr)  = \mpr + \msK(\Cpr; A, \Gamma) (y-A\mpr),\label{eq:MeanOperator}\\
  \msC: \mcS_+^d \to \mcS_+^d,
  \qquad &
  \msC(\Cpr; A, \Gamma) =
   \msC(\Cpr) = \bigl(I-\msK(\Cpr; A, \Gamma)A \bigr)\Cpr. \label{eq:CovarianceOperator}
\end{align}
The pointwise continuity and boundedness of all three operators was established in \cite{kwiatkowski2015convergence}, and we summarize these results in Lemmas~\ref{lem:KalmanOpCtyBdd}, \ref{lem:MeanOpCtyBdd}, and \ref{lem:CovarOpCtyBdd}. We note that the Kalman update \eqref{eq:KFmeancovariance} can be rewritten succinctly as
\begin{align}\label{eq:PosteriorMandelOps}
\begin{split}
    \mpost &= \msM(\mpr, \Cpr),\\ 
    \Cpost &= \msC(\Cpr).
\end{split}
\end{align}

\subsection{Ensemble Algorithms for Posterior Approximation}\label{ssec:ensembleposteriorapprox}
We study two main classes of posterior-approximation algorithms based on Perturbed Observation (PO) and  Square Root (SR) ensemble Kalman updates. In both implementations, the updated ensemble has sample mean $\hatmpost$ and sample covariance $\hatSigma$ that are, by design, consistent estimators of the posterior mean $\mu$ and covariance $\Sigma$ in \eqref{eq:PosteriorMandelOps}. Although PO and SR updates are asymptotically equivalent, differences between the two algorithms do exist in finite ensembles, and this difference is captured in our non-asymptotic analysis in Subsection~\ref{ssec:mainresults3}.

\subsubsection{Perturbed Observation Update} \label{sec:POensembleupdate}
The PO update, introduced in \cite{evensen1994sequential}, transforms each particle of the prior ensemble according to 
\begin{align*}
    \upost_n
    &=\upr_n+ \msK(\hatCpr ) \bigl(y  - A\upr_n- \eta_n\bigr)\\
    &= \msM(\upr_n , \hatCpr ) - \msK(\hatCpr )\eta_n, \qquad \eta_n \iid \mathcal{N}(0, \Gamma), \quad 1\le  n \le N.
\end{align*}
The form of the update is similar to the Kalman mean update \eqref{eq:PosteriorMandelOps} albeit with the $n$-th ensemble member being assigned a perturbed observation $y-\eta_n$. Consequently, denoting the sample mean of the perturbations by $\bareta \equiv N^{-1} \sum_{n=1}^N \eta_n$, the updated ensemble has sample mean
\begin{align*}
    \hatmpost 
    \equiv 
    \frac{1}{N} \sum_{n=1}^N \upost_n
    = \msM(\hatmpr, \hatCpr ) - \msK(\hatCpr ) \bareta,
\end{align*}
and sample covariance
\begin{align}\label{eq:sigma}
\begin{split}
    \hatCpost
    &\equiv \frac{1}{N-1} \sum_{n=1}^N
    \bigl(\upost_n- \hatmpost \bigr) \bigl(\upost_n- \hatmpost \bigr)^\top   \\
    &=\bigl(I - \msK(\hatCpr) A \bigr) 
    \hatCpr  \bigl(I - \msK(\hatCpr)  A \bigr)^\top 
     + 
    \msK(\hatCpr) \widehat{\Gamma} \msK^\top(\hatCpr) \\
    & \hspace{2cm} -  
    \bigl(I-\msK(\hatCpr) A\bigr) 
    \whatC^{\, u \eta}
    \msK^\top(\hatCpr)
     -   \msK(\hatCpr) (\whatC^{\, u \eta})^{\top} \bigl(I-A^\top\msK^\top(\hatCpr)\bigr),
\end{split}
\end{align}
 where 
\begin{align*}
    \widehat{\Gamma}
    \equiv \frac{1}{N-1} \sum_{n=1}^{N} 
    (\eta_n - \bareta_N)(\eta_n - \bareta)^\top, 
    \qquad 
    \text{and }
    \qquad 
    \whatC^{\, u\eta}
    \equiv \frac{1}{N-1} \sum_{n=1}^{N} 
    (\upr_n - \hatmpr)(\eta_n - \bareta)^\top.
\end{align*}

To facilitate comparison with the Kalman update in \eqref{eq:PosteriorMandelOps}, we rewrite the PO update as follows: 
\begin{align}\label{eq:POupdate}
\begin{split}
   \hatmpost &= \msM(\hatmpr, \hatCpr ) - \msK(\hatCpr ) \bareta,  \\
   \hatCpost &= \msC(\hatCpr) + \widehat{O},
\end{split}
\end{align}
where the \textit{offset} term $\widehat{O},$ obtained as the difference between \eqref{eq:sigma} and $\msC(\hatCpr)$, 
is given by 
\begin{align}\label{eq:offset}
\widehat{O} =  \msK(\hatCpr) (\widehat{\Gamma} - \Gamma) \msK^\top(\hatCpr)
    -  \bigl(I-\msK(\hatCpr) A\bigr) 
    \whatC^{\, u \eta} 
    \msK^\top(\hatCpr)
      -   \msK(\hatCpr) (\whatC^{\, u \eta})^{\top} \bigl(I-A^\top\msK^\top(\hatCpr)\bigr).
\end{align}
The offset term $\widehat{O}$ was introduced in \cite[Proposition 4]{furrer2007estimation}. The addition of perturbations serves the purpose of correcting the sample covariance, in the sense that without perturbations the sample covariance is an inconsistent estimator of $\Sigma$. To see the consistency of the PO covariance estimator $\hatCpost$ in \eqref{eq:POupdate}, note that by Lemma~\ref{lem:CovarOpCtyBdd} the map $\msC$ is continuous, and so the continuous mapping theorem together with the fact that $\hatCpr$ is consistent for $\Cpr$ imply that $\msC(\hatCpr) \convp \msC(\Cpr) = \Sigma$. Further, the offset $\widehat{O}$ converges in probability to zero, which can be shown using that $\widehat{\Gamma} \convp \Gamma$, $\whatC^{\, u \eta} \convp O_{d \times k},$ and the continuity of $\msK$ established in Lemma~\ref{lem:KalmanOpCtyBdd}.

\subsubsection{Square Root Update} \label{sec:SRensembleupdate}

The PO update relies crucially on the added perturbations to maintain consistency and, as noted for example in \cite{evensen2004sampling, tippett2003ensemble, bishop2001adaptive}, is asymptotically equivalent to the exact posterior update \eqref{eq:KFmeancovariance}. However, for a finite ensemble of size $N$, the addition of random perturbations introduces an extra source of error into the ensemble Kalman update. The SR update, introduced in \cite{evensen2004sampling} and surveyed in \cite{tippett2003ensemble, lange2020mean}, is a deterministic alternative to the PO update. It updates the prior ensemble in a manner that ensures that $\hatSigma \equiv \msC(\hatC)$. This is achieved by first identifying a map $g:\R^{d \times N} \to \R^{d \times N}$ such that $\hatsrSigma = g(\hatsrC),$ where 
\begin{align*}
    \hatCpr= \hatsrC \hatsrC^\top, 
    \qquad \text{and} 
    \qquad \msC(\hatC) = \hatsrSigma\hatsrSigma^\top,
\end{align*}
with both factorizations guaranteed to exist since $\hatCpr, \msC(\hatC) \in \mcS^d_+$. Consistency of $\hatCpost$ can then be ensured by choosing $g$ to satisfy $g(\hatsrC) g(\hatsrC)^\top \equiv \msC(\hatCpr)$, with this being referred to as the \textit{consistency condition} in \cite{lange2020mean}. There are infinitely many such $g,$ each of which lead to a variant of the SR update. Here we describe two of the most popular variants in the literature as outlined in \cite{tippett2003ensemble}: the Ensemble Transform Kalman update \cite{bishop2001adaptive} and the Ensemble Adjustment Kalman update \cite{anderson2001ensemble} with respective transformations $g_T(\hatsrC) = \hatsrC T$ and $g_A(\hatsrC) = B\hatsrC$, for matrices $T$ and $B$. Both $g_T$ and $g_A$ are therefore linear maps, with $g_T$ post-multiplying $\hatsrC$, which implies a transformation on the $N$-dimensional space spanned by the ensemble, and $g_A$ pre-multiplying $\hatsrC$, so that the transformation is applied to the $d$-dimensional state-space instead. In both approaches we identify the relevant matrix by first writing 
\begin{align*}
        \hatsrSigma\hatsrSigma^\top
        =
        \msC(\hatC)
        = \hatsrC (I - V D^{-1} V^\top) \hatsrC^\top,
\end{align*}
where $V = (A \hatsrC)^\top$ and $D= V^\top V + \Gamma$.
\begin{enumerate}
    \item \textit{Ensemble Transform Kalman Update:} taking $\hatsrSigma = \hatsrC F U$ for any $F$ satisfying $FF^\top = I- VD^{-1}V^\top$ and arbitrary orthogonal $U$ satisfies the consistency condition. One approach for finding such a matrix $F$ is by rewriting 
    \begin{align*}
        I - V D^{-1} V^\top
        = 
        (I + \hatsrC^\top A^\top \Gamma^{-1} A \hatsrC)^{-1} = 
        E(I + \Lambda)^{-1/2} (I+\Lambda)^{-1/2}E^\top = FF^\top,
    \end{align*}
    where the first equality follows by the Sherman-Morrison formula, and $E \Lambda E^\top$ is the eigenvalue decomposition of $\hatsrC^\top A^\top \Gamma^{-1} A \hatsrC$. In summary, we have $g_{T}(\hatsrC) = \hatsrC E(I + \Lambda)^{-1/2} U$.
    
    \item \textit{Ensemble Adjustment Kalman Update:} Introducing $M = V \Gamma^{-1/2}$, we can write 
    \begin{align*}
        \hatsrC (I - V D^{-1} V^\top) \hatsrC^\top
        = \hatsrC (I + MM^\top)^{-1} \hatsrC^\top.
    \end{align*}
    Noting that $\hatsrC$ has full column rank, we may then define $B = \hatsrC (I + MM^\top)^{-1/2} (\hatsrC^\top)^\dagger$, and so 
    \begin{align*}
        g_A(\hatsrC) 
        = B \hatsrC =  \hatsrC (I + MM^\top)^{-1/2} (\hatsrC^\top)^\dagger\hatsrC  = \hatsrC (I + MM^\top)^{-1/2}.
    \end{align*}
\end{enumerate}

Once a choice of $g$ has been made, and an estimate $\hatCpost$ has been computed, the updated ensemble has first two moments given by 
\begin{align}\label{eq:SRupdate}
\begin{split}
    \hatmpost&= \msM(\hatmpr, \hatCpr),\\
    \hatCpost &= \msC(\hatCpr).
    \end{split}
\end{align}
Frequently, only $\hatmpost, \hatCpost$ are of concern to the practitioner, but it is still possible to \textit{back-out} the individual members of the updated ensemble as they may be of interest. It is clear that one choice for $\hatsrC$ is 
\begin{align*}
    \hatsrC  = \frac{1}{\sqrt{N-1}}
    \begin{bmatrix}
    \upr_1- \hatmpr,
    \cdots,
    \upr_N - \hatmpr
    \end{bmatrix},
\end{align*}
in which case it holds that $\hatsrC 1_N = 0_d$, and so 
\begin{align}\label{eq:SRensemblemembers}
    \upost_n = \sqrt{N-1}[\hatsrSigma]_n + \msM(\hatmpr, \hatCpr),
    \qquad 1 \le n \le N,
\end{align}
where $[\hatsrSigma]_n $ denotes the $n$-th column of $\hatsrSigma$. 

In Subsection \ref{ssec:mainresults3} we establish error bounds for the approximation of the posterior mean and covariance $(\mpost, \Cpost)$ in \eqref{eq:KFmeancovariance} by $(\hatmpost,\hatCpost)$ as estimated using the PO and SR updates in \eqref{eq:POupdate} and \eqref{eq:SRupdate}. It is clear from \eqref{eq:SRupdate} that as long as the choice of $g$ is valid, in the sense that the resulting $\hatSigma$ is consistent, then the specific choice of $g$ is irrelevant to the accuracy of a single SR update. We therefore make no assumptions in our subsequent analysis of the SR algorithm beyond that of $g$ satisfying the consistency condition. Note that, when compared to the SR update in \eqref{eq:SRupdate}, the PO update in \eqref{eq:POupdate} contains additional stochastic terms that will, as our bounds indicate, hinder the estimation of $(\mpost,\Cpost).$ As noted in the literature, for example in \cite{tippett2003ensemble}, the PO update increases the probability of underestimating the analysis error covariance. While our presentation and analysis of PO and SR updates is carried out in the linear-Gaussian setting, both updates are frequently utilized in nonlinear  and non-Gaussian settings, with empirical evidence suggesting that the PO updates can outperform SR updates \cite{lawson2004implications,leeuwenburgh2005impact}. In fact, the consistency argument outlined above is only valid in the linear case $\mcG(u) = Au,$ and the statistical advantage of SR implementations in linear settings may not translate into the nonlinear case. 

\subsection{Dimension-Free Covariance Estimation} \label{ssec:covest}
 We define the \emph{effective dimension} \cite{wainwright2019high} of  a matrix $Q \in \mcS^d_+$ by 
 \begin{align} \label{eq:effectiveDim}
    r_2(Q) \equiv \frac{\ttrace(Q)}{\normn{Q}}.   
 \end{align}
 The effective dimension quantifies the number of directions where $Q$ has significant spectral content \cite{tropp2015introduction}.
 The monographs \cite{tropp2015introduction,vershynin2018high} refer to $r_2(Q)$ as the intrinsic dimension, while \cite{koltchinskii2017concentration} uses the term effective rank.  This terminology is motivated by the observation that $1 \le r_2(Q) \le \text{rank} (Q) \le d $ and that $r_2(Q)$ is insensitive to changes in the scale of $Q$, see \cite{tropp2015introduction}.  In situations where the eigenvalues of $Q$ decay rapidly, $r_2(Q)$ is a better measure of dimension than the state dimension $d$. The following result \cite[Theorem 9]{koltchinskii2017concentration} gives a non-asymptotic sufficient sample size requirement for accurate covariance estimation in terms of the effective dimension of the covariance matrix. We recall that the sample covariance estimator $\hatCpr$ is defined in \eqref{eq:samplecovariancedef}.  

\begin{proposition}[Covariance Estimation with Sample Covariance ---Unstructured Case] \label{thm:Koltchinski}
    Let $\upr_1,\dots, \upr_N$ be $d$-dimensional i.i.d. sub-Gaussian random vectors with $\E [\upr_1]=m$ and $\tvar [\upr_1]=\Cpr$. Then, for all $t \ge 1$, it holds with probability at least $1- ce^{-t}$ that
    \begin{align*}
        \normn{\hatCpr  - \Cpr} \lesssim 
        \normn{\Cpr} \inparen{ \sqrt{\frac{r_2(\Cpr)}{N}} \lor \frac{r_2(\Cpr)}{N} 
        \lor \sqrt{\frac{t}{N}} \lor \frac{t}{N}}.
    \end{align*}
\end{proposition}

 \begin{remark}[Effective Dimension and Smoothness]
Proposition \ref{thm:Koltchinski} motivates defining $r_2(C) \equiv \ttrace(C)/\norm{C}$ to be
the effective dimension of a $d$-dimensional sub-Gaussian random vector $u$ with $\tvar[u] = C$. As in the definition for matrices, $r_2(C)$ quantifies the number of directions where the distribution of $u$ has significant spread. Proposition \ref{thm:Koltchinski} and our results in Subsection \ref{ssec:mainresults3} may be extended to sub-Gaussian random variables defined in an infinite-dimensional separable Hilbert space, say $\mcH = L^2(0,1)$. It is then illustrative to note that any Gaussian measure $\mcN(m, C)$ in $\mcH$ satisfies that $\ttrace(C) <\infty;$ in other words, all Gaussian measures have finite effective dimension. In this context, $r_2(C)$ is related to the rate of decay of the eigenvalues of $C,$ and hence to the almost sure Sobolev regularity of functions $u$ drawn from the Gaussian measure  $\mcN(m, C)$ on $\mcH = L^2(0,1),$ see e.g. \cite{bogachev1998gaussian,AS10}.
In computational inverse problems and data assimilation, $u$ is often a $d$-dimensional vector that represents a fine discretization of a Gaussian random field; then, $r_2(C)$ quantifies the smoothness of the undiscretized field.  
\end{remark}

\subsection{Main Results: Posterior Approximation with Finite Ensemble} \label{ssec:mainresults3}
In this subsection we state finite ensemble approximation results for the posterior mean and covariance with PO and SR ensemble updates. To highlight some key insights, including the dependence of the bounds on the effective dimension of $C$ and the differences between PO and SR updates, we opt to present expectation bounds in Theorems \ref{lem:meanwithoutlocExpectation} and \ref{lem:covwithoutlocExpectation} that are less notationally cumbersome than the stronger exponential tail bounds in Theorems \ref{th:meanwithoutloc} and \ref{th:covwithoutloc} in Appendix \ref{sssec:A3}.  
Throughout this section, the data $y$ is treated as a fixed quantity.

\begin{theorem}[Posterior Mean Approximation with Finite Ensemble ---Expectation Bound] \label{lem:meanwithoutlocExpectation}
 Consider the PO and SR ensemble Kalman updates given by \eqref{eq:POupdate} and \eqref{eq:SRupdate}, respectively, leading to an estimate $\hatmpost$ of the posterior mean $\mpost$ defined in \eqref{eq:KFmeancovariance}. Set $\phi=1$ for the PO update and $\phi=0$ for the SR update. Then, for any $p\ge 1$,
\begin{align} \label{eq:meandeviation}
    \bigl[ \E  \normn{\hatmpost- \mpost}^p_2 \bigr]^{1/p}
    &\lesssim_p
        c_1
        \inparen { \sqrt{\frac{r_2(\Cpr)}{N}} \lor \inparen{\frac{r_2(\Cpr)}{N}}^{3/2}}
        + \phi c_2
        \inparen{
        \sqrt{\frac{r_2(\Gamma)}{N}} \lor 
        \frac{r_2(\Cpr)}{N}
        \sqrt{\frac{r_2(\Gamma)}{N}}
        },
\end{align}
where $c_1 = c_1(\normn{\Cpr}, \normn{A}, \normn{\Gamma^{-1}}, \normn{y-A\mpr}_2 )$ and $c_2 = c_2(\normn{\Cpr}, \normn{A}, \normn{\Gamma^{-1}})$.
\end{theorem}

Importantly, the bound \eqref{eq:meandeviation} does not depend on the dimension $d$ of the state-space, and the only dependence on $C$ is through its operator norm and  the effective dimension $r_2(C).$ 
The term 
multiplied by $\phi$  in the PO update accounts for the additional error incurred by the presence of the offset term \eqref{eq:offset} in the PO update \eqref{eq:POupdate}. The following remark discusses another important consequence of Theorem \ref{lem:meanwithoutlocExpectation}: the stable performance of ensemble Kalman updates in small noise regimes when compared with other sampling algorithms. 

\begin{remark}[Dependence of Constants on Model Parameters]\label{rem:dependenceconstants}
The proof of Theorem~\ref{lem:meanwithoutlocExpectation} in Appendix \ref{sec:proofs} provides an explicit definition of $c_1$ and $c_2$ up to constants, i.e. it describes how these quantities rely on their arguments, and Theorem~\ref{th:meanwithoutloc} establishes a high probability bound on $\normn{\hatmpost - \mpost}_2.$   
In particular, it is important to note that the constants $c_1$ and $c_2$ in Theorem \ref{lem:meanwithoutlocExpectation} deteriorate in the small noise limit where the observation noise goes to zero, and $c_2$ deteriorates with $r_2(\Gamma)$. In the small noise limit, the posterior and prior distribution become mutually singular, and it is hence expected for ensemble updates to be unstable. To illustrate this intuition in a concrete setting, assume that $\Gamma = \gamma I$ for a positive constant $\gamma$, and, for simplicity, that $N \ge r_2(\Cpr)$ as well as $\normn{\Cpr} =  \normn{A} = \normn{y-Am}_2 = 1$. Then, the expression for $c_1$ established in Theorem \ref{lem:meanwithoutlocExpectation}  implies that for the SR update, for any error $\varepsilon>0$ and $p \ge 1$ ,
\begin{align*}
    N  \gtrsim \frac{r_2(\Cpr)}{\varepsilon^2 \gamma^4}
    \implies 
    \E \bigl[ \normn{\hatmpost- \mpost}^p_2 \bigr]^{1/p} \lesssim_p \varepsilon.
\end{align*}
Similarly, the expressions for $c_1$ and $c_2$ imply that for the PO update,
\begin{align*}
    N  \gtrsim \frac{r_2(\Cpr)}{\varepsilon^2 \gamma^4} \lor \frac{k}{\varepsilon^2 \gamma}
    \implies 
    \E \bigl[ \normn{\hatmpost- \mpost}^p_2 \bigr]^{1/p} \lesssim_p \varepsilon,
\end{align*}
where we recall that $k$ denotes the dimension of the data $y.$ 
The papers \cite{agapiou2017importance,sanz2020bayesian} show the need to increase the sample size along small noise limits in importance sampling when target and proposal are given, respectively, by posterior and prior. While our bounds here only give sufficient rather than necessary conditions on the ensemble size, it is noteworthy that, for fixed $k,$ the scaling of $N$ as $\gamma \to 0$ shown here is independent of $k.$ In contrast, necessary sample size conditions for importance sampling show a polynomial dependence on $k,$ see \cite{sanz2020bayesian}.              
\end{remark}

\begin{theorem} [Posterior Covariance Approximation with Finite Ensemble ---Expectation Bound] \label{lem:covwithoutlocExpectation}
Consider the PO and SR ensemble Kalman updates given by \eqref{eq:POupdate} and \eqref{eq:SRupdate}, respectively, leading to an estimate $\hatCpost$ of the posterior covariance $\Cpost$ defined in \eqref{eq:KFmeancovariance}. Set $\phi=1$ for the PO update and $\phi=0$ for the SR update. Then, for any $p \ge 1$,
    \begin{align*}
        \bigl[\E  \normn{\hatCpost - \Cpost}^p  \bigr ]^{1/p}
        &\lesssim_p
        c_1
        \inparen{\sqrt{\frac{r_2(\Cpr)}{N}} \lor \inparen{\frac{r_2(\Cpr)}{N}}^2}
        + \phi \msE,
    \end{align*}
    where
    \begin{align*}
    \msE
    &=c_2 
    \inparen{ 
     \sqrt{\frac{r_2(\Cpr)}{N}} \lor 
     \inparen{\frac{r_2(\Cpr)}{N}}^3 \lor 
     \inparen{\sqrt{\frac{r_2(\Gamma)}{N}} \lor \frac{r_2(\Gamma)}{N}}
     \inparen{ 
     1 \lor \inparen{\frac{r_2(\Cpr)}{N}}^2 
     }
     },
    \end{align*}
    where $c_1 = c_1(\normn{\Cpr}, \normn{A}, \normn{\Gamma^{-1}})$ and $c_2 = c_2(\normn{\Cpr}, \normn{A}, \normn{\Gamma^{-1}}, \normn{\Gamma}).$
\end{theorem}

As in Theorem \ref{lem:meanwithoutlocExpectation}, the bound in Theorem \ref{lem:covwithoutlocExpectation} does not depend on the dimension $d$ of the  state-space, and the dependence on $C$ is through the operator norm and the  effective dimension $r_2(C).$

\begin{remark}[Dependence of Constants on Model Parameters]
The proof of Theorem~\ref{lem:covwithoutlocExpectation} in Appendix \ref{sec:proofs} provides an explicit definition of $c_1$ and $c_2$ up to constants and Theorem~\ref{th:covwithoutloc} establishes a high probability bound on $\normn{\hatCpost - \Cpost}.$  
 As discussed in Remark \ref{rem:dependenceconstants}, these bounds may be used to establish sufficient ensemble size requirements in small noise limits and other singular limits of practical importance. 
\end{remark}

\begin{remark}[Comparison to the Literature]\label{rem:MultiStep}
The results in this section complement many of the existing analyses of ensemble Kalman updates in the literature. In one direction, our Theorems~\ref{lem:meanwithoutlocExpectation} and \ref{lem:covwithoutlocExpectation} can be viewed in the context of \cite[Section 3.4]{furrer2007estimation}, which claims that for finite ensembles the square root filter is always more efficient than the perturbed observation filter, since the latter introduces additional variability through noisy perturbations of the data. Our results quantify this additional variability both in probability and in expectation.  In \cite{majda2018performance}, the authors put forward a non-asymptotic analysis of a multi-step EnKF augmented by a spectral projection step in which the Kalman gain matrix is projected onto the linear span of its leading eigenvalues exceeding a threshold level. They refer to the dimension $d_{subs}$ of this subspace as the effective dimension and provide guarantees on the performance of the algorithm so long as the ensemble size scales with $d_{subs}.$ In contrast, our one-step analysis does not require any augmentation of standard implementations (see e.g. \cite{furrer2007estimation}) of the ensemble update. They also employ (forecast) covariance inflation, which is a de-biasing technique standard in the literature, see for example \cite[Section 5]{furrer2007estimation}, which our results do not require.  In another direction, our results can be directly compared to \cite[Theorem 6.1]{kwiatkowski2015convergence}, which states that for iteration $t$ of the square root EnKF and for any $p\ge 1$
\begin{align}\label{eq:MandelBound}
     \insquare{ \E \normn{\hatmu^{(t)} - \mu^{(t)}}^p_2  }^{1/p} 
     \le \frac{c(p, t)}{\sqrt{N}}
     \qquad 
     \text{and}
     \qquad 
     \insquare{ \E \normn{\hatSigma^{(t)} - \Sigma^{(t)}}^p  }^{1/p} 
     \le \frac{c(p, t)}{\sqrt{N}},
\end{align}
where $\hatmu^{(t)}$ and $\hatSigma^{(t)}$ are the sample mean and covariance of the updated (analysis) ensemble at iteration $t$, and $\mu^{(t)}$ and $\Sigma^{(t)}$ are the corresponding Kalman Filter posterior mean and covariance, respectively. The term $c(p,t)$ that arises in both of their bounds denotes a constant that depends only on $p,$ the iteration index $t$, and the norms of the non-random inputs of the algorithm, but do not depend on dimension or ensemble size. Importantly, they do not distinguish between settings with different effective dimensions as our bounds do. In Appendix~\ref{sec:Multistep}, we provide an explicit outline of the multi-step algorithm considered in their paper along with definitions of all quantities described here. As previously noted, our bounds cover the perturbed observation setting whereas \eqref{eq:MandelBound} is specific to the square root setting. In Appendix~\ref{sec:Multistep} we also establish (see Corollary \ref{cor:MultiStepSREnKF}) a simple extension of Theorems~\ref{lem:meanwithoutlocExpectation} and \ref{lem:covwithoutlocExpectation} to the multi-step square root setting, which shows that for any $p \ge 1$, iteration $t$, and assuming for simplicity that $N \gtrsim r_2(\Cpost^{(0)})$, then 
    \begin{align}
    \begin{split}
    \label{eq:multistep}
        \insquare{ \E \normn{\hatmu^{(t)} - \mu^{(t)}}^p_2  }^{1/p}
        &\lesssim_p 
        \sqrt{\frac{r_2(\Cpost^{(0)})}{N}}
        \times 
        c ( \{ \normn{M^{(l)}}, \normn{A^{(l)}}, \normn{\Cpost^{(l-1)}}, \normn{y^{(l)} - A^{(l)} \mpr^{(l)} }_2 \}_{l=1}^t, \normn{\Gamma^{-1}}), \\
        \insquare{ \E \normn{\hatSigma^{(t)} - \Sigma^{(t)}}^p  }^{1/p} 
        &\lesssim_p 
        \sqrt{\frac{r_2(\Cpost^{(0)})}{N}}
        \times 
        c ( \{ \normn{M^{(l)}}, \normn{A^{(l)}}, \normn{\Cpost^{(l-1)}}\}_{l=1}^t, \normn{\Gamma^{-1}}).
        \end{split}
    \end{align}
     Our bounds therefore refine those in \cite{kwiatkowski2015convergence} as they explicitly capture the dependence on the state dimension through the effective dimension of the initial distribution, $r_2(\Cpost^{(0)})$. It follows then that in the case of the square root EnKF, it suffices to use an ensemble on the order of the effective dimension of $\Cpost^{(0)}$ multiplied by constants depending on the operator norms of the forward model matrices $\{\normn{A^{(l)}} \}_{l=1}^t$, analysis covariance matrices $\{\normn{\Cpost^{(l)}}\}_{l=1}^t$, inverse of the noise covariance, $\normn{\Gamma^{-1}}$ and $\ell_2$-norm of the model errors $\{\normn{y^{(l)} - A^{(l)} \mpr^{(l)} }_2\}_{l=1}^t$. We note that extensions to the multi-step setting  for other variants of the EnKF that do not use SR updates  may not follow as easily. In this direction, the recent work \cite{ghattas2023ensemble} studies a multi-step EnKF with PO updates which incorporates an additional resampling step.  
\end{remark}

\section{Ensemble Kalman Updates: Sequential Optimization Algorithms}\label{sec:withlocalization}
In the optimization approach, the solution to the inverse problem \eqref{eq:IP} is found by minimizing an objective function. As discussed in \cite{chada2021iterative}, an entire suite of ensemble algorithms have been derived that differ in the choice of objective function and optimization scheme. In this subsection we introduce the Ensemble Kalman Inversion (EKI) algorithm \cite{iglesias2013ensemble} and a new localized implementation of EKI, which we call localized EKI (LEKI) following \cite{tong2022localized}. Both EKI and LEKI use an ensemble approximation of a Levenberg-Marquardt (LM) optimization scheme to minimize a data-misfit objective
\begin{align}\label{eq:DMObjective}
    \JJ(u) = \frac{1}{2} \normn{ \Gamma^{-1/2}\bigl(y - \mcG(u)\bigr)}_2^2,
\end{align}
which promotes fitting the data $y$. 
Before deriving EKI the in Subsection \ref{sssec:EKI} and LEKI in Subsection \ref{sssec:LEKI}, we give some background that will help us interpret both methods as ensemble-based implementations of classical gradient-based LM schemes. The finite ensemble approximation of an idealized mean-field EKI update using EKI and LEKI updates will be studied in Subsection \ref{ssec:mainresults4}.

Recall that classical iterative optimization algorithms choose an initialization $u^{(0)}$ and set
\begin{align}\label{eq:iterationScheme}
    u^{(t+1)} = u^{(t)} + w^{(t)}, \qquad t=0,1,\dots,
\end{align}
until a pre-specified convergence criterion is met. Here, $w^{(t)}$ is some favorable direction determined by the optimization algorithm at iteration $t,$ given the current estimate $u^{(t)}$.  
In the case that the inverse problem is ill-posed,  directly minimizing \eqref{eq:DMObjective} leads to a solution that over-fits the data. Then, implicit regularization can be achieved through the optimization scheme used to obtain the update $w^{(t)}$. Under the assumption that $r(u) \equiv y-\mcG(u)$ is differentiable, the Levenberg-Marquardt (LM) algorithm chooses $w^{(t)}$ by solving the constrained minimization problem
\begin{align*}
    \min_{w} 
    \mathsf{J}_t^{\text{lin}}(w) 
    \quad 
    \text{subject to} 
    \quad 
    \normn{\Cpr^{-1/2} w}_2^2 \le \delta_l,
\end{align*}
where 
\begin{align*}
    \mathsf{J}_t^{\text{lin}}(w) 
    \equiv 
    \frac{1}{2} \normn{D r(u^{(t)}) w  + r(u^{(t)}) }_2^2,
\end{align*}
and $Dr$ denotes the Jacobian of $r.$
 The LM algorithm belongs to the class of trust region optimization methods, and it chooses each increment to minimize a linearized objective, $\mathsf{J}_t^{\text{lin}}$, but with the added constraint that the minimizer belongs to the ball $\{\normn{C^{-1/2} w}^2 \le \delta_l \}$, in which we \textit{trust} that the objective may be replaced by its linearization. Equivalently, $w^{(t)}$ can be viewed as the unconstrained minimizer of a regularized objective,
\begin{align}\label{eq:LMoptimization}
    \min_w \mathsf{J}_t^{\text{U}}(w),
    \qquad 
    \mathsf{J}_t^{\text{U}}(w) 
    \equiv \mathsf{J}_t^{\text{lin}}(w) + \frac{1}{2 \alpha_t} 
    \normn{C^{-1/2} w}_2^2,
\end{align}
where $\alpha_t >0$ acts as a Lagrange multiplier. 

We are interested in ensemble sequential-optimization algorithms, which instead of updating a single estimate $u^{(t)}$ ---as in \eqref{eq:iterationScheme}--- propagate an {\emph{ensemble}} of estimates. 
Ensemble-based optimization schemes often rely on \textit{statistical linearization} to avoid the computation of derivatives. Underpinning this idea \cite{ungarala2012iterated,chada2021iterative,kim2022hierarchical} is the argument that if $\mcG(u)=Au$ were linear, then $\hatCpr^{\, up} = \hatC A^\top$, leading to the approximation in the general nonlinear case 
\begin{align}\label{eq:StatLinearization}
    D \mcG(\upr_n) 
    \approx (\hatCpr^{\, up})^\top \hatCpr^\dagger 
    \equiv G. 
\end{align}
This approximation motivates the derivative-free label often attached to ensemble-based algorithms \cite{kovachki2019ensemble}, and we note that they may be employed whenever computing $D\mcG(u)$ is expensive or when $\mcG$ is not differentiable. For the remainder, our analysis focuses on a single step of EKI and LEKI, and so we drop the iteration index $t$ from our notation; we will use instead our previous terminology of prior ensemble and updated ensemble. Finally, similar to our presentation of posterior-approximation algorithms, our exposition is simplified by introducing the \textit{nonlinear gain-update} operator $\msP$,
\begin{align}\label{eq:NonlinearGainOperator}
    \msP: \R^{d \times k} \times \mcS_+^k \to \R^{d \times k},
    \qquad 
    & \msP(\Cpr^{\,up}, \Cpr^{\,pp}; \Gamma) 
    = \msP(\Cpr^{\,up}, \Cpr^{\,pp}) 
    = \Cpr^{\,up}(\Cpr^{\,pp}+\Gamma)^{-1},
\end{align}
which is shown to be both pointwise continuous and bounded in Lemma~\ref{lem:EKIopCtyBdd}.

\subsection{Ensemble Algorithms for Sequential Optimization}

\subsubsection{Ensemble Kalman Inversion Update}\label{sssec:EKI}
In the EKI, each particle in the prior ensemble is updated according to the LM algorithm, so that 
\begin{align*}
    \upost_n = \upr_n + w_n, \qquad 1 \le n \le N,
\end{align*}
where $w_n$ is the minimizer of a linearized and regularized data-misfit objective
\begin{align} \label{eq:EKIUnconstrained}
    \mathsf{J}^{\text{lin}}_n(w)
    = \frac{1}{2} \normn{\Gamma^{-1/2} \bigl( y-\eta_n - \mcG(u_n) - G w\bigr)}^2_2
    + \frac{1}{2 \alpha} \normn{\hatC^{-1/2} w}^2_2,
    \qquad \eta_n \sim \mcN(0, \Gamma).
\end{align}
Following \cite{iglesias2013ensemble}, we henceforth set $\alpha = 1,$  but note that our main results can be readily extended to any $\alpha>0$. Note that each ensemble member solves the optimization \eqref{eq:EKIUnconstrained} with a perturbed observation $y-\eta_n$, similar in spirit to the PO update of Subsection~\ref{sec:POensembleupdate}. The minimizer of \eqref{eq:EKIUnconstrained} (with $\alpha = 1$) is given by
\begin{align*}
    w_n = \hatC G^\top
    ( G \hatCpr G^\top +  \Gamma)^{-1} \bigl(y - \eta_n - \mcG(u_n)\bigr).
\end{align*}
Substituting $\hatCpr G^\top= \hatCpr^{\, up}$, and approximating
\begin{align*}
     G \hatCpr G^\top
     =  G \hatCpr^{\, up}
     =   (\hatCpr^{\, up})^\top \hatCpr^\dagger \hatCpr^{\, up}
     \approx \hatCpr^{\, pp}
\end{align*}
 leads to the EKI update  
\begin{align}\label{eq:EKIupdate}
    \upost_n 
    &= \upr_n + \msP(\hatCpr^{\, up}, \hatCpr^{\, pp}) \bigl(y - \mcG(\upr_n) - \eta_n \bigr), 
    \qquad 
    1 \le n \le N.
\end{align}
In the linear forward-model setting,  $\msP(\hatCpr^{\, up}, \hatCpr^{\, pp}) = \msK(\hatCpr)$, and \eqref{eq:EKIupdate} takes on a form identical to the PO update in \eqref{eq:POupdate}. We further define the \textit{mean-field} EKI update
\begin{align} \label{eq:MFEKIUpdate}
    \upost_n^* 
    &= \upr_n + \msP(C^{\, up}, \Cpr^{\, pp}) \bigl(y - \mcG(\upr_n) - \eta_n \bigr), 
    \qquad 
    1 \le n \le N,
\end{align}
which is the update that would be performed if one had access to the population quantities $C^{up}$ and $C^{pp}$ or, equivalently, to an infinite ensemble. We will analyze the approximation of the update \eqref{eq:EKIupdate} to the mean-field update \eqref{eq:MFEKIUpdate} in Subsection \ref{ssec:mainresults4}.
The study of mean-field ensemble Kalman methods of the form \eqref{eq:MFEKIUpdate} was proposed in \cite{herty2019kinetic} and is overviewed in \cite{calvello2022ensemble}. While mean-field algorithms are not useful for practical implementation, they facilitate a transparent mathematical analysis that can provide understanding on the performance of practical ensemble approximations. Desirable properties of mean-field algorithms include convergence to the desired target in a continuous-time limit \cite{carrillo2021wasserstein}, a gradient flow structure \cite{garbuno2020interacting}, or the ability to approximate derivative-based optimization algorithms \cite{chada2021iterative}. 
The transfer of theoretical insights from mean-field algorithms to particle-based algorithms tacitly presupposes, however, that the ensemble is large enough for ensemble-based updates to approximate well idealized mean-field updates. In this direction, \cite{ding2021ensemble} establishes a $\mathcal{O}(N^{-1/2})$ rate for an approximation of a mean-field evolution equation in terms of the ensemble size $N$.
Our first main result of this section, Theorem \ref{thm:EKI}, will show that the mean-field update \eqref{eq:MFEKIUpdate} can be well approximated with the EKI update \eqref{eq:EKIupdate} with a number of particles of the order of the \emph{effective dimension} of the problem, which is defined as for posterior-approximation algorithms.

\subsubsection{Localized Ensemble Kalman Inversion Update}\label{sssec:LEKI}

In practice, ensemble-based algorithms are often implemented with $N \ll d$, that is, with an ensemble that is much smaller than the state dimension. In this setting, the update is augmented with an additional \textit{localization} procedure applied to $\hatCpr$ in the case of linear forward model, and to 
 both $\hatCpr^{\,pp}$ and $\hatCpr^{\,up}$ in the case of a nonlinear forward model. In either case, localization is seen as an approach to deal both with the extreme rank deficiency and the sampling error that arise from using an ensemble that is significantly smaller than the dimension of the state and/or the dimension of the observation, see for example \cite{houtekamer2001sequential, houtekamer2016review, farchi2019efficiency}. 
Localization is also useful when the state $\upr,$ or the transformed state $\mcG(\upr)$, has elements $\mcE(i)$ and $\mcE(j)$ that represent the values of a variable of interest at physical locations that are a known distance $\mathsf{d}(i,j)$ apart: correlations may decay quickly with the physical distance of the variables and localization may help to remove spurious correlations in the sample covariance estimator.
In ensemble Kalman methods, localization has most commonly been carried out via the Schur (elementwise) product of the estimator and a positive-semidefinite matrix $\mathsf{M}$ of equal dimension. In the vast majority of cases, the elements of $\mathsf{M}$ are taken to be $\mathsf{M}_{ij} = \kappa(\mathsf{d}(i,j)/b)$, where $\kappa$ is a locally supported correlation function ---usually the Gaspari Cohn $5^{\text{th}}$-order compact piecewise polynomial \cite{gaspari1999construction}--- and $b>0$ is a length-scale parameter chosen by the practitioner. Since $\kappa$ tapers off to zero as its argument becomes larger, i.e. when the underlying variables are further apart, the Schur-product operation zeroes out the corresponding elements of the estimator, and the rate at which this tapering occurs is controlled by the size of the length-scale. The localized EKI (LEKI), recently studied in \cite{tong2022localized}, replaces both $\hatCpr^{\,pp}$ and $\hatCpr^{\, up}$ with their localized counterparts, $\mathsf{M}_1 \circ \hatCpr^{\, pp}$ and $\mathsf{M}_2 \circ \hatCpr^{\, up}$, where $\mathsf{M}_1$ and $\mathsf{M}_2$ are localization matrices of appropriate dimension. Two important issues have, in our opinion, hindered the rigorous study of localized ensemble algorithms, and we highlight these next before moving on to introduce our localization framework.

\begin{enumerate}
 \item Optimality: The justification outlined above for localization in the ensemble Kalman literature has been largely heuristic, and relying on these arguments alone one cannot hope to define a localization procedure that is demonstrably optimal. Notably, the widespread usage of the Gaspari-Cohn correlation function is not rooted in any sense of optimality. Generally, focusing solely on a band of entries near the diagonal is a sub-optimal approach to covariance estimation, as noted in the high-dimensional covariance estimation literature, see for example \cite{chen2012masked, levina2012partial, bickel2008regularized}. Moreover, even in cases where focusing on elements near the diagonal is justified, for example by assuming that the underlying target is a banded matrix, the bandwidth $b>0$ must be chosen carefully as a function of the ensemble size, problem dimension, and dependence structure \cite{bickel2008covariance}. This type of analysis has, to the best of our knowledge, not been carried out for the Gaspari-Cohn localization scheme. An important message in the covariance estimation literature is that localization ---regardless of how it is employed--- can only be optimal if the target of estimation itself is sparse, and such sparsity assumptions must be made explicit in order to facilitate a rigorous mathematical analysis of the procedure. 
The difficulty of optimal localization in ensemble updates has also been highlighted in \cite{furrer2007estimation},
where the authors derive an optimal localization matrix $\mathsf{M}$ under the unrealistic assumption that $\Cpr$ is a diagonal matrix. 
 
\item Schur-Product Approximations: In the literature on ensemble Kalman methods, a consensus has not been reached on how best to apply localization in practice. The issue here can be sufficiently described by deferring to the linear forward-model setting, i.e. $\mcG(\upr) = A \upr$, in which the Kalman gain  is a central quantity. As mentioned for example in \cite{houtekamer2001sequential}, in a localized update, the Kalman gain operator should in theory be applied to $ \mathsf{M} \circ \hatCpr$, i.e. one should study the quantity 
\begin{equation*}
\msK(\mathsf{M} \circ \hatCpr) =  
(\mathsf{M} \circ \hatCpr) A^\top \bigl(A (\mathsf{M} \circ \hatCpr )A^\top + \Gamma \bigr)^{-1},
\end{equation*}
although their experimental results are based on the more computationally convenient approximation
\begin{align} \label{eq:SchurApprox}
\msK(\mathsf{M} \circ \hatCpr) \approx 
\bigl(\mathsf{M} \circ(\hatCpr A^\top) \bigr) \bigl( \mathsf{M} \circ (A \hatCpr A^\top) + \Gamma\bigr)^{-1},
\end{align}
which, as they mention, is a reasonable approximation in the case that $A$ is diagonal. Subsequently, much of the literature on localization in ensemble Kalman updates has adopted this or similar approximations, as discussed in greater depth in \cite[Section 3.3]{petrie2008localization}. In general, however, approximations made on the Schur product are difficult to justify without strong assumptions on the forward model $\mcG$.

\end{enumerate}

With these issues in mind, we opt to study an alternative, data-driven approach to localization often employed in the high-dimensional covariance estimation literature \cite{bickel2008covariance,  cai2012minimax, cai2012optimal}, where it is referred to as \textit{thresholding}. We ground our analysis in the assumption that the target of estimation belongs to the following soft sparsity matrix class: 
\begin{align}\label{eq:LocalizationMatrixClass}
    \msU_{d_1, d_2}(q, R_q)
    \equiv 
    \biggl\{
        B \in \R^{d_1 \times d_2} : 
        \max_{i \le d_1} \sum_{j=1}^{d_2} |B_{ij}|^q \le R_q \biggr\},
\end{align}
where $q \in [0,1)$ and $R_q > 0$, and write $\msU_{d}(q, R_q)$ in the case $d_1=d_2=d$. In the special case $q=0$, matrices in $\msU_{d_1,d_2}(0, R_0)$ possess rows that have no more than $R_0$ non-zero entries ---a special case of which are banded matrices---  which is the classical hard-sparsity constraint. In contrast, for $q \in (0,1),$ the class $\msU_{d_1,d_2}(q, R_q)$ contains matrices with rows belonging to the $\ell_q$ ball of radius $R_q^q$. This includes matrices with rows that contain possibly many non-zero entries so long as their magnitudes decay sufficiently rapidly, and so is often referred to as a \textit{soft}-sparsity constraint. Importantly, the class $\msU_d(q, R_q)$ is sufficiently rich to capture the motivating intuition that correlations decay with physical distance in a rigorous manner that avoids the optimality issues mentioned above. Structured covariance matrices, such as those belonging to $\msU_{d_1,d_2}(q,R_q)$ are optimally estimated using localized versions of their sample covariances. To this end, we study the localized matrix estimator $B_{\rho_N} \equiv \mcL_{\rho_N} (B)$, where $\mcL_{\rho_N}(u) = u \indicator_{\{|u| \ge \rho_N\}}$ is a localization operator with localization radius $\rho_N$, and which is applied elementwise to its argument $B$. In Section~\ref{sec:withlocalization} we detail how the localization radius $\rho_N$ can be chosen optimally in terms of the parameters of the inverse problem \eqref{eq:IP} and the ensemble size $N$. 

Throughout our analysis, we refrain from using approximations such as the one outlined in \eqref{eq:SchurApprox}; that is, our analysis of localization replaces all non-localized quantities in the original update \eqref{eq:EKIupdate} with their localized counterparts. We introduce the LEKI update:
\begin{align}\label{eq:LEKIupdate}
    \upost_n^{\rho}
    &= \upr_n + \msP(\hatCpr^{\, up}_{\rho_{N,1}}, \hatCpr^{\, pp}_{\rho_{N,2}}) \bigl(y - \mcG(\upr_n) - \eta_n \bigr), 
    \qquad 
    1 \le n \le N,
\end{align}
where $\rho_{N,1}$ and $\rho_{N,2}$ are two, potentially different localization radii. As in the non-localized case, in Subsection~\ref{ssec:mainresults4} we provide finite sample bounds on the deviation of the LEKI update from the mean-field update of \eqref{eq:MFEKIUpdate}, and describe in detail how the additional structure imposed on $\Cpr^{\,up}$ and $\Cpr^{\,pp}$ leads to  improved bounds relative to the non-localized setting. Our second main result of this section, Theorem \ref{thm:LEKI} will be based on new covariance estimation bounds that may be of independent interest, and on a suitable notion of effective dimension that we introduce in Subsection \ref{ssec:newcovthms}. Our theory explains the improved sample complexity that can be achieved by simultaneously exploiting spectral decay and sparsity of the covariance model.     

An important issue that warrants discussion is that of positive-semidefiniteness of the estimator $\widehat{B}_{\rho_N}$ when the target $B$ is a square covariance matrix. In the case of the Schur-product estimator, any localization matrix $\mathsf{M}$ derived from a valid correlation function $\kappa$ is guaranteed to be positive-semidefinite by definition \cite{gaspari1999construction}, and so by the Schur-product Theorem \cite[Theorem 7.5.3]{horn2012matrix} the estimator $\mathsf{M} \circ \widehat{B}$ is positive-semidefinite as well. In contrast, the localization operator $\mcL_{\rho_N}$ thresholds the sample covariance $\widehat{B}$ elementwise and does not in general preserve positive-semidefiniteness. As discussed in \cite{el2008operator, cai2012optimal}, $\widehat{B}_{\rho_N}$ is positive-semidefinite with high probability, but in practice one may opt to use an augmented estimator that guarantees positive-semidefiniteness. We describe this estimator here for completeness: let $\widehat{B}_{\rho_N} = \sum_{j=1}^d \hat{\lambda}_j v_j v_j^\top$ be the eigen-decomposition of $\widehat{B}_{\rho_N}$, so that $\lambda_j, v_j$ are the $j$-th eigenvalue and eigenvector of $\widehat{B}_{\rho_N}$. Consider then the positive-part estimator $\widehat{B}_{\rho_N}^+ \equiv \sum_{j=1}^d (0 \lor \hat{\lambda}_j) v_j v_j^\top$. Clearly then, $\widehat{B}_{\rho_N}^+$ is positive-semidefinite, and furthermore it achieves the same rate as $\widehat{B}_{\rho_N}$ since
\begin{align*}
    \normn{\widehat{B}_{\rho_N}^+ - B} 
    \le \normn{\widehat{B}_{\rho_N}^+ - \widehat{B}_{\rho_N}} 
    +\normn{\widehat{B}_{\rho_N} - B} 
    \le \max_{j: \hat{\lambda}_j < 0}  |\hat{\lambda}_j - \lambda_j|
    +\normn{\widehat{B}_{\rho_N} - B} 
    \le 2\normn{\widehat{B}_{\rho_N} - B},
\end{align*}
where $\lambda_j$ is the $j$-th eigenvalue of $B$. In light of this fact, we abuse notation slightly and assume that $B_{\rho_N}$ is positive-semidefinite throughout this work.

\subsection{Dimension-Free Covariance Estimation Under Soft Sparsity}\label{ssec:newcovthms}

For the covariance estimation problem under (approximate) sparsity, there are estimators that significantly improve upon the sample covariance. In particular, \cite[Chapter 6.5]{wainwright2019high} notes that for sub-Gaussian data the operator-norm covariance estimation error depends logarithmically on the state dimension $d$ for localized estimators, while the error depends linearly on $d$ for the sample covariance. If no sparse structure is assumed, 
the effective dimension $r_2$ defined in $\eqref{eq:effectiveDim}$
characterizes the error of the sample covariance estimator, as described in Proposition~\ref{thm:Koltchinski}. We introduce an analogous notion of effective dimension that is more suitable than $r_2$ in the sparse covariance estimation problem, termed the \textit{max-log effective dimension} and which, for $Q \in \mcS^d_+,$ is given by
\begin{align*}
    r_{\infty}(Q) 
    \equiv 
    \frac{\max_{j \le d} Q_{(j)} \log(j+1)}{ Q_{(1)}}, 
\end{align*}
where $Q_{(1)} \ge Q_{(2)} \ge \ldots \ge Q_{(d)}$ is the decreasing rearrangement of the diagonal entries of $Q$. To the best of our knowledge, this notion of dimension has not been previously considered in the literature, and, as will be shown, refines the rate of covariance estimation under sparsity 
by incorporating intrinsic properties of the underlying matrix, albeit differently to $\eqref{eq:effectiveDim}$.  In particular, $r_{\infty}(Q)$ is small whenever $Q$ exhibits a decay of the ordered elements $Q_{(1)}, Q_{(2)}, \dots$ that is faster than $\log(j+1)$. We use the subscript $\infty$ to highlight that the quantity $r_\infty$ is related to the dimension-free sub-Gaussian maxima result of Lemma~\ref{lem:DimFreeGaussianMax}. Similarly, we use the subscript $2$ to draw the connection between $r_2$ and the sub-Gaussian $2$-norm concentration of Theorem~\ref{thm:SubGaussianConcentration}. Importantly, bounds based on $r_\infty$ will be dimension-free, in the sense that they exhibit no dependence on the state dimension $d$. The next result is our analog of Proposition~\ref{thm:Koltchinski} for estimation
under sparsity 
using the localized sample covariance estimator. Recall that $\Cpr_{(1)}$ denotes the largest element on the diagonal of $\Cpr$, $ \hatCpr_{\rho_N} \equiv \mcL_{\rho_N}(\hatCpr\, )$ denotes the localized sample covariance matrix, and $\msU_d(q,R_q)$ is the sparse matrix class defined in  \eqref{eq:LocalizationMatrixClass}. All proofs in this subsection have been deferred to Appendix~\ref{ssec:B1}.

\begin{theorem}[Covariance Estimation with Localization ---Soft Sparsity Assumption] \label{thm:SoftSparistyCovarianceBound}
 Let $\upr_1,\dots, \upr_N$ be $d$-dimensional i.i.d. sub-Gaussian random vectors with $\E [\upr_1] = \mpr$ and $\tvar[\upr_1] = \Cpr$. Further, assume that $\Cpr \in \msU_d(q, R_q)$ for some $q \in [0,1)$ and $R_q > 0$. For any $t \ge 1$, set 
\begin{align*}
    \rho_N \asymp 
    \Cpr_{(1)}
    \inparen{
    \sqrt{\frac{r_\infty (\Cpr)}{N}}
    \lor 
    \frac{t r_\infty(\Cpr)}{N}
    \lor 
    \sqrt{\frac{t}{N}}
    \lor 
    \frac{t}{N}
    }
\end{align*}
and let  $\hatCpr_{\rho_N} \equiv \mcL_{\rho_N}(\hatCpr_{\rho_N} )$   be the localized sample covariance estimator. There exists a constant $c>0$ such that, with probability at least $1-ce^{-t}$, 
\begin{align*}
    \normn{\hatCpr_{\rho_N} - \Cpr } \lesssim R_q \rho_N^{1-q}.
\end{align*}
\end{theorem}

\begin{remark}[Max-Log Effective Dimension]\label{rem:maxLogEffDim}
    The proof of Theorem~\ref{thm:SoftSparistyCovarianceBound} can be found in Section~\ref{ssec:B1} and, up to the choice of $\rho_N$, follows an identical approach to the standard proof for localized covariance estimators in the literature, for example \cite[Theorem 6.27]{wainwright2019high}. The result depends crucially on the order of the maximum elementwise distance between the sample and true covariance matrices, $\normn{\hatCpr - \Cpr}_{\max}$, which is where our analysis differs from the exiting literature. Our proof utilizes techniques in \cite{koltchinskii2017concentration} combined with the dimension-free sub-Gaussian maxima bound of Lemma~\ref{lem:DimFreeGaussianMax} to obtain a bound in terms of $r_\infty$. In the worst case, for example when $\Cpr = c I_d$ for some constant $c > 0$ so that the ordered diagonal elements of $\Cpr$ exhibit no decay, we recover exactly the standard logarithmic dependence on the state dimension. In particular, when  $N \ge r_\infty(\Cpr) (= \log d)$, 
     Theorem~\ref{thm:SoftSparistyCovarianceBound} matches the result for recovering $\Cpr$ in operator norm in the sub-Gaussian setting over the class $\msU_d(q,R_q)$, as shown in \cite[Theorem 1]{bickel2008covariance}.  
    If the ordered variances exhibit sufficiently fast decay, our upper bound is significantly better. (Recall that in many applications $d \sim 10^9$ and $N \sim 10^2,$ and so the logarithmic dependence on $d$ may play a significant role in determining a sufficient ensemble size.) Importantly, many of the results in the structured covariance estimation literature rely similarly on the maximum elementwise norm, and so our results can be utilized to achieve refined bounds on the estimation error of the localized estimator under structural assumptions on $\Cpr$ that differ from the soft-sparsity assumption considered in this work.
\end{remark}

A result analogous to Theorem~\ref{thm:SoftSparistyCovarianceBound}  holds for cross-covariance estimation under sparsity. 
For a formal statement we refer to Theorem~\ref{thm:CrossCovarianceEffectiveDim} whose proof is based on a deep generic chaining bound for product empirical processes \cite[Theorem 1.12]{mendelson2016upper}. Here we present a cross-covariance estimation result that is specific to the LEKI setting in that it relies on a smoothness assumption on the forward model.

\begin{theorem}[Cross-Covariance Estimation with Localization ---Soft Sparsity Assumption] \label{thm:SoftSparistyCrossCovarianceBound}
    Let $\upr_1,\dots, \upr_N$ be $d$-dimensional i.i.d. sub-Gaussian random vectors with $\E [\upr_1]=\mpr$ and $\tvar[\upr_1] = \Cpr$.  Let $\mcG: \R^d \to \R^k$ be a Lipschitz continuous forward model and assume that  $\Cpr^{\,up} \in \mcU_{d, k}(q_1, R_{q_1})$ and $\Cpr^{\,pu} \in \mcU_{k,d}(q_2, R_{q_2})$ where $q_1,q_2 \in [0,1)$ and $R_{q_1}, R_{q_2}$ are positive constants. For any $t\ge 1$, set 
    \begin{align*}
        \rho_N \asymp 
       (\Cpr_{(1)} \lor \Cpr_{(1)}^{\,pp})
        \inparen{ 
        \inparen{\frac{t}{N} \lor \sqrt{\frac{t}{N}}}
        \inparen{
        \sqrt{r_\infty(\Cpr)}
        \lor 
        \sqrt{r_\infty(\Cpr^{\,pp})}
        }
        \lor 
        \sqrt{    \frac{ r_{\infty}(\Cpr)}{N}   } 
        \sqrt{    \frac{ r_{\infty}(\Cpr^{\,pp})}{N}   }
        },
    \end{align*}
    and let  $\hatCpr^{\,up}_{\rho_N} \equiv  \mcL_{\rho_N}(\hatCpr^{\,up} )$  be the localized sample cross-covariance estimator. There exist positive universal constants $c_1,c_2$ such that, with probability at least $1-c_1e^{-c_2t}$,
    \begin{align*}
        \normn{\hatCpr^{\,up}_{\rho_N} - \Cpr^{\,up} } \lesssim R_{q_1} \rho_N^{1-q_1} 
        \lor 
        R_{q_2} \rho_N^{1-q_2}
        .
    \end{align*}
\end{theorem}

\begin{remark}[Sparsity of the Cross-Covariance]
    To the best of our knowledge, estimation of the cross-covariance matrix under structural assumptions has not been a point of focus in the literature. Indeed, one may implicitly estimate the cross-covariance by applying Theorem~\ref{thm:SoftSparistyCovarianceBound} to the full covariance matrix 
    \begin{align*}
        \begin{bmatrix}
        \Cpr & \Cpr^{\,up}\\
        \Cpr^{\,pu} & \Cpr^{\,pp}
        \end{bmatrix}
    \end{align*}
    of the sub-Gaussian vector $[\upr^\top, \mcG(\upr)^\top]^\top$, and extracting a bound on $\normn{\Cpr_{\rho_N}^{\,up}- \Cpr^{\,up}}$. This approach however requires one to place sparsity assumptions on the full covariance matrix, making the result potentially less useful in practice. That is, one may wish to make structural assumptions on $\Cpr^{\,up}$ and $\Cpr^{\,pp}$ without imposing any restrictions on $\Cpr$, which our result allows for.
\end{remark}

\subsection{Main Results: Approximation of Mean-Field Particle Updates with Finite Ensemble Size} \label{ssec:mainresults4}
In this subsection we state finite ensemble approximation results for EKI and LEKI updates. The main results, Theorems \ref{thm:EKI} and \ref{thm:LEKI}, showcase the dependence on the effective dimension of $C$ and $C^{pp}$ for EKI and on the max-log dimension of these matrices for LEKI. For both algorithms, we study the update of a generic particle $\upr_n$ and the analysis is carried out conditional on both $\upr_n$ and the noise perturbation $\eta_n.$

\begin{theorem}[Approximation of Mean-Field EKI with EKI ---Operator-Norm Bound] \label{thm:EKI}
    Let $y$ be generated according to \eqref{eq:IP} with Lipschitz forward model $\mcG:\R^d \to \R^k$. Let $\upost_n$ and $\upost_n^*$ be the EKI and mean-field EKI updates defined in \eqref{eq:EKIupdate} and \eqref{eq:MFEKIUpdate} respectively. Then, for any $t\ge 1$, there exists universal positive constants $c_1,c_2$ such that, with probability at least $1-c_1 e^{-c_2t}$,
    \begin{align*}
        \normn{ \upost_n - \upost_n^*}_2 
        &\le
        c_1 \inparen{
        \frac{c_2}{N}
        \lor
        \sqrt{\frac{r_2(C)}{N}}
        \lor 
        \frac{r_2(C)}{N}
        \lor 
        \sqrt{\frac{r_2(\Cpr^{\,pp})}{N}}
        \lor 
        \frac{r_2(\Cpr^{\,pp})}{N}
        \lor \sqrt{\frac{t}{N}}
        \lor \frac{t}{N} 
        },
    \end{align*}
    where $c_1=c_1(\normn{y-\mcG(\upr_n) -\eta_n}_2, \normn{\Gamma^{-1}}, \normn{\Cpr}, \normn{\Cpr^{\,up}}, \normn{\Cpr^{\,pp}})$ and for $\upr \sim \mcN(m, \Cpr),$  $c_2=c_2(\normn{\upr_n}_2, \normn{\mpr}_2, \normn{\mcG(\upr_n)}_2, \normn{\E[ \mcG(\upr)]}_2)$. 
\end{theorem}

\begin{remark}[Dependence of Constants on Model Parameters] \label{rem:dependenceconstantsEKI}
The proof of Theorem \ref{thm:EKI} in Appendix~\ref{sec:localizationProofs} gives an explicit expression for the dependence of $c$ on its arguments.  These bounds may be used to establish the sufficient ensemble size to ensure that the EKI update approximates well the mean-field EKI update in the unstructured covariance setting.
\end{remark}

\begin{theorem}[Approximation of Mean-Field EKI with LEKI ---Operator-Norm Bound] \label{thm:LEKI}
    Let $y$ be generated according to \eqref{eq:IP} with Lipschitz forward model $\mcG:\R^d \to \R^k$. Assume that $\Cpr^{\,up} \in \msU_{d,k}(q_1, R_{q_1})$, $\Cpr^{\,pu} \in \msU_{k,d}(q_2, R_{q_2})$ and $\Cpr^{\,pp} \in \msU_{k}(q_3, R_{q_3})$ for $q_1, q_2, q_3 \in [0,1)$, and positive constants $R_{q_1}, R_{q_2}, R_{q_3}$. Let $\upost_n^{\rho}$ and $\upost_n^*$ be the LEKI and mean-field EKI updates outlined in \eqref{eq:LEKIupdate} and \eqref{eq:MFEKIUpdate} respectively. For any $t\ge 1$, set 
    \begin{align*}
        \rho_{N,1}=\rho_{N,2} \asymp 
        \frac{c_1}{N}
        +
       (\Cpr_{(1)} \lor \Cpr^{\,pp}_{(1)})
        \inparen{ 
        \inparen{\frac{t}{N} \lor \sqrt{\frac{t}{N}}}
        \inparen{
        \sqrt{r_\infty(\Cpr)}
        \lor 
        \sqrt{r_\infty(\Cpr^{\,pp})}
        }
        \lor 
        \sqrt{    \frac{ r_{\infty}(\Cpr)}{N}   } 
        \sqrt{    \frac{ r_{\infty}(\Cpr^{\,pp})}{N}   }
        },
    \end{align*}
    and 
    \begin{align*}
        \rho_{N, 3} \asymp
        \frac{c_2}{N}
        +
        \Cpr_{(1)}^{\,pp}
        \inparen{
        \sqrt{\frac{r_\infty (\Cpr^{\,pp})}{N}}
        \lor 
        \sqrt{\frac{t}{N}}
        \lor 
        \frac{t}{N}
        \lor 
        \frac{t r_\infty(\Cpr^{\,pp})}{N}
        },
    \end{align*}
    where $c_1=c_1(\normn{\upr_n}_{\infty}, \normn{\mpr}_{\infty}, \normn{\mcG(\upr_n)}_{\infty}, \normn{\E[\mcG(\upr)]}_{\infty})$ and $c_2=c_2(\normn{\mcG(\upr_n)}_{\infty}, \normn{\E[\mcG(\upr)]}_{\infty})$, with $\upr \sim \mcN(\mpr,\Cpr).$ 
    There exist positive universal constants $c_3,c_4$ such that, with probability at least $1-c_3e^{-c_4t}$, 
    \begin{align*}  
        \normn{ \upost_n^{\rho} - \upost_n^* }_2 
        \le
        c_5
        (R_{q_1} \rho_{N,1}^{1-q_1} 
        \lor 
        R_{q_2} \rho_{N,2}^{1-q_2}
        \lor 
        R_{q_3} \rho_{N,3}^{1-q_3}),
    \end{align*}
    where $c_5 = c_5(\normn{y-\mcG(\upr_n) -\eta_n}_2 , \normn{\Gamma^{-1}}, \normn{\Cpr^{\,up}}).$
\end{theorem}

\begin{remark}[Dependence of Constants on Model Parameters]
The proof of Theorem \ref{thm:LEKI} in Appendix \ref{sec:localizationProofs} gives an explicit expression for the dependence of $c$ on its arguments.  As discussed in Remark \ref{rem:dependenceconstantsEKI}, these bounds may be used to establish the sufficient ensemble size to ensure that the LEKI update approximates well the mean-field EKI update in the structured covariance setting.
\end{remark}

\begin{remark}[On the Soft-Sparsity Assumptions]
    Importantly, Theorem~\ref{thm:LEKI} makes no assumptions on the covariance matrix $\Cpr$, and so can be used even in cases where $\Cpr$ is dense, but the covariances $\Cpr^{\,up}$, $\Cpr^{\,pu},$ and $\Cpr^{\,pp}$ can be reasonably assumed to be sparse. In the case that sparsity assumptions on $\Cpr$ are appropriate, then an interesting question is: what (explicit) assumptions on $\mcG$ ensure sparsity of $\Cpr^{\,up}$, $\Cpr^{\,pu}$, and $\Cpr^{\,pp}$? We provide here two simple arguments that may provide some insight. Throughout, $c_1,c_2,c_2, c_4,c_5$ are arbitrary positive constants independent of both state and observation dimensions $d$ and $k$, and $q\in [0,1)$.
    \begin{enumerate}
        \item Suppose $\Cpr \in \msU_{d}(q,c_1)$ and $\E[D \mcG]^\top \in \msU_{d,k}(q,c_2)$. Then there exists $c_3$ such that $\Cpr^{\, up} \in \msU_{d,k}(q, c_3)$. We provide a formal statement of this result in Lemma~\ref{lem:CrossCovarianceStein}. 
         Similarly, if $\E[D \mcG] \in \msU_{k,d}(q,c_4)$, then there exists $c_5$ such that $\Cpr^{\, pu} \in \msU_{k,d}(q, c_5)$. 
        The assumptions on the expected Jacobian $\E[D \mcG]$ can be understood as the requirements that, in expectation:
        \begin{enumerate}
            \item Any coordinate function $\mcG_j$ of $\mcG$ depends on its input $\upr$ only through a subset of $\upr$ whose size does not grow with $k$ nor $d$.
            \item Any state coordinate $\upr_j$ of $\upr$ is acted on only by a subset of the coordinate-functions of $\mcG$ whose size does not grow with $k$ nor $d$.
        \end{enumerate}
        For example, a Jacobian that is banded in expectation would satisfy these two properties.

        \item Suppose $\Cpr \in \msU_{d}(q,c_1).$ Then there exists $c_2$ such that $\Cpr^{\,pp} \in \msU_{k}(q,c_2)$ whenever $\mcG(u) = Au$ is a linear map with $A \in \msU_{k,d}(q, c_3)$ and $A^\top \in \msU_{d,k}(q,c_4),$ i.e. whenever $A$ has both rows and columns that are sparse. This condition holds, for example, for banded $A$.  We provide a formal statement of this result in Lemma~\ref{lem:BSBsparsity}. 
    \end{enumerate}
    The two arguments above indicate that if $\mcG$ acts on local subsets of $\upr$, which holds for instance for convolution or moving average operators, then one can expect the sparsity of $\Cpr$ to carry on to $\Cpr^{\,up},$ $\Cpr^{\,pu},$ and $\Cpr^{\,pp}$. 
\end{remark}

\begin{remark}[Comparison to the Literature]
    Although the focus of this subsection is the LEKI, it is useful to compare our Theorems~\ref{thm:EKI} and \ref{thm:LEKI} to existing results for the performance of ensemble based algorithms with localization. In this regard, our results are closest to those of \cite{tong2018performance}, which shows that an ensemble that scales with the logarithm of the state dimension times a localization radius suffices for good performance of the localized EnKF (LEnKF). They study performance over multiple time steps and linear dynamics under a stability assumption which enforces control over the model matrices as well as a sparse ($q=0$) structure of the underlying true covariance matrices. They consider \textit{domain localization} whereas we study \textit{covariance localization}. In contrast to our results, \cite{tong2018performance} employs covariance localization and utilizes a Schur-product localization scheme in which elements whose indices are beyond a certain bandwidth are set to zero, whereas we study localization via thresholding (recall our discussion comparing these two approaches in Subsection~\ref{sssec:LEKI}). Consequently, our required localization radius is in terms of the max-log effective dimension whereas theirs is in terms of the bandwidth of the underlying covariance matrix. Our results are dimension-free in that they do not rely on the state dimension $d$, and so as noted in Remark~\ref{rem:maxLogEffDim}, our bounds can have significantly better than logarithmic dependence on dimension. Our setting also differs from \cite{tong2018performance} in that our dynamics are allowed to be nonlinear, and our prior ensemble can be sub-Gaussian as opposed to Gaussian. Related to this point is that the analysis in \cite{tong2018performance} does not account for noise introduced from adding perturbations to the ensemble update, which is justified by a law of large numbers argument; however in the non-asymptotic and nonlinear settings, it is likely that one must account for this noise  especially  when considering the covariance between the current ensemble and the perturbation noise at a given iteration of the algorithm. We view it as an important avenue to extend the results of this subsection to a multi-step analysis, and a particularly important question is whether dimension-free control of the LEnKF can be rigorously shown utilizing a combination of our results and those of \cite{tong2018performance}.  The LEKI has also been recently studied in \cite{tong2022localized} under a nonlinear, multi-step setting. The authors study convergence of the iterates to a global minimizer and the rate of collapse of the ensemble. 
    They argue that localization is a remedy for the ``subspace property'' of the EKI, which refers to the fact that ensembles at any given iteration live in the linear subspace spanned by the initial ensemble, which cannot capture the true state if $N < d$. Their analysis differs from ours in that they study the continuous-time setting whereas we analyze discrete time updates as implemented in practice. 
    Further, while they discuss that the size of the ensemble may be much smaller than the state dimension, as well as illustrate this with simulations, they do not provide an explicit characterization of the sufficient ensemble size. Our results also show that the LEKI is close to the mean field version of the problem, which is not considered in their set-up. An interesting open question is whether the results of this section can be used in conjunction with results in \cite{tong2022localized} to provide a sufficient ensemble size for LEKI over multiple iterations.
\end{remark}

\section{Conclusions, Discussion, and Future Directions}\label{sec:conclusions}
This paper has introduced a non-asymptotic approach to the study of ensemble Kalman methods. Our theory explains why these algorithms may be accurate provided that the ensemble size is larger than a suitable notion of effective dimension, which may be dramatically smaller than the state dimension due to spectrum decay and/or approximate sparsity. Our non-asymptotic results in Section \ref{sec:withoutlocalization} tell apart PO and SR updates for posterior approximation, and our results in Section \ref{sec:withlocalization} demonstrate the potential advantage of using localization in sequential-optimization algorithms.

As discussed in Subsection~\ref{sssec:LEKI}, localization is also often used in posterior-approximation algorithms. For instance, one may define a localized PO update by
\begin{align}\label{eq:locPOupdate}
\begin{split}
   \hatmpost &= \msM(\hatmpr, \hatCpr_{\rho_N}) - \msK(\hatCpr_{\rho_N}) \bareta,  \\
   \hatCpost &= \msC(\hatCpr_{\rho_N}) + \whatO_{\rho_N},
\end{split}
\end{align}
where $\whatO_{\rho_N}$ is defined replacing $\hatCpr$ with $\hatCpr_{\rho_N}$ in \eqref{eq:offset}.
Similarly, one may define a localized SR update by
\begin{align}\label{eq:locSRupdate}
\begin{split}
    \hatmpost&= \msM(\hatmpr, \hatCpr_{\rho_N}),\\
    \hatCpost &= \msC(\hatCpr_{\rho_N}).
    \end{split}
\end{align}
It is then natural to ask if localized PO and SR updates can yield better approximation of the posterior mean and covariance than those without localization in Theorems~\ref{lem:meanwithoutlocExpectation} and \ref{lem:covwithoutlocExpectation}. The answer for the posterior mean seems to be negative.

To see why, consider for intuition that we are given a random sample $X_1,\dots, X_N$ from a normal distribution with mean $\mu^X$ and covariance $\Sigma^X$ with the objective to estimate $\mu^X$. Standard results, see e.g. \cite[Example 1.14]{lehmann2006theory}, show that the sample mean $\barX$ is minimax optimal for $\ell_2$-loss regardless of whether or not $\Sigma^X$ is known. In other words, the minimax rate of estimating $\mu^X$ can be achieved without making use of information regarding $\Sigma^X$. It follows then that placing assumptions on $\Sigma^X$ can lead to impressive improvements in the covariance estimation problem (as shown in Section~\ref{sec:withlocalization}) but cannot be expected to affect the mean estimation problem. Similarly, in our inverse problem setting, sparsity assumptions on the prior covariance $\Cpr$ cannot be expected to translate into a better bound on $\normn{\hatmpost - \mpost}_2$:  this quantity is a function of both the covariance deviation $\normn{ \hatCpr_{\rho_N} - \Cpr}$ and the prior mean deviation $\normn{\hatmpr - \mpr}_2$ and since the latter is unaffected it dominates the overall bound, yielding an error bound of the same order as that in Theorem~\ref{lem:meanwithoutlocExpectation}. As discussed in Remark~\ref{rem:MultiStep}, a potential avenue for future investigation is to utilize techniques introduced in this manuscript to study alternative localization schemes in the posterior approximation setting, such as \textit{domain localization} considered in \cite{tong2018performance}. In short, covariance localization as defined in \eqref{eq:locSRupdate} does not lead to improved bounds for the posterior-approximation problem.

Similar issues to those arising in the estimation of the posterior mean affect the analysis of the localized offset $\whatO_{\rho_N}$, and we therefore do not expect improvement on the bound in Theorem~\ref{lem:covwithoutlocExpectation} for covariance estimation with the localized PO update. We note, however, that for localized SR it is possible to derive an analog to the high probability version of Theorem~\ref{lem:covwithoutlocExpectation} (see Theorem~\ref{th:covwithoutloc})  with an improved error bound, which we present in Theorem~\ref{th:SRcovwithloc}. 

Our discussion here should not be taken to imply that localization in posterior-approximation algorithms is not useful; it is plausible that localization in one step of the algorithm can lead to improved bounds in later steps, and we leave this multi-step analysis of localized posterior approximation ensemble updates as an important line for future work. A related phenomenon is known to occur in sequential Monte Carlo, where a proposal density that may be optimal for one step of the filter may not be optimal over multiple steps \cite{agapiou2017importance}. Another interesting direction for future study is the non-asymptotic analysis of ensemble Kalman methods for likelihood approximations in state-space models \cite{chen2021auto}. Finally, we envision that the non-asymptotic approach set forth here may be adopted to design and analyze new multi-step methods for posterior-approximation and sequential-optimization in inverse problems and data assimilation.

\section*{Acknowledgments}
DSA is thankful to the National Science Foundation for their support through the grants NSF DMS-2027056 and NSF DMS-2237628, to the BBVA Foundation for the Jos\'e Luis Rubio de Francia start-up grant, and to the Department of Energy for funding DOE DE-SC0022232. The authors are grateful to Jiaheng Chen, Subhodh Kotekal, Yandi Shen, and Nathan Waniorek for many helpful discussions. The authors are also thankful to the anonymous reviewers for their insightful suggestions that improved the manuscript.

\bibliographystyle{plain}
\bibliography{references}

\renewcommand{\theHsection}{A\arabic{section}}
\begin{appendix}

\section*{\Large Appendix}
We provide proofs of all theorems in the main body. We will use the following result extensively and summarise it here for brevity. Given events $E_1,\dots, E_J$ that each occur with probability at least $1-ce^{-t}$, where $t \ge 1$ and $c>0$ is a universal constant that may be different for each event, then 
    \begin{align*}
        \P \inparen{\bigcap_{j=1}^J E_j}
        = 1 - \P \inparen{\bigcup_{j=1}^J \bar{E}_j}  
        \ge  1- \sum_{j=1}^J \P(\bar{E}_j)
        \ge 1-ce^{-t}.
    \end{align*}

\section{Proofs: Section \ref{sec:withoutlocalization}} \label{sec:proofs}
This appendix contains the proofs of all the theorems in Section \ref{sec:withoutlocalization}. Background results on covariance estimation are reviewed in Subsection \ref{sssec:A1} and the continuity and boundedness of the Kalman gain, mean-update, covariance-update, and nonlinear gain-update operators are summarized in Subsection \ref{sssec:A2}. These preliminary results are used in Subsection \ref{sssec:A3} to establish our main theorems. 

\subsection{Preliminaries: Concentration and Covariance Estimation}\label{sssec:A1}

\begin{theorem}[{Sub-Gaussian Norm Concentration, \cite[Exercise 6.3.5]{vershynin2018high}}] \label{thm:SubGaussianConcentration}
    Let $X$ be a $d$-dimensional sub-Gaussian random vector with $\E[X] = \mu^X$, $\tvar [X]=\Sigma^X$. Then, for any $t \ge 1,$ with probability at least $1-ce^{-t}$ it holds that
    \begin{align*}
        \| X- \mu^X\|_2 \lesssim  
        \sqrt{\ttrace(\Sigma^X)} + \sqrt{t\normn{\Sigma^X} } 
        \lesssim 
        \sqrt{\normn{\Sigma^X} (r_2(\Sigma^X) \lor t )} \,.
    \end{align*}
\end{theorem}

    \begin{proof}[Proof of Proposition~\ref{thm:Koltchinski}]
        For $n=1,\dots,N$, let $\upr_n = Z_n + m$, where $Z_n$ is a centered sub-Gaussian random vector with $\tvar[Z_n]=\Cpr.$ Then we may write 
        \begin{align*}
            \hatCpr
            &= \frac{1}{N-1} \sum_{n=1}^N 
            (Z_n - \barZ)
            (Z_n - \barZ)^\top
            \asymp 
            \frac{1}{N} 
            \sum_{n=1}^N 
            Z_n Z_n^\top -\barZ \barZ^\top
            \equiv \hatCpr^0 -\barZ \barZ^\top.
        \end{align*}
        Therefore, 
        \begin{align*}
            \normn{\hatCpr - \Cpr}
            &\le \normn{\hatCpr^0 - \Cpr}
            + \normn{\barZ \barZ^\top}
            = \normn{\hatCpr^0 - \Cpr}
            + \normn{\barZ}_2^2.
        \end{align*} 
        Let $E_1$ denote the event on which 
        \begin{align*}
            \normn{\hatCpr^0 - \Cpr} 
            \lesssim 
            \normn{\Cpr} \inparen{ \sqrt{\frac{r_2(\Cpr)}{N}} \lor \frac{r_2(\Cpr)}{N} 
            \lor \sqrt{\frac{t}{N}} \lor \frac{t}{N}},
        \end{align*}
        and $E_2$ the event on which 
        \begin{align*}
            \normn{\barZ}_2^2 
            \lesssim 
            \normn{\Cpr} \inparen{\frac{r_2(\Cpr)}{N} \lor \frac{t}{N}}.
        \end{align*}
        Then by Theorem 9 of \cite{koltchinskii2017concentration}, $\P(E_1) \ge 1-e^{-t}$, and by Theorem~\ref{thm:SubGaussianConcentration}, $\P(E_2) \ge 1-e^{-t}$. Therefore, the result holds on $E_1 \cap E_2,$ which has probability at least $1-ce^{-t}$. 
    \end{proof}

\begin{lemma} [Sample Covariance Operator Norm Bound] \label{lem:ForecastCovarianceBound}
    Let $\upr_1,\dots, \upr_N$ and $\hatCpr$ be as in Proposition~\ref{thm:Koltchinski}. Then, for any $t \ge 1,$ it holds with probability at least $1-ce^{-t}$ that
    \begin{align*}
        \normn{\hatCpr}
        \lesssim 
        \normn{\Cpr} \inparen{1 \lor \frac{r_2(\Cpr)}{N} \lor \frac{t}{N} }.
    \end{align*}
\end{lemma}
\begin{proof}
    By the triangle inequality
    $
        \normn{\hatCpr}
        \le \normn{\hatCpr - \Cpr} + \normn{\Cpr}.
    $
    The result follows by Proposition~\ref{thm:Koltchinski} noting that, for any $x \ge 0$, $1 \lor \sqrt{x} \lor x = 1 \lor x$.
\end{proof}

\begin{lemma}[Cross-Covariance Estimation ---Unstructured Case] \label{lem:ForecastNoiseCovarianceBound}
    Let $\upr_1,\dots, \upr_N$ be $d$-dimensional i.i.d. sub-Gaussian random vectors with $\E[ \upr_1] = m$ and $\tvar[\upr_1] = \Cpr$. Let $\eta_1,\dots,\eta_N$ be $k$-dimensional i.i.d. sub-Gaussian random vectors with $\E [\eta_1] = 0$ and $\tvar[\eta_1]=\Gamma$, and assume that the two sequences are independent. Consider the estimator
    \begin{align*}
        \hatCpr^{\, \upr\eta} = \frac{1}{N-1} \sum_{n=1}^{N} 
        (\upr_n - \hatmpr)(\eta_n  - \bareta)^{\top}
    \end{align*}
    of the cross-covariance $\Cpr^{\upr\eta} \equiv \E \bigl[ (\upr_1-m)\eta_1^\top \bigr].$ Then there exists a constant $c$ such that, for all $t \ge 1,$ it holds with probability at least $1-ce^{-t}$ that
    \begin{align*}
        \normn{\hatCpr^{\, \upr\eta} - \Cpr^{\, \upr\eta}} 
        \lesssim
        (\normn{\Cpr} \lor \normn{\Gamma}) 
        \inparen{
        \sqrt{\frac{r_2(\Cpr)}{N}} 
        \lor 
        \frac{r_2(\Cpr)}{N} 
        \lor  
        \sqrt{\frac{r_2(\Gamma) }{N}} 
        \lor
        \frac{r_2(\Gamma)}{N} 
        \lor 
        \sqrt{\frac{t}{N}}
        \lor 
        \frac{t}{N}}.
    \end{align*}
\end{lemma}
\begin{proof}
    First, we note that 
    \begin{align*}
        \hatCpr^{\, \upr\eta} 
        \asymp 
        \frac{N-1}{N} 
        \inparen{
            \frac{1}{N} \sum_{n=1}^{N} 
            (\upr_n  - \hatmpr)(\eta_n  - \bareta)^{\top}
            }
        \equiv \frac{N-1}{N} \wtildeC^{\, \upr\eta},
    \end{align*}
    and so it suffices to prove the claim for the biased sample covariance estimator, which we denote by $\wtildeC^{\, \upr\eta}$. Letting $Z_n = \upr_n - m$, it follows that 
    \begin{align} \label{eq:ForecastnoiseCovarianceBoundT1}
        \normn{\wtildeC^{\upr\eta}} 
        =
        \norm{\frac{1}{N}\sum_{n=1}^N Z_n \eta_n^{\top} -  \barZ \bareta^\top}
        \le 
        \norm{\frac{1}{N}\sum_{n=1}^N Z_n \eta_n^{\top}} + 
        \normn{\barZ \bareta^\top}.
    \end{align}
    For the second term in the right-hand side of \eqref{eq:ForecastnoiseCovarianceBoundT1}, let $E_1$ denote the event on which 
    \begin{align*}
        \normn{\barZ}_2  \lesssim \sqrt{\normn{\Cpr} \inparen{\frac{r_2(\Cpr)}{N} \lor \frac{t}{N}}},
    \end{align*}
    and $E_2$ the event on which 
    \begin{align*}
        \normn{\bareta}_2  \lesssim 
        \sqrt{\normn{\Gamma} \inparen{\frac{r_2(\Gamma)}{N} \lor \frac{t}{N}}},
    \end{align*}
    each of which have probability at least $1-e^{-t}$ by Theorem~\ref{thm:SubGaussianConcentration}. Therefore, the event $E_1 \cap E_2$ occurs with probability at least $1-ce^{-t}$, and on which it follows that 
    \begin{align*}
        \normn{\barZ \bareta^\top}
        = \normn{\barZ}_2 \normn{\bareta}_2
        &\lesssim 
        (\normn{\Cpr} \lor \normn{\Gamma} )
        \inparen{
        \frac{r_2(\Cpr)}{N}
        \lor
        \frac{r_2(\Gamma)}{N}
        \lor 
        \frac{t}{N}
        },
    \end{align*}
    where the inequality follows since $\sqrt{ab} \lesssim a\lor b$ for $a,b \ge 0$. To control the first term in the right-hand side of \eqref{eq:ForecastnoiseCovarianceBoundT1}, we define the vector  
    \begin{align*}
        W_n = \begin{bmatrix}
            Z_n \\ \eta_n
        \end{bmatrix}
        \in \R^{ d + k}, \qquad 1 \le n \le N,
    \end{align*}
    and note that $W_1, \dots, W_N$ is an i.i.d. sub-Gaussian sequence with $\E [W_1] = [m^\top, 0^\top_{k}]^\top$ and variance $\Cpr^W = \tdiag(\Cpr, \Gamma)$. Let $E_3$ denote the event on which 
    
    \begin{align*}
        \norm{\frac{1}{N} \sum_{n=1}^N W_n W_n^\top - \Cpr^W}
        &\lesssim \normn{\Cpr^W} 
        \inparen{ 
        \sqrt{\frac{r_2(\Cpr^W)}{N}}
        \lor 
        \frac{r_2(\Cpr^W)}{N} 
        \lor 
        \sqrt{\frac{t}{N}}
        \lor 
        \frac{t}{N}
        }\\
        &\lesssim (\normn{\Cpr} \lor \normn{\Gamma}) 
        \inparen{ 
        \inparen{\sqrt{ \frac{\ttrace(\Cpr)}{N \normn{\Cpr}}} + 
        \sqrt{\frac{\ttrace(\Gamma)}{N \normn{\Gamma}}}} \lor 
        \frac{\ttrace(\Cpr) + \ttrace(\Gamma)}{ N (\normn{\Cpr} \lor \normn{\Gamma})}\lor 
        \sqrt{\frac{t}{N}}
        \lor 
        \frac{t}{N}
        }\\
        &\lesssim (\normn{\Cpr} \lor \normn{\Gamma}) 
        \inparen{
        \inparen{
        \sqrt{ \frac{r_2(\Cpr)}{N}} + 
        \sqrt{\frac{r_2(\Gamma) }{N}}} 
        \lor 
        \inparen{
        \frac{r_2(\Cpr)}{N} + 
        \frac{r_2(\Gamma)}{N} 
        } 
        \lor
        \inparen{
        \sqrt{\frac{t}{N}}
        \lor 
        \frac{t}{N}
        }
        }\\
        &
        \lesssim 
        (\normn{\Cpr} \lor \normn{\Gamma}) 
        \inparen{
        \sqrt{\frac{r_2(\Cpr)}{N}} 
        \lor 
        \frac{r_2(\Cpr)}{N} 
        \lor  
        \sqrt{\frac{r_2(\Gamma) }{N}} 
        \lor
        \frac{r_2(\Gamma)}{N}\lor 
        \sqrt{\frac{t}{N}}
        \lor 
        \frac{t}{N}
        }.
    \end{align*}
    By Proposition~\ref{thm:Koltchinski}, it holds for any $t\ge 1$ that $\P(E_3) \ge 1-e^{-t}$.
   Note that we can express 
    \begin{align*}
        \mcP \equiv \frac{1}{N} \sum_{n=1}^N W_n W_n^\top - \begin{bmatrix}
            \Cpr & O \\  O & \Gamma
        \end{bmatrix}
        =
        \begin{bmatrix}
        N^{-1} \sum_{n=1}^N Z_n Z_n^\top - \Cpr & N^{-1} \sum_{n=1}^N Z_n\eta_n^\top \\  
        N^{-1} \sum_{n=1}^N \eta_n Z_n^\top &  N^{-1} \sum_{n=1}^N\eta_n \eta_n^\top - \Gamma
        \end{bmatrix},
    \end{align*}
    and that
    \begin{align*}
        \norm{\frac{1}{N} \sum_{n=1}^N Z_n\eta_n^\top}
        = \normn{E_{11} \mcP E_{12}}
        \le \normn{E_{11}} \normn{\mcP} \normn{E_{12}}
        =\normn{\mcP},
    \end{align*}
    where $E_{11}, E_{12}$ are block \textit{selection} matrices that pick the relevant sub-block matrix of $\mcP$. Therefore, it holds on $E_3$ that 
    \begin{equation*}
        \norm{\frac{1}{N} \sum_{n=1}^N Z_n\eta_n^\top} 
        \lesssim 
        (\normn{\Cpr} \lor \normn{\Gamma}) 
        \inparen{
        \sqrt{\frac{r_2(\Cpr)}{N}} 
        \lor 
        \frac{r_2(\Cpr)}{N} 
        \lor  
        \sqrt{\frac{r_2(\Gamma) }{N}} 
        \lor
        \frac{r_2(\Gamma)}{N} 
        \lor 
        \sqrt{\frac{t}{N}}
        \lor 
        \frac{t}{N}}.
    \end{equation*}    
    The final result follows by noting that the intersection $E_1\cap E_2 \cap E_3$ has probability at least $1-ce^{-t}$.
\end{proof}

\subsection{Continuity and Boundedness of Update Operators}\label{sssec:A2}
The next three lemmas, shown in \cite{kwiatkowski2015convergence}, ensure the continuity and boundedness of the Kalman gain, mean-update, and covariance-update operators introduced in Section \ref{sec:withoutlocalization}. We include them here for completeness. Lemma \ref{lem:EKIopCtyBdd} below establishes similar properties for the nonlinear gain-update operator  introduced in Section \ref{sec:withlocalization}.  

\begin{lemma}[{Continuity and Boundedness of Kalman Gain Operator \cite[Lemma 4.1 \& Corollary 4.2]{kwiatkowski2015convergence}}] \label{lem:KalmanOpCtyBdd} 
Let $\msK$ be the Kalman gain operator defined in \eqref{eq:KalmanGainOperator}. Let $P, Q \in \mcS_+^d$, $\Gamma \in \mcS_{++}^k,$ and $A \in \R^{k \times d}$. The following hold:
\begin{align*}
\begin{split}
        \normn{\msK(Q) - \msK(P)} 
        &\le 
        \normn{Q - P} 
        \normn{A} \normn{\Gamma^{-1}}
        \Bigl(1 + 
        \min \inparen{\normn{P}, \normn{Q}} 
        \normn{A}^2 \normn{\Gamma^{-1}} \Bigr),\\
        \normn{\msK(Q)} &\le \normn{Q} \normn{A} \normn{\Gamma^{-1}},\\
        \normn{I-\msK(Q) A} 
        &\le 1+\normn{Q} \normn{A}^2 \normn{\Gamma^{-1}}.
\end{split}
     \end{align*}
\end{lemma}

\begin{lemma}[{Continuity and Boundedness of Mean-Update Operator 
\cite[Corollary 4.3 \& Lemma 4.7]{kwiatkowski2015convergence}}] \label{lem:MeanOpCtyBdd}
Let $\msM$ be the mean-update operator defined in \eqref{eq:MeanOperator}. Let $P, Q \in \mcS_+^d$, $\Gamma \in \mcS_{++}^k$, $A \in \R^{k \times d}$, $y \in \R^k,$ and $m,m' \in \R^d$. The following hold:
 \begin{align*}
 \begin{split}
        \norm{\msM(m,Q)} &\le \norm{m}
        + \normn{Q}\normn{A} \normn{\Gamma^{-1}} \norm{y-Am}_2,\\
        \norm{\msM(m, Q)- \msM(m', P)} &\le 
        \norm{m-m'} \bigl(1 + \normn{A}^2 \normn{\Gamma^{-1}} \normn{Q}\bigr)  \\ & + 
        \normn{Q-P} \normn{A} \normn{\Gamma^{-1}} 
        \bigl(1 + \normn{A}^2 \normn{\Gamma^{-1}} \normn{P} \bigr) 
        \norm{y-Am'}_2.
\end{split}        
     \end{align*}
\end{lemma}

\begin{lemma}[{Continuity and Boundedness of Covariance-Update Operator
\cite[Lemma 4.4 \& Lemma 4.6]{kwiatkowski2015convergence}}] \label{lem:CovarOpCtyBdd}
Let $\msC$ be the covariance-update operator defined in \eqref{eq:CovarianceOperator}. Let $P, Q \in \mcS_+^d$, $\Gamma \in \mcS_{++}^k$, $A \in \R^{k \times d}$, $y \in \R^k,$ and $m,m' \in \R^d$. The following hold:
  \begin{align*}
  \begin{split}
        \normn{\msC(Q) - \msC(P)} 
        &\le \normn{Q - P} \Bigl( 1 + \normn{A}^2 \normn{\Gamma^{-1}} (\normn{Q}+\normn{P}) + \normn{A}^4 \normn{\Gamma^{-1}}^2 \normn{Q}\normn{P} \Bigr),\\
        0 &\preccurlyeq \msC(Q) \preccurlyeq Q,\\
        \normn{\msC(Q)} &\le \normn{Q}.
    \end{split}    
     \end{align*}
\end{lemma}

\begin{lemma}[Continuity and Boundedness of Nonlinear Gain-Update Operator]\label{lem:EKIopCtyBdd}
Let $\msP$ be the nonlinear gain-update operator defined in \eqref{eq:NonlinearGainOperator}.  Let $P, \tildeP \in \R^{d \times k},$ $Q, \tildeQ \in \mcS_+^k$,  and $\Gamma \in \mcS_{++}^k.$ The following hold: 
\begin{align*}
    \normn{\msP(P,Q) - \msP(\tildeP,\tildeQ)}
    &\le 
    \normn{\Gamma^{-1}} \normn{P - \tildeP} 
    + \normn{\Gamma^{-1}}^2 \normn{P}  \normn{Q - \tildeQ},\\
    \normn{\msP(P,Q)} 
    &\le  \normn{\Gamma^{-1}}\normn{P} + \normn{\Gamma^{-1}}^2 \normn{Q}.
\end{align*}
\end{lemma}
\begin{proof}
    The proof follows in similar style to Lemma 4.1 in \cite{kwiatkowski2015convergence}. We note that 
    \begin{align*}
        \normn{P(Q+\Gamma)^{-1} - \tildeP(\tildeQ+\Gamma)^{-1}}
        &\le 
        \normn{P(Q+\Gamma)^{-1} - P(\tildeQ+\Gamma)^{-1}} + \normn{P(\tildeQ+\Gamma)^{-1} - \tildeP(\tildeQ+\Gamma)^{-1}}\\
        &\le 
        \normn{P} \normn{(Q+\Gamma)^{-1} - (\tildeQ+\Gamma)^{-1}} + 
        \normn{\tildeP - P} \normn{(Q+\Gamma)^{-1}}.
    \end{align*}
    Since $\Gamma \succ 0$ and $Q \succeq 0$, it holds that $Q+\Gamma \succeq \Gamma$ and so $(Q+\Gamma)^{-1} \preccurlyeq \Gamma^{-1}$, which in turn implies $\normn{(Q+\Gamma)^{-1}} \le \normn{\Gamma^{-1}}$. Further,
    \begin{align*}
        \normn{(Q+\Gamma)^{-1} - (\tildeQ+\Gamma)^{-1}} 
        &= \normn{\Gamma^{-1/2} 
        [(\Gamma^{-1/2}Q\Gamma^{-1/2}+I)^{-1} - 
        (\Gamma^{-1/2}\tildeQ\Gamma^{-1/2}+I)^{-1}]\Gamma^{-1/2}}  \\
        &\le \normn{\Gamma^{-1}} 
        \normn{(\Gamma^{-1/2}Q\Gamma^{-1/2}+I)^{-1} 
        - (\Gamma^{-1/2}\tildeQ\Gamma^{-1/2}+I)^{-1}}\\
        &\le \normn{\Gamma^{-1}} 
        \normn{\Gamma^{-1/2}Q\Gamma^{-1/2} -
        \Gamma^{-1/2}\tildeQ\Gamma^{-1/2}}
        \\
        &\le \normn{\Gamma^{-1}}^2 \normn{Q - \tildeQ},
    \end{align*}
    where the second to last equality follows by the fact that $\normn{(I+A)^{-1}- (I+B)^{-1} } \le \normn{B-A}$ for $A,B\in \mcS^k_+$.
   To prove the pointwise boundedness of $\msP$, take $\tildeP$ to be the $d \times k$ matrix of zeroes, and $\tildeQ$ to be the $k \times k$ matrix of zeroes, and plug these values into the continuity bound. 
\end{proof}

\subsection{Proof of Main Results in Section \ref{sec:withoutlocalization}}
\label{sssec:A3}
 
    \begin{theorem} [Posterior Mean Approximation with Finite Ensemble ---High Probability Bound]\label{th:meanwithoutloc}
    Consider the PO and SR ensemble Kalman updates given by \eqref{eq:POupdate} and \eqref{eq:SRupdate}, respectively, leading to an estimate $\hatmpost$ of the posterior mean $\mpost$ defined in \eqref{eq:KFmeancovariance}. Set $\phi=1$ for the PO update and $\phi=0$ for the SR update. Then there exists a constant $c$ such that, for all $t \ge 1,$ it holds with probability at least $1-ce^{-t}$ that
    \begin{align*}
        \normn{\hatmpost - \mpost}_2
        &\lesssim 
        (\normn{\Cpr}^{1/2} \lor \normn{\Cpr}^2)
        (\normn{A} \lor \normn{A}^4)
        (\normn{\Gamma^{-1}} \lor \normn{\Gamma^{-1}}^2)
        (1 \lor \normn{y-A\mpr}_2)\\
        &\times 
        \inparen{
        \sqrt{\frac{r_2(\Cpr)}{N}} \lor 
        \sqrt{\frac{t}{N}} \lor 
        \inparen{\frac{r_2(\Cpr)}{N}}^{3/2} \lor 
        \inparen{\frac{t}{N}}^{3/2} \lor
        \frac{r_2(\Cpr)}{N} \sqrt{\frac{t}{N}}\lor
        \sqrt{\frac{r_2(\Cpr)}{N}} \frac{t}{N}
        }
        + 
        \phi \msE,
    \end{align*}
  where  
    \begin{align*}
        \msE = \normn{A} \normn{\Gamma^{-1}}\normn{\Gamma}^{1/2} \normn{\Cpr}
        \inparen{
        \sqrt{\frac{r_2(\Gamma)}{N}} \lor
        \sqrt{\frac{t}{N}} \lor
        \frac{r_2(\Cpr)}{N}\sqrt{\frac{r_2(\Gamma)}{N}} \lor
        \frac{r_2(\Cpr)}{N}\sqrt{\frac{t}{N}} \lor
        \frac{t}{N} \sqrt{\frac{r_2(\Gamma)}{N}} \lor
        \inparen{\frac{t}{N}}^{3/2}
        }.
    \end{align*}
\end{theorem}

\begin{proof}
    It follows from Lemma \ref{lem:MeanOpCtyBdd} that 
    \begin{align}
        \normn{\hatmpost - \mpost}_2 
        &= \normn{\msM(\hatmpr , \hatCpr ) - \phi \msK(\hatCpr) \bareta - \msM(\mpr, \Cpr)}_2 \nonumber\\
        &\le \normn{\msM(\hatmpr , \hatCpr ) - \msM(\mpr, \Cpr)}_2 + \phi \normn{\msK(\hatCpr) \bareta}_2 \nonumber\\
        &\le \normn{ \hatmpr - \mpr}_2 \inparen{1 + \normn{A}^2 \normn{\Gamma^{-1}} \normn{\hatCpr}} \label{eq:ctyBddCovT1}\\ 
        &+\normn{\hatCpr  - \Cpr} \normn{A} \normn{\Gamma^{-1}} 
        \inparen{1 + \normn{A}^2 \normn{\Gamma^{-1}} \normn{\Cpr}} 
        \normn{y- A\mpr}_2 \label{eq:ctyBddCovT2} \\
        &+\phi\normn{\msK(\hatCpr)} \norm{\bareta}_2. \label{eq:ctyBddCovT3}
    \end{align}
    We now control each of the terms in equations \eqref{eq:ctyBddCovT1}, \eqref{eq:ctyBddCovT2}, and \eqref{eq:ctyBddCovT3} separately. For \eqref{eq:ctyBddCovT1}, we note that $\hatmpr - \mpr \sim \mathcal{N} \left(0, \Cpr/N \right)$. Let $E_1$ be the set on which 
    \begin{align*}
        \normn{\hatmpr - \mpr }_2
         \lesssim \sqrt{\normn{\Cpr} \inparen{\frac{r_2(\Cpr)}{N} \lor \frac{t}{N}}},
    \end{align*}
    let $E_2$ be the set on which 
    \begin{align*}
        \normn{\hatCpr - \Cpr}
        \lesssim 
        \normn{\Cpr} \inparen{\sqrt{\frac{r_2(\Cpr)}{N}} \lor \frac{r_2(\Cpr)}{N} 
        \lor \sqrt{\frac{t}{N}} \lor \frac{t}{N} }
    \end{align*}
    and 
    \begin{align*}
        \normn{\hatCpr}
        \lesssim 
        \normn{\Cpr} 
        \inparen{
        1  
        \lor \frac{r_2(\Cpr)}{N} 
        \lor \frac{t}{N} },
    \end{align*}
    and let $E_3$ be the set on which 
    \begin{align*}
        \normn{\bareta}_2 \lesssim 
        \sqrt{\normn{\Gamma}\inparen{\frac{r_2(\Gamma)}{N} \lor \frac{t}{N}}} \,.
    \end{align*}
   By Theorem~\ref{thm:SubGaussianConcentration}, Proposition~\ref{thm:Koltchinski}, and Lemma~\ref{lem:ForecastCovarianceBound}, \ the set $E = E_1 \cap E_2 \cap E_3$ has probability at least $1-ce^{-t}$, and it holds on this set that \eqref{eq:ctyBddCovT1} is bounded above by 
    
    \begin{align} \label{eq:ctyBddCovB1}
        &(\normn{\Cpr}^{1/2} \lor \normn{\Cpr}^{3/2})  
        (1 \lor   \normn{A}^2 \normn{\Gamma^{-1}} )
        \bigg (
        \sqrt{\frac{r_2(\Cpr)}{N}}\lor
        \sqrt{\frac{t}{N}} \lor
        \inparen{\frac{r_2(\Cpr)}{N}}^{3/2} \lor 
        \inparen{\frac{t}{N}}^{3/2} \lor \nonumber \\
        & \hspace{9cm} \frac{r_2(\Cpr)}{N} \sqrt{\frac{t}{N}} \lor
        \sqrt{\frac{r_2(\Cpr)}{N}} \frac{t}{N}
        \bigg ).
    \end{align}
    Further, on the set $E$ we can bound \eqref{eq:ctyBddCovT2} above by
    \begin{align} \label{eq:ctyBddCovB2}
        &
        (\normn{\Cpr} \lor \normn{\Cpr}^2)
        (\normn{A} \lor \normn{A}^3)
        \inparen{\normn{\Gamma^{-1}} \lor \normn{\Gamma^{-1}}^2 }
        \norm{y -A\mpr} 
         \inparen{\sqrt{\frac{r_2(\Cpr)}{N}} \lor \frac{r_2(\Cpr)}{N} 
        \lor \sqrt{\frac{t}{N}} \lor \frac{t}{N} }.
    \end{align}
Finally, for \eqref{eq:ctyBddCovT3}, it follows from Lemma~\ref{lem:KalmanOpCtyBdd}, 
    \begin{align*} 
        \normn{\msK(\hatCpr)} \norm{\bareta} \le \normn{A} \normn{\Gamma^{-1}} \normn{\hatCpr}  \norm{\bareta}
    \end{align*}
    and so on the set $E$, we can show that \eqref{eq:ctyBddCovT3} is bounded above by 
    \begin{align}\label{eq:ctyBddCovB3}
        &\normn{A} \normn{\Gamma^{-1}} \normn{\Cpr}\normn{\Gamma}^{1/2} 
        \inparen{ 
        \sqrt{\frac{r_2(\Gamma)}{N}} \lor
        \sqrt{\frac{t}{N}} \lor
        \frac{r_2(\Cpr)}{N}\sqrt{\frac{r_2(\Gamma)}{N}} \lor
        \frac{r_2(\Cpr)}{N}\sqrt{\frac{t}{N}} \lor
        \frac{t}{N} \sqrt{\frac{r_2(\Gamma)}{N}} \lor
        \inparen{\frac{t}{N}}^{3/2}
        } \equiv \mcE.
    \end{align}
    Putting the three bounds \eqref{eq:ctyBddCovB1}, \eqref{eq:ctyBddCovB2} and \eqref{eq:ctyBddCovB3} together we see that on $E$, it holds that 
    \begin{align*}
        \normn{\hatmpost - \mpost}_2
        &\lesssim 
        (\normn{\Cpr}^{1/2} \lor \normn{\Cpr}^2)
        (\normn{A} \lor \normn{A}^4)
        (\normn{\Gamma^{-1}} \lor \normn{\Gamma^{-1}}^2)
        (1 \lor \normn{y-A\mpr}_2)\\
        &\times 
        \inparen{
        \sqrt{\frac{r_2(\Cpr)}{N}} \lor 
        \sqrt{\frac{t}{N}} \lor 
        \inparen{\frac{r_2(\Cpr)}{N}}^{3/2} \lor 
        \inparen{\frac{t}{N}}^{3/2} \lor
        \frac{r_2(\Cpr)}{N} \sqrt{\frac{t}{N}}\lor
        \sqrt{\frac{r_2(\Cpr)}{N}} \frac{t}{N}
        }\\
        &+ 
        \phi  \mcE. \qedhere
    \end{align*}
\end{proof}

\begin{proof}[Proof of Theorem~\ref{lem:meanwithoutlocExpectation}]
    Recall that from Theorem~\ref{th:meanwithoutloc}, for all $t \ge 1$ with probability at least $1-ce^{-t}$,
    \begin{align*}
        \normn{\hatmpost - \mpost}_2
        &\lesssim 
        (\normn{\Cpr}^{1/2} \lor \normn{\Cpr}^2)
        (\normn{A} \lor \normn{A}^4)
        (\normn{\Gamma^{-1}} \lor \normn{\Gamma^{-1}}^2)
        (1 \lor \normn{y-A\mpr}_2)\\
        &\times 
        \inparen{
        \sqrt{\frac{r_2(\Cpr)}{N}} \lor 
        \inparen{\frac{r_2(\Cpr)}{N}}^{3/2} \lor 
        \sqrt{\frac{t}{N}} \lor 
        \inparen{\frac{t}{N}}^{3/2} \lor
        \frac{r_2(\Cpr)}{N} \sqrt{\frac{t}{N}}\lor
        \sqrt{\frac{r_2(\Cpr)}{N}} \frac{t}{N}
        }
        + 
        \phi \msE.
    \end{align*}
    For notational brevity, let 
    \begin{align*}
        \mcW \equiv  (\normn{\Cpr}^{1/2} \lor \normn{\Cpr}^2)
        (\normn{A} \lor \normn{A}^4)
        (\normn{\Gamma^{-1}} \lor \normn{\Gamma^{-1}}^2)
        (1 \lor \normn{y-A\mpr}_2),
    \end{align*}
    and let $B \equiv \mcW \inparen { \sqrt{\frac{r_2(\Cpr)}{N}} \lor \inparen{\frac{r_2(\Cpr)}{N}}^{3/2}}$. Then, for $\phi=0$ and $p \ge 1$, 
    \begin{align*}
        \E \bigl[ \normn{\hatmpost- \mpost}^p_2 \bigr]
        &= p \int_0^{\infty} x^{p-1} \P (\normn{\hatmpost- \mpost}_2 >x ) dx \\
        &\le   p \int_0^{B} x^{p-1} dx
        +
        p \int_B^{\infty} x^{p-1} \P (\normn{\hatmpost- \mpost}_2 >x ) dx \\
        &\lesssim  B^p
        +
        p \int_0^{\infty} x^{p-1} \exp \inparen{ - 
        \min \inparen{
        \frac{Nx^2}{\mcW^2}, \frac{Nx^{2/3}}{\mcW^{2/3}}, 
        \frac{N^3x^2}{\mcW^2 r_2^2(\Cpr)},
        \frac{N^{3/2} x}{\mcW \sqrt{r_2(\Cpr)}}
        }} dx \\
        &=  B^p
        + 
         p\max \bigg \{ 
        \frac{1}{2}\Gamma\inparen{\frac{p}{2}} \inparen{\frac{\mcW}{\sqrt{N}}}^p,
        \frac{1}{2}\Gamma\inparen{\frac{3p}{2}} \inparen{\frac{\mcW}{N^{3/2}}}^p,\\
        & \hspace{3cm}
        \frac{1}{2}\Gamma\inparen{\frac{p}{2}} \inparen{\frac{\mcW r_2(\Cpr)}{N^{3/2}}}^p,
        \Gamma(p) \inparen{\frac{\mcW \sqrt{r_2(\Cpr)}}{N^{3/2}}}^p
        \bigg \},
    \end{align*}
    where the final 
    equality follows by direct integration. It follows then that 
    \begin{align*}
        \insquare{ \E \normn{\hatmpost- \mpost}^p_2 }^{1/p}
        &\lesssim  B
        + 
        c(p) \mcW \max \inparen{
        \frac{1}{\sqrt{N}},
        \frac{1}{N^{3/2}},
        \frac{ r_2(\Cpr)}{N^{3/2}},
        \frac{\sqrt{r_2(\Cpr)}}{N^{3/2}}
        } \lesssim c(p)B,
    \end{align*}
    where the final inequality holds since $r_2(\Cpr)\ge 1$. The result for the $\phi=1$ case is identical and thus omitted. The constants in the statement of the result are then:
    \begin{align*}
    c_1&=(\normn{\Cpr}^{1/2} \lor \normn{\Cpr}^2) (\normn{A} \lor \normn{A}^4) (\normn{\Gamma^{-1}} \lor \normn{\Gamma^{-1}}^2) (1 \lor \normn{y-A\mpr}_2),\\
    c_2&= \normn{A} \normn{\Gamma^{-1}}\normn{\Gamma}^{1/2} \normn{\Cpr} \qedhere.
\end{align*}
    \end{proof}

\begin{theorem} [Posterior Covariance Approximation with Finite Ensemble ---High Probability Bound] \label{th:covwithoutloc}
 Consider the PO and SR ensemble Kalman updates given by \eqref{eq:POupdate} and \eqref{eq:SRupdate}, respectively, leading to an estimate $\hatCpost$ of the posterior covariance $\Cpost$ defined in \eqref{eq:KFmeancovariance}. Set $\phi=1$ for the PO update and $\phi=0$ for the SR update. For any $t\ge 1,$ it holds with probability at least $1-ce^{-t}$ that
    \begin{align*}
        \normn{\hatCpost - \Cpost} 
        &\lesssim
        (\normn{\Cpr} \lor \normn{\Cpr}^3)
        (\normn{A}^2 \lor \normn{A}^4)
        (\normn{\Gamma^{-1}} \lor \normn{\Gamma^{-1}}^2)
        \inparen{\sqrt{\frac{r_2(\Cpr)}{N}} \lor \inparen{\frac{r_2(\Cpr)}{N}}^2 
        \lor \sqrt{\frac{t}{N}} \lor \inparen{\frac{t}{N}}^2 }
        + \phi \msE,
    \end{align*}
    where
    \begin{align*}
    \msE
    &=
    (\normn{A} \lor \normn{A}^3 )
     (\normn{\Gamma^{-1}} \lor \normn{\Gamma^{-1}}^2 )
     (\normn{\Cpr} \lor \normn{\Gamma}) (\normn{\Cpr} \lor \normn{\Cpr}^2)\\
     & \times \inparen{ 
     \sqrt{\frac{r_2(\Cpr)}{N}} \lor 
     \inparen{\frac{r_2(\Cpr)}{N}}^3 \lor 
     \sqrt{\frac{t}{N}} \lor 
     \inparen{\frac{t}{N}}^3 \lor 
     \inparen{\sqrt{\frac{r_2(\Gamma)}{N}} \lor \frac{r_2(\Gamma)}{N}}
     \inparen{ 
     1 \lor \inparen{\frac{r_2(\Cpr)}{N}}^2 \lor \inparen{\frac{t}{N}}^2
     }
     }.
    \end{align*}
\end{theorem}
\begin{proof}
    From Proposition 4 of \cite{furrer2007estimation}, for the PO-ensemble Kalman update we may write 
    \begin{align*}
        \hatCpost = \msC(\hatCpr) + \whatO,
    \end{align*}
    while for the SR-ensemble Kalman update we have $\hatCpost = \msC(\hatCpr)$. We deal initially with the $\msC(\hatCpr)$ term that is common to both expressions, and then proceed to show how the operator norm of the additional $\whatO$ term can be controlled. From Lemma~\ref{lem:CovarOpCtyBdd}, the continuity of $\msC$ immediately implies that 
    \begin{align*}
        \normn{\msC(\hatCpr) - \msC(\Cpr)} 
        &\le \normn{\hatCpr - \Cpr} 
        \inparen{ 1 + 
        \normn{A}^2 \normn{\Gamma^{-1}} \bigl(\normn{\hatCpr} + \normn{\Cpr} \bigr)
        + \normn{A}^4 \normn{\Gamma^{-1}}^2\normn{\hatCpr} \normn{\Cpr} }\\
        &= 
        \insquare{
        \normn{A}^2 \normn{\Gamma^{-1}} + 
        \normn{A}^4 \normn{\Gamma^{-1}}^2 \normn{\Cpr}
        }
        \normn{\hatCpr - \Cpr} \normn{\hatCpr}\\
        &+ 
        \insquare{
        1 + \normn{A}^2 \normn{\Gamma^{-1}} \normn{\Cpr}
        } \normn{\hatCpr - \Cpr}.
    \end{align*}
    For any $N \in \N$ and $a>0$, let $\msR_{N}(a) \equiv \sqrt{\frac{a}{N}}  \lor \frac{a}{N}$. Let $E_1$ be the set on which both
    \begin{align*}
        \normn{\hatCpr - \Cpr}
        \lesssim 
        \normn{\Cpr} 
        \inparen{
        \msR_{N}(r_2(\Cpr))
        \lor 
        \msR_{N}(t) 
        },
    \qquad \text{and} \qquad 
        \normn{\hatCpr}
        \lesssim 
        \normn{\Cpr} 
        \inparen{1 \lor \msR_{N}(r_2(\Cpr)) \lor \msR_{N}(t)}.
    \end{align*}
    Let $E_2$ be the set on which 
    \begin{align*}
        \normn{\hatGamma - \Gamma}
        \lesssim 
        \normn{\Gamma} \inparen{ \msR_{N}(r_2(\Gamma)) \lor  \msR_{N}(t)},
    \end{align*}
    and $E_3$ the set on which
    \begin{align*}
        \normn{\hatCpr^{\, \upr\eta} - \Cpr^{\, \upr\eta}} 
        \lesssim
        (\normn{\Cpr} \lor \normn{\Gamma}) 
        \inparen{
        \msR_{N}(r_2(\Cpr))
        \lor  
        \msR_{N}(r_2(\Gamma))
        \lor 
        \msR_{N}(t)
        }.
    \end{align*}
    
    Then, by Proposition~\ref{thm:Koltchinski} applied separately to $E_1$ and $E_2$, and Lemma~\ref{lem:ForecastNoiseCovarianceBound} applied to $E_3$, the intersection $E=E_1 \cap E_2 \cap E_3$ has probability at least $1-ce^{-t}$. It follows that on $E$:
    \begin{align}
        \normn{\hatCpr - \Cpr} \normn{\hatCpr} 
        &\lesssim 
        \normn{\Cpr}^2 
        \inparen{ \msR_{N}(r_2(\Cpr))\lor \msR_{N}(t)}
        \inparen{1 \lor \msR_{N}(r_2(\Cpr)) \lor \msR_{N}(t)} \nonumber \\
        &\lesssim \normn{\Cpr}^2 
        \inparen{ \msR_{N}(r_2(\Cpr))\lor \msR_{N}(t) 
        \lor \msR^2_{N,2}(\Cpr)\lor \msR^2_{N,2}(t)
        }
        , \label{eq:(C-C)C}\\
        \normn{\hatGamma - \Gamma} \normn{\hatCpr}^2 
        &\lesssim 
        \normn{\Cpr}^2
        \normn{\Gamma}
        \inparen{ \msR_{N}(r_2(\Gamma)) \lor  \msR_{N}(t)}
        \inparen{1 \lor \msR^2_{N,2}(\Cpr) \lor \msR^2_{N,2}(t)}
        ,\label{eq:(G-G)C^2}\\
        \normn{\hatCpr^{\,u\eta}- \Cpr^{\,u\eta}} \normn{\hatCpr}
        &\lesssim 
        \normn{\Cpr}(\normn{\Cpr} \lor \normn{\Gamma})
        \inparen{1 \lor \msR_{N}(r_2(\Cpr)) \lor \msR_{N}(t)}
        \inparen{\msR_{N}(r_2(\Cpr))\lor  \msR_{N}(r_2(\Gamma))\lor \msR_{N}(t)}, \label{eq:(Cueta-Cueta)C}\\
        \normn{\hatCpr^{\,u\eta}- \Cpr^{\,u\eta}} \normn{\hatCpr}^2
        &\lesssim 
        \normn{\Cpr}^2(\normn{\Cpr} \lor \normn{\Gamma})
        \inparen{1 \lor \msR^2_{N,2}(\Cpr) \lor \msR^2_{N,2}(t)}
        \inparen{\msR_{N}(r_2(\Cpr))\lor  \msR_{N}(r_2(\Gamma))\lor \msR_{N}(t)}. \label{eq:(Cueta-Cueta)C^2}
    \end{align}
    Using \eqref{eq:(C-C)C}, it follows that on $E$,
    \begin{align*}
        \normn{\hatCpost - \Cpost} 
        &\lesssim 
        (\normn{\Cpr} \lor \normn{\Cpr}^3)
        (\normn{A}^2 \lor \normn{A}^4)
        (\normn{\Gamma^{-1}} \lor \normn{\Gamma^{-1}}^2)
       \inparen{ \msR_{N}(r_2(\Cpr))\lor \msR_{N}(t) 
        \lor \msR^2_{N,2}(\Cpr)\lor \msR^2_{N,2}(t)
        }\\
        &=
        (\normn{\Cpr} \lor \normn{\Cpr}^3)
        (\normn{A}^2 \lor \normn{A}^4)
        (\normn{\Gamma^{-1}} \lor \normn{\Gamma^{-1}}^2)
       \inparen{ 
       \sqrt{\frac{r_2(\Cpr)}{N}}
       \lor 
       \inparen{\frac{r_2(\Cpr)}{N}}^2
       \lor 
       \sqrt{\frac{t}{N}}
       \lor 
       \inparen{\frac{t}{N}}^2
        }
    \end{align*}
    Next, for the PO-ensemble Kalman update, it follows by the triangle inequality that 
    \begin{align}
        \normn{\whatO} &\le  
        \normn{\msK(\hatCpr) (\widehat{\Gamma} - \Gamma) \msK^\top(\hatCpr)} \label{eq:OneStepSRENKFCovarianceBoundT1}\\
        &+\normn{(I-\msK(\hatCpr) A) \whatC^{\, u \eta} \msK^\top(\hatCpr)}\label{eq:OneStepSRENKFCovarianceBoundT2}\\
        &+\normn{ \msK(\hatCpr) (\whatC^{\, u \eta})^\top (I-A^\top\msK^\top(\hatCpr))}\label{eq:OneStepSRENKFCovarianceBoundT3},
    \end{align}
    and so we may proceed by bounding each of the three terms \eqref{eq:OneStepSRENKFCovarianceBoundT1}, \eqref{eq:OneStepSRENKFCovarianceBoundT2}, and \eqref{eq:OneStepSRENKFCovarianceBoundT3} separately. For \eqref{eq:OneStepSRENKFCovarianceBoundT1}, invoking first the bound on $\msK$ from Lemma~\ref{lem:KalmanOpCtyBdd} as well as the inequality in \eqref{eq:(G-G)C^2}, it holds on $E$ that 
 \begin{align*}
     \normn{\msK(\hatCpr) (\widehat{\Gamma} - \Gamma) \msK^\top(\hatCpr)}
     &\le
     \normn{\msK(\hatCpr)}^2 \normn{\widehat{\Gamma} - \Gamma}\\
     &\le
     \normn{A}^2 \normn{\Gamma^{-1}}^2 \normn{\hatCpr}^2  \normn{\widehat{\Gamma} - \Gamma}\\
     &\lesssim \normn{A}^2 \normn{\Gamma^{-1}}^2 
     \normn{\Cpr}^2 \normn{\Gamma}
        \inparen{ \msR_{N}(r_2(\Gamma)) \lor  \msR_{N}(t)}
        \inparen{1 \lor \msR^2_{N,2}(\Cpr) \lor \msR^2_{N,2}(t)}
 \end{align*}
 
 Both \eqref{eq:OneStepSRENKFCovarianceBoundT2} and \eqref{eq:OneStepSRENKFCovarianceBoundT3} are equal in operator norm, and so we consider only \eqref{eq:OneStepSRENKFCovarianceBoundT2}. We use Lemma~\ref{lem:KalmanOpCtyBdd} and Lemma~\ref{lem:ForecastCovarianceBound}, along with the inequalities \eqref{eq:(Cueta-Cueta)C} and \eqref{eq:(Cueta-Cueta)C^2} to show that on $E$,
 \begin{align*}
     \normn{(I-\msK(\hatCpr) A) \whatC^{\, u \eta} \msK^\top(\hatCpr)}
     &\le 
      \normn{\msK(\hatCpr)} \normn{I-\msK(\hatCpr) A}\normn{\whatC^{\, u \eta} } \\
     &\le 
     \normn{\msK(\hatCpr)} \inparen{1+\normn{\msK(\hatCpr)}\normn{A}}  \normn{\whatC^{\, u \eta} }\\
     &\le \normn{A} \normn{\Gamma^{-1}} \normn{\hatCpr} 
     \inparen{1+\normn{A}^2 \normn{\Gamma^{-1}} \normn{\hatCpr}}  
     \normn{\whatC^{\, u \eta} }\\
     &\lesssim
     (\normn{A} \lor \normn{A}^3 )
     (\normn{\Gamma^{-1}} \lor \normn{\Gamma^{-1}}^2 )
     [ \normn{\hatCpr} +  \normn{\hatCpr}^2] \normn{\whatC^{\, u \eta}}\\
      &\lesssim
     (\normn{A} \lor \normn{A}^3 )
     (\normn{\Gamma^{-1}} \lor \normn{\Gamma^{-1}}^2 )
     (\normn{\Cpr} \lor \normn{\Gamma}) (\normn{\Cpr} \lor \normn{\Cpr}^2)\\
     &\times \inparen{1 \lor \msR_{N}(r_2(\Cpr))\lor \msR^2_{N,2}(\Cpr) 
     \lor \msR_{N}(t) \lor \msR^2_{N,2}(t)}\\
     & \times \inparen{\msR_{N}(r_2(\Cpr))\lor  \msR_{N}(r_2(\Gamma))\lor \msR_{N}(t)}.
 \end{align*}

 Some algebra shows that 
 \begin{align*}
     &\inparen{1 \lor \msR_{N}(r_2(\Cpr))\lor \msR^2_{N,2}(\Cpr) 
     \lor \msR_{N}(t) \lor \msR^2_{N,2}(t)}
     \inparen{\msR_{N}(r_2(\Cpr))\lor  \msR_{N}(r_2(\Gamma))\lor \msR_{N}(t)}\\
     =
     & \inparen{ 
     \sqrt{\frac{r_2(\Cpr)}{N}} \lor 
     \inparen{\frac{r_2(\Cpr)}{N}}^3 \lor 
     \sqrt{\frac{t}{N}} \lor 
     \inparen{\frac{t}{N}}^3 \lor 
     \msR_{N}(r_2(\Gamma)) \inparen{ 
     1 \lor \inparen{\frac{r_2(\Cpr)}{N}}^2 \lor \inparen{\frac{t}{N}}^2
     }
     },
 \end{align*}
 and so 
 \begin{align*}
     \normn{\whatO} 
     &\lesssim 
     (\normn{A} \lor \normn{A}^3 )
     (\normn{\Gamma^{-1}} \lor \normn{\Gamma^{-1}}^2 )
     (\normn{\Cpr} \lor \normn{\Gamma}) (\normn{\Cpr} \lor \normn{\Cpr}^2)\\
     & \times \inparen{ 
     \sqrt{\frac{r_2(\Cpr)}{N}} \lor 
     \inparen{\frac{r_2(\Cpr)}{N}}^3 \lor 
     \sqrt{\frac{t}{N}} \lor 
     \inparen{\frac{t}{N}}^3 \lor 
     \msR_{N}(r_2(\Gamma)) \inparen{ 
     1 \lor \inparen{\frac{r_2(\Cpr)}{N}}^2 \lor \inparen{\frac{t}{N}}^2
     }
     }. \qedhere
 \end{align*}
\end{proof}

\begin{proof}[Proof of Theorem~\ref{lem:covwithoutlocExpectation}]
    The proof follows similarly to that of Theorem~\ref{lem:meanwithoutlocExpectation} and is therefore omitted. The constants in the statement of the result are:
    \begin{align*}
        c_1 &=  (\normn{\Cpr} \lor \normn{\Cpr}^3)
        (\normn{A}^2 \lor \normn{A}^4)
        (\normn{\Gamma^{-1}} \lor \normn{\Gamma^{-1}}^2),\\
        c_2 &= (\normn{A} \lor \normn{A}^3 ) (\normn{\Gamma^{-1}} \lor \normn{\Gamma^{-1}}^2 ) (\normn{\Cpr} \lor \normn{\Gamma}) (\normn{\Cpr} \lor \normn{\Cpr}^2). \qedhere
    \end{align*}
\end{proof}

\newcommand{\Mt}{M^{(t)}}

\subsection{Multi-Step Analysis of the Square Root Ensemble Kalman Filter} \label{sec:Multistep}

Here we provide a description of the multi-step EnKF algorithm discussed in Remark~\ref{rem:MultiStep}. As described there, we focus on the square root EnKF studied in \cite{kwiatkowski2015convergence}. Given an initial ensemble $\{\upost_n^{(0)}\}_{n=1}^N$, the algorithm iterates the steps of the square root ensemble update \eqref{eq:SRupdate} with new observations $\yt$ and with possibly varying model matrices $\At$. We assume that the noise distribution does not change over time, though this assumption can easily be relaxed at the expense of more cumbersome notation. 
We summarize both the Kalman filter and the square root EnKF in Table~\ref{tab:multi-step}. 
In this filtering set-up, $\Mt \in \R^{d \times d}$ is the dynamics map and $\At \in \R^{ k\times d}$ is the observation map at time $t\ge 1$. As detailed in \cite{sanzstuarttaeb}, such a filtering set-up leads to a sequence of inverse problems of the form \eqref{eq:LinearIP}, where the forward model is given by the observation map, and the prior \emph{forecast distribution} blends the dynamics map with previous probabilistic estimates. 
Throughout this subsection, we write $\normn{\hatmpost - \mpostt}_p \equiv \insquare{ \E \normn{\hatmu^{(t)} - \mu^{(t)}}^p_2  }^{1/p}$ and $\normn{\hatCpostt - \Cpostt}_p \equiv \insquare{ \E \normn{\hatCpost^{(t)} - \Cpost^{(t)}}^p  }^{1/p}$.

\begin{table}[h]
    \def\arraystretch{1.7}
    \begin{tabular}{c|l|l}
        \hline
        &Kalman filter & Square root EnKF \\
        \hline\hline
        Input & $\{\yt, \At, \Mt \}_{t=1}^T$, $\Gamma$, $\mu^{(0)}, \Sigma^{(0)}$ &
        $\{\yt, \At, \Mt\}_{t=1}^T$, $\Gamma$, $\{ \upost_n^{(0)}\}_{n=1}^N \iid \mathcal{N}(\mu^{(0)}, \Sigma^{(0)})$ \\
        \hline
         \multirow{3}{*}{Forecast}&
         & $\uprt_n = \Mt \upostone_n, ~~ n=1,\dots, N$\\
         &$\mprt = \Mt \mpostone$ &
         $\hatmprt = \frac{1}{N}\sum_{n=1}^N \uprt_n$
         \\
         &$\Cprt = \Mt \Cpostone (\Mt)^\top$ 
         & $\hatCprt 
    = \frac{1}{N-1}\sum_{n=1}^N (\uprt_n - \hatmprt) (\uprt_n - \hatmprt)^\top$ \\
        \hline
         \multirow{3}{*}{Analysis}&
         & $\upostt_n = \msM(\uprt_n, \hatCprt; \At, \yt, \Gamma), ~~ n=1,\dots, N$\\
         &$\mpostt =  \msM(\mprt, \Cprt; \At, \yt, \Gamma)$ &
         $\hatmpostt = \frac{1}{N}\sum_{n=1}^N \upostt_n$
         \\
         &$\Cpostt = \msC(\Cprt; \At, \Gamma)$ 
         & $\hatCpostt = \frac{1}{N-1}\sum_{n=1}^N (\upostt_n - \hatmpostt) (\upostt_n - \hatmpostt)^\top$ \\
    \hline
    Output& $\{\mpost^{(t)}, \Cpost^{(t)}\}_{t=1}^T$ &  $\{\hatmpostt ,\hatCpostt\}_{t=1}^T$\\ 
    \hline 
    \end{tabular}
    \caption{\label{tab:multi-step} Comparison of the Kalman filter and square root EnKF considered in \cite{kwiatkowski2015convergence}. The forecast and analysis steps are to be repeated for $t=1,\dots,T$ iterations.}
\end{table}

We will use two auxiliary lemmas to prove the main result of this subsection, Corollary~\ref{cor:MultiStepSREnKF} below.

\begin{lemma}[{Continuity and Boundedness of Covariance-Update Operator in $L^p$
\cite[Corollary 4.8]{kwiatkowski2015convergence}}] \label{lem:CovarOpCtyBddLp} 
    Let $\msC$ be the covariance-update operator defined in \eqref{eq:CovarianceOperator}. Let $Q \in \mcS_+^d$ be a random matrix and $P \in \mcS_+^d$ be a deterministic matrix, $\Gamma \in \mcS_{++}^k$, $A \in \R^{k \times d}$, $y \in \R^k,$ and $m,m' \in \R^d$. Then, for any $1 \le p <\infty$, the following holds:
      \begin{align*}
      \begin{split}
            \normn{\msC(Q) - \msC(P)}_{p} 
            &\le
            \normn{Q - P}_{p}
            (1 + \normn{A}^2 \normn{\Gamma^{-1}} \normn{P})\\
            &+ (
                \normn{A}^2 \normn{\Gamma^{-1}} + 
                \normn{A}^4 \normn{\Gamma^{-1}}^2 \normn{P}
            ) 
            \normn{Q}_{2p} \normn{Q-P}_{2p}.
        \end{split}    
         \end{align*}
    \end{lemma}

    \begin{lemma}[{Continuity and Boundedness of Mean-Update Operator in $L^p$ \cite[Corollary 4.10]{kwiatkowski2015convergence}}] \label{lem:MeanOpCtyBddLp}
        Let $\msM$ be the mean-update operator defined in \eqref{eq:MeanOperator}. Let $P, Q \in \mcS_+^d$, $\Gamma \in \mcS_{++}^k$, $A \in \R^{k \times d}$, $y \in \R^k,$ and $m,m' \in \R^d$. Assume that $Q$ and $m$ are random, and that $P$ and $m'$ are deterministic. The following holds:
         \begin{align*}
         \begin{split}
                \norm{\msM(m, Q)- \msM(m', P)}_{p} 
                &\le 
                \norm{m-m'}_{p}
                + 
                \normn{A}^2 \normn{\Gamma^{-1}} \normn{Q}_{2p} \norm{m-m'}_{2p}\\   
                & + 
                \normn{Q-P}_{p} \normn{A} \normn{\Gamma^{-1}} 
                \bigl(1 + \normn{A}^2 \normn{\Gamma^{-1}} \normn{P} \bigr) 
                \norm{y-Am'}_2.
        \end{split}        
             \end{align*}
        \end{lemma}

        The next result shows how our one-step bounds in Theorems~\ref{lem:meanwithoutlocExpectation} and \ref{lem:covwithoutlocExpectation} can be extended to provide non-asymptotic bounds on the performance of the multi-step square root EnKF. The proof follows a similar argument to the proof of \cite[Theorem 6.1]{kwiatkowski2015convergence}.
        \begin{corollary}\label{cor:MultiStepSREnKF} 
            Consider the square root EnKF defined in Table~\ref{tab:multi-step}. Suppose that $N \gtrsim r_2(\Cpost^{(0)})$. Then, for any $t \ge 1$ and $p \ge 1,$
            \begin{align*}
               \normn{\hatmpostt - \mpostt}_p
                &\lesssim_p 
                \sqrt{\frac{r_2(\Cpost^{(0)})}{N}}
                \times 
                c ( \{ \normn{M^{(l)}}, \normn{A^{(l)}}, \normn{\Cpost^{(l-1)}}, \normn{y^{(l)} - A^{(l)} \mpr^{(l)} } \}_{l=1}^t, \normn{\Gamma^{-1}}), \\
              \normn{\hatCpostt - \Cpostt}_p 
                &\lesssim_p 
                \sqrt{\frac{r_2(\Cpost^{(0)})}{N}}
                \times 
                c ( \{ \normn{M^{(l)}}, \normn{A^{(l)}}, \normn{\Cpost^{(l-1)}}\}_{l=1}^t, \normn{\Gamma^{-1}}).
            \end{align*}
        \end{corollary}

        \begin{proof}
            The proof follows by strong induction on the predicate in the statement of the theorem. To that end, the base case ($t=1$) holds by Theorems~\ref{lem:meanwithoutlocExpectation} and \ref{lem:covwithoutlocExpectation}, which state that, for any $p \ge 1,$
            \begin{align*}
                \normn{\hatmpost^{(1)} - \mpost^{(1)}}_p 
                &\lesssim_p 
                \sqrt{\frac{r_2(\Cpost^{(0)})}{N}}
                \times 
                c ( \normn{M^{(1)}}, \normn{A^{(1)}}, \normn{\Cpost^{(0)}}, \normn{y^{(1)} - A^{(1)} \mpr^{(1)} }, \normn{\Gamma^{-1}}), \\
               \normn{\hatCpost^{(1)} - \Cpost^{(1)}}_p 
                & \lesssim_p 
                \sqrt{\frac{r_2(\Cpost^{(0)})}{N}}
                \times 
                c (  \normn{M^{(1)}}, \normn{A^{(1)}}, \normn{\Cpost^{(0)}}, \normn{\Gamma^{-1}}).
            \end{align*}
            Suppose now that the claim holds for $l=2,\dots, t-1.$ Then, for $l=t$, we have by Lemma~\ref{lem:CovarOpCtyBddLp}
            \begin{align} \label{eq:InductiveProof1}
                \normn{\hatCpostt - \Cpostt}_p 
                &=\normn{\msC(\hatCprt) - \msC(\Cprt)}_{p} \nonumber \\ 
                &\le
                \normn{\hatCprt - \Cprt}_{p}
                (1 + \normn{\At}^2 \normn{\Gamma^{-1}} \normn{\Cprt}) \nonumber  \\ 
                &+ (
                    \normn{\At}^2 \normn{\Gamma^{-1}} + 
                    \normn{\At}^4 \normn{\Gamma^{-1}}^2 \normn{\Cprt}
                ) 
                \normn{\hatCprt}_{2p} \normn{\hatCprt-\Cprt}_{2p}.   
            \end{align}
            By the definition of $\hatCprt, \Cprt$ together with the inductive hypothesis, it follows that, for $\mathfrak{p} \in \{ p, 2p\},$ 
            \begin{align*}
                \normn{\hatCprt - \Cprt}_{\mathfrak{p}}
                &= \normn{\Mt (\hatCpost^{(t-1)} - \Cpost^{(t-1)})(\Mt)^\top}_{\mathfrak{p}}\\
                &\le  \normn{\Mt}^2 \normn{\hatCpost^{(t-1)} - \Cpost^{(t-1)}}_{\mathfrak{p}}\\
                &\lesssim_p  
                \normn{\Mt}^2 
                \sqrt{\frac{r_2(\Cpost^{(0)})}{N}}
                \times c ( \{ \normn{M^{(l)}}, \normn{A^{(l)}}, \normn{\Cpost^{(l-1)}}\}_{l=1}^{t-1}, \normn{\Gamma^{-1}}) \\
                &=
                \sqrt{\frac{r_2(\Cpost^{(0)})}{N}}
                \times c ( \{ \normn{M^{(l)}}, \normn{A^{(l)}}, \normn{\Cpost^{(l-1)}}\}_{l=1}^{t}, \normn{\Gamma^{-1}}).
            \end{align*}
            Further, we have 
            \begin{align*}
                \normn{\hatCprt}_{2p} 
                &\le \normn{\hatCprt-\Cprt}_{2p}+\normn{\Cprt} \\
                & \lesssim_p  \sqrt{\frac{r_2(\Cpost^{(0)})}{N}}
                c ( \{ \normn{M^{(l)}}, \normn{A^{(l)}}, \normn{\Cpost^{(l-1)}}\}_{l=1}^{t}, \normn{\Gamma^{-1}}) + \normn{\Mt}^2\normn{\Cpost^{(t-1)}}.
            \end{align*}
            Plugging these two results into \eqref{eq:InductiveProof1} gives 
            \begin{align*}
                \normn{\hatCpostt - \Cpostt}_p 
                &\lesssim_p 
                \sqrt{\frac{r_2(\Cpost^{(0)})}{N}}
                \times c ( \{ \normn{M^{(l)}}, \normn{A^{(l)}}, \normn{\Cpost^{(l-1)}}\}_{l=1}^{t}, \normn{\Gamma^{-1}}).
            \end{align*}
            Similarly, by Lemma~\ref{lem:MeanOpCtyBddLp} we have 
            \begin{align} \label{eq:InductiveProof2}
                \normn{\hatmpostt - \mpostt}_p 
                &= 
                \normn{\msM(\hatmprt, \hatCprt)- \msM(\mprt, \Cprt)}_{p} \nonumber \\
                &\le 
                \normn{\hatmprt-\mprt}_{p}
                + 
                \normn{\At}^2 \normn{\Gamma^{-1}} \normn{\hatCprt}_{2p} 
                \normn{\hatmprt-\mprt}_{2p} \nonumber \\   
                & + 
                \normn{\hatCprt-\Cprt}_{p} \normn{\At} \normn{\Gamma^{-1}} 
                \bigl(1 + \normn{\At}^2 \normn{\Gamma^{-1}} \normn{\Cprt} \bigr) 
                \normn{\yt-\At \mprt}_2.
            \end{align}
            By the definition of $\hatmprt, \mprt$ together with the inductive hypothesis, we have, for $\mathfrak{p} \in \{ p, 2p\},$ 
            \begin{align*}
                \normn{\hatmprt-\mprt}_{\mathfrak{p}}
                &=\normn{ \Mt (\hatmpost^{(t-1)}-\mpost^{(t-1)})  }_{\mathfrak{p}}\\
                &\le \normn{\Mt}\normn{ \hatmpost^{(t-1)}-\mpost^{(t-1)}}_{\mathfrak{p}}\\
                &\lesssim_p 
                \normn{\Mt} 
                \sqrt{\frac{r_2(\Cpost^{(0)})}{N}}
                \times 
                c ( \{ \normn{M^{(l)}}, \normn{A^{(l)}}, \normn{\Cpost^{(l-1)}}, \normn{y^{(l)} - A^{(l)} \mpr^{(l)} } \}_{l=1}^{t-1}, \normn{\Gamma^{-1}})
                \\
                &=
                \sqrt{\frac{r_2(\Cpost^{(0)})}{N}}
                \times 
                c ( \{ \normn{M^{(l)}}, \normn{A^{(l)}}, \normn{\Cpost^{(l-1)}}, \normn{y^{(l)} - A^{(l)} \mpr^{(l)} } \}_{l=1}^{t}, \normn{\Gamma^{-1}}).
            \end{align*}
            Plugging this bound and the one for $\normn{\hatCprt}_{2p}$ derived previously in the proof into \eqref{eq:InductiveProof2} yields 
            \begin{equation*}
                \normn{\hatmpostt - \mpostt}_p 
                \lesssim_p 
                \sqrt{\frac{r_2(\Cpost^{(0)})}{N}}
                \times 
                c ( \{ \normn{M^{(l)}}, \normn{A^{(l)}}, \normn{\Cpost^{(l-1)}}, \normn{y^{(l)} - A^{(l)} \mpr^{(l)} } \}_{l=1}^t, \normn{\Gamma^{-1}}).  \qedhere
            \end{equation*}
        \end{proof}

\section{Proofs: Section \ref{sec:withlocalization}}\label{sec:Sub-Gaussian}
This appendix contains the proofs of all the theorems in Section \ref{sec:withlocalization}. Results on covariance estimation are in Subsection \ref{ssec:B1} and our main results on ensemble Kalman updates are in Subsection \ref{sec:localizationProofs}.

\subsection{Covariance Estimation}\label{ssec:B1}
Here we establish Theorems \ref{thm:SoftSparistyCovarianceBound} and \ref{thm:SoftSparistyCrossCovarianceBound}. We first collect some required technical results in Subsection \ref{sssec:backgroundpreliminaries}. Next we study covariance and cross-covariance estimation under soft sparsity in Subsections \ref{sssec:proof4.2} and \ref{sssec:proof4.3}, respectively.

\subsubsection{Background and Preliminaries} \label{sssec:backgroundpreliminaries}
\begin{definition}[{\cite[Definition 2.2.17]{talagrand2014upper}}]
    Given a set $T$, an admissible sequence of partitions of $T$ is an increasing sequence $(\Delta_n)$ of partitions of $T$ such that $\text{card}(\Delta_0) = 1$ and $\text{card}(\Delta_n) \le 2^{2^n}$ for $n \ge 1$.
\end{definition}

The notion of an admissible sequence of partitions allows us to define the following notion of complexity of a set $T$, often referred to as \textit{generic complexity}.

\begin{definition}[{\cite[Definition 2.2.19]{talagrand2014upper}}]
    Let $(T,\mathsf{d})$ be a possibly infinite metric space, and define 
    $$
    \gamma_2(T,\mathsf{d}) = \inf \sup_{t \in T}
    \sum_{n \ge 0} 2^{n/2} \text{Diam} \bigl(\Delta_n(t)\bigr),
    $$
    where $\Delta_n(t)$ denotes the unique element of the partition to which $t$ belongs, and the infimum is taken over all admissible sequences of partitions.
\end{definition}

The following theorem is known as the Majorizing Measure Theorem and provides upper and lower bounds for centered Gaussian processes in terms of the generic complexity.

\begin{theorem}[{\cite[Theorem~2.4.1]{talagrand2014upper}}] \label{thm:TalagrandUpperLowerGP}
    Let $X_t$, $t \in T$ be a centered Gaussian process which induces a metric $\mathsf{d}_X : T\times T \to [0, \infty]$ defined by \begin{align*}
        \mathsf{d}_X^2(s,t) = \E \insquare{ (X_s - X_t)^2 } .
    \end{align*}
    Then there exists an absolute constant $L > 0$ such that 
    \begin{align*}
        \frac{1}{L} \gamma_2(T,\mathsf{d}_X) 
        \le 
        \E \insquare{\, \sup_{t \in T} X_t \, }
        \le 
        L \gamma_2(T,\mathsf{d}_X).
    \end{align*}
\end{theorem}

We will be primarily interested in the case that $T= \mcF$ is some function class on the probability space $(\mcX, \mcA, \P)$, and with $\mathsf{d}$ being the metric induced either by $\| \cdot \|_{L_2}$ or $\| \cdot \|_{\psi_2}$. We denote these spaces by $(\mcF, L_2)$ and $(\mcF, \psi_2)$ respectively throughout this section. The next result is an exponential generic chaining bound, which was introduced in \cite[Corollary 5.7]{dirksen2015tail} and described in \cite[Theorem 8]{koltchinskii2017concentration}. We present it as it was described in the latter reference.

\begin{theorem}[{\cite[Theorem 8]{koltchinskii2017concentration}}] \label{thm:GenericConcentration}
Let $(\mcX, \mcA, \P)$ be a probability space and consider the random sample $X, X_1,\dots,X_N \iid \P$. Let $\mcF$ be a class of measurable functions on $(\mcX, \mcA)$. There exists a universal constant $c>0$ such that, for all $t \ge 1,$ it holds with probability at least $1-e^{-t}$ that
\begin{align*}
    \sup_{f \in \mcF}
    \abs{
    \frac{1}{N} \sum_{n=1}^N f^2(X_n) - \E [f^2 (X)]}
    \le
    c
    \inparen{
    \sup_{f \in \mcF} \normn{f}_{\psi_2} \frac{\gamma_2(\mcF, \psi_2)}{\sqrt{N}}
    \lor 
    \frac{\gamma^2_2(\mcF, \psi_2)}{N}
    \lor 
    \sup_{f \in \mcF} \normn{f}^2_{\psi_2} \sqrt{\frac{t}{N}}
    \lor 
    \sup_{f \in \mcF} \normn{f}^2_{\psi_2} \frac{t}{N}
    }.
\end{align*}
\end{theorem}

\begin{lemma} [{Expectation Bound from Probability Bound, \cite[Lemma 2.2.3]{talagrand2014upper}}] \label{lem:TalagrandTrick}
    Let $Y \ge 0$ be a random variable satisfying 
    \begin{align*}
        \P(Y \ge r) \le a \exp \inparen{-\frac{r^2}{b^2}}, \qquad r \ge 0,
    \end{align*}
    for certain numbers $a \ge 2$ and $b>0$. Then there is a universal constant $c$ such that
    \begin{align*}
        \E [Y] \le c b \sqrt{\log a}.
    \end{align*}
\end{lemma}

Finally, we recall the following dimension-free bound for the maxima of sub-Gaussian random variables. 

\begin{lemma} [{Dimension-Free Sub-Gaussian Maxima, \cite[Lemma 2.4]{van2017spectral}}] \label{lem:DimFreeGaussianMax}
    Let $X_1,\dots X_N$ be not necessarily independent sub-Gaussian random variables with
    \begin{align*}
        \P(X_n > x) \le c e^{-x^2/c\sigma^2_n}, \qquad \text{for all } 
        x \ge 0, ~1 \le n\le N,
    \end{align*}
    where $\sigma_n \ge 0$ is given, or alternatively $\normn{X_n}_{\psi_2} \lesssim \sigma_n$.  Then, for any $t \ge 1,$ it holds with probability at least $1-ce^{-ct}$ that
    \begin{align*}
        \max_{n \le N} X_n \lesssim \sqrt{t} \max_{n \le N } \sigma_{(n)} \sqrt{\log(n+1)} ,
    \end{align*}
    where $\sigma_{(1)} \ge \sigma_{(2)} \ge \dots \sigma_{(N)}$ is the decreasing rearrangement of $\sigma_1,\dots, \sigma_N$. Further
    \begin{align*}
        \E \insquare{\max_{n \le N} X_n }\lesssim \max_{n\le N} \sigma_{(n)} \sqrt{\log(n+1)}.
    \end{align*}
\end{lemma}
\begin{proof}
    The proof of the upper bound is based on the proof of Proposition 2.4.16 in \cite{talagrand2014upper}. By permutation invariance, we can assume without loss of generality that $\sigma_1\ge \sigma_2 \ge \dots \ge \sigma_N$. Then 
    \begin{align*}
        \P \inparen{
            \max_{n \le N} \frac{ X_n}{\sigma_n \sqrt{\log(n+1)}}
            \ge \sqrt{t}
        } 
        \le 
        \sum_{n=1}^{N} \P \Bigl(X_n \ge  \sigma_n \sqrt{t\log(n+1)} \Bigr)
        \lesssim
         \sum_{n=1}^{N}  \exp \inparen{ -\frac{t}{c}\log(n+1)}.
    \end{align*}
    For $t \ge 2c$, the final expression in the above display is finite, and we may write 
    \begin{align*}
        \sum_{n=1}^{N}  \exp \inparen{ -\frac{t}{c}\log(n+1)}
        &= 
        \sum_{n=2}^{N+1}  \exp \inparen{ -\frac{t}{c}\log(n)}
        \le \exp \inparen{ -\frac{t}{c}\log(2)}
        +\int_2^\infty x^{-t/c}dx \le c e^{-t/c}.
    \end{align*}
    Therefore, for any $t \ge 2c$, it holds with probability at least $1-ce^{-t/c}$ that 
    \begin{align*}
        \max_{n \le N} X_n 
        \lesssim 
        \sqrt{t}  \max_{n \le N } \sigma_{(n)} \sqrt{\log(n+1)}.
    \end{align*}
    This implies that, for any $t\ge 1$, it holds with probability at least $1-ce^{-(t \lor 2c)/c}$ that 
    \begin{align*}
        \max_{n \le N} X_n 
        \lesssim 
        (\sqrt{t} \lor \sqrt{2c}) 
        \max_{n \le N } \sigma_{(n)} \sqrt{\log(n+1)}
        \lesssim 
        \sqrt{t}\max_{n \le N } \sigma_{(n)} \sqrt{\log(n+1)}.
    \end{align*}
    Since $1-ce^{-(t \lor 2c)/c} \ge 1-ce^{-t/c}$ it holds that, for any $t \ge 1$, with probability at least $1-ce^{-t/c}$
    \begin{align*}
        \max_{n \le N} X_n 
        \lesssim 
        \sqrt{t}\max_{n \le N } \sigma_{(n)} \sqrt{\log(n+1)}.
    \end{align*}
    It follows by Lemma~\ref{lem:TalagrandTrick} that
    \begin{align*}
        \E \insquare{\max_{n \le N} \frac{ X_n}{\sigma_n \sqrt{\log(n+1)}}}
        \le c,
    \end{align*}
    which in turn implies 
    \begin{equation*}
        \E \insquare{\max_{n \le N} X_n }
        \lesssim
        \max_{n \le N} \sigma_n \sqrt{\log(n+1)}
        = 
        \max_{n \le N} \sigma_{(n)} \sqrt{\log(n+1)}.   \qedhere
    \end{equation*}
\end{proof}

\subsubsection{Covariance Estimation under Soft Sparsity}\label{sssec:proof4.2}
This subsection contains the proof of Theorem \ref{thm:SoftSparistyCovarianceBound}. We follow the approach in \cite[Theorem~4]{koltchinskii2017concentration}, but we restrict our attention to finite dimensional spaces. 
Our proof will rely on the following max-norm covariance estimation bound, which may be of independent interest.

\begin{theorem}[Covariance Estimation with Sample Covariance ---Max-Norm Bound] \label{thm:MaxNormEffectiveDim}
Let $X_1,\dots,X_N$ be $d$-dimensional i.i.d. sub-Gaussian random vectors with $\E[X_1]=\mu^X$ and $\tvar(X_1) = \Sigma^X.$ Let $\hatSigma^X = (N-1)^{-1}\sum_{n=1}^N (X_n-\mu^X)(X_n-\mu^X)^\top$. Then there exists a constant $c$ such that, for all $t \ge 1,$ it holds with probability at least $1-ce^{-t}$ that
\begin{align*}
    \normn{\hatSigma^X - \Sigma^X}_{\max} 
    \le
    c 
    \Sigma^X_{(1)}
    \inparen{
    \sqrt{\frac{r_\infty (\Sigma^X)}{N}}
    \lor 
    \sqrt{\frac{t}{N}}
    \lor 
    \frac{t}{N}
    \lor 
    \frac{t r_\infty(\Sigma^X)}{N}
    }, 
\end{align*}
where 
\begin{align*}
    r_\infty(\Sigma^X) \equiv \frac{\max_{j} \Sigma^X_{(j)}  \log(j+1) }{\Sigma^X_{(1)}}.
\end{align*}
\end{theorem}
\begin{proof}
The proof of this result is based on the proof of the upper bound of Theorem 4 of \cite{koltchinskii2017concentration}, in conjunction with Theorem~\ref{thm:GenericConcentration}. We deal with the case $\mu^X=0$ first. To this end, let $Z_1,\dots,Z_N$ be $d$-dimensional i.i.d. sub-Gaussian random vectors with zero mean and $\tvar[Z_1]=\Sigma^X$. We denote the distribution of $Z_1$ by $\P$, and note that $\normn{\cdot}_{\psi_1}$, $\normn{\cdot}_{\psi_2},$ and $ \normn{\cdot}_{L_2}$ are defined implicitly with respect to $\P$. Let $\hatSigma^0 = N^{-1} \sum_{n=1}^N Z_nZ_n^\top$. We rewrite the expectation of interest as a squared empirical process term over an appropriate class of functions. For $j \ge 1$ we denote the $j$-th canonical vector (the vector with $1$ in the $j$-th index and zero otherwise) by $e_j$. Then, we note that 
\begin{align*}
    \normn{\hatSigma^0 - \Sigma^X}_{\max}
    &=  \sup_{i,j} \inp{ e_i, (\hatSigma^0-\Sigma^X) e_j}\\
    &=  \sup_{i,j} 
    \insquare{
    \inp{ \frac{e_i+e_j}{2}, (\hatSigma^0-\Sigma^X) \frac{e_i+e_j}{2}}
    -
    \inp{ \frac{e_i-e_j}{2}, (\hatSigma^0-\Sigma^X) \frac{e_i-e_j}{2}}
    }\\
    &\le 2  \sup_{u \in \mcU} 
    \abs{\inp{(\hatSigma^0-\Sigma^X) u, u}},
  \end{align*}
where $\mcU = \bigl\{ u \in \R^d: u = \pm \frac{1}{2}(e_i \pm e_j), ~ 1 \le  i, j  \le d\ \bigr\}$. Define the set of functions $\mcF_U = \bigl\{ \inp{\cdot, u}: u \in \mcU \bigr\}$, and note that, for any $f \in \mcF_\mcU$, $-f \in \mcF_\mcU$ and $\E [f(Z_1)] = 0$. It then follows by Theorem~\ref{thm:GenericConcentration} that for the same universal constant $c$ in the statement of the theorem, 
\begin{align*}
    2 \sup_{u \in \mcU} 
    \abs{\inp{(\hatSigma^0-\Sigma^X) u, u}}
    &= 2  \sup_{u \in \mcU} \abs{
    \frac{1}{N} \sum_{n=1}^N \inp{Z_n, u}^2 - \inp{ u, \Sigma^X u}}\\
    &= 2  \sup_{f \in \mcF_\mcU}
    \abs{
    \frac{1}{N}\sum_{n=1}^N f^2(Z_n) - \E [ f^2(Z_1)]
    }\\
    &\le
    2c \inparen{
    \sup_{f \in \mcF_\mcU} \normn{f}_{\psi_2}
    \frac{\gamma_2(\mcF_\mcU; \psi_2)}{\sqrt{N}}
    \lor 
    \frac{\gamma^2_2(\mcF_\mcU; \psi_2)}{N}
    \lor 
    \sup_{f \in \mcF_\mcU} \normn{f}^2_{\psi_2} 
    \sqrt{\frac{t}{N}}
    \lor 
    \sup_{f \in \mcF_\mcU} \normn{f}^2_{\psi_2} 
    \frac{t}{N}
    }.
\end{align*}
Using the equivalence of the $\psi_2$ and $L_2$ norms for linear functionals, we have 
\begin{align*}
     \sup_{f \in \mcF_\mcU} \normn{f}_{\psi_2}
     &\lesssim
      \sup_{f \in \mcF_\mcU} \normn{f}_{L_2}
      =\max_{u \in \mcU} 
      \sqrt{ \E \bigl[\inp{Z_1, u}^2 \bigr] }
      =\max_{u \in \mcU} 
      \sqrt{ \langle u, \Sigma^X u \rangle }
      = \frac{1}{2}\max_{i,j} \sqrt{ \bigl\langle e_i \pm e_j, \Sigma^X (e_i \pm e_j) \bigr\rangle } \\
      &= \frac{1}{2}\max_{i,j} \sqrt{ \langle e_i, \Sigma^X e_i \rangle + \langle e_j, \Sigma^X e_j\rangle 
      \pm 2 \langle e_i, \Sigma^X e_j \rangle}
       = \frac{1}{2}\max_{i,j} \sqrt{\Sigma^X_{ii}+ \Sigma^X_{jj}
      \pm 2 \Sigma^X_{ij}}
      \le \sqrt{\Sigma^X_{(1)}}.
\end{align*}

To control the generic complexity $\gamma_2(\mcF_U, \psi_2)$, let $Y \sim \mcN(0,\Sigma^X)$ be a $d$-dimensional Gaussian vector, with induced metric
\begin{align*}
    \mathsf{d}_Y(u,v) = \sqrt{\E \bigl[( \inp{Y, u}- \inp{Y, v})^2 \bigr]}
     = \normn{\inp{\cdot, u}-\inp{\cdot, v}}_{L_2}, 
    \qquad u,v \in \mcU.
\end{align*}
Using again the equivalence of the $\psi_2$ and $L_2$ norms for linear functionals, we have that 
\begin{align*}
    \gamma_2(\mcF_\mcU; \psi_2)
     \lesssim 
     \gamma_2(\mcF_\mcU; L_2)
     = \gamma_2(\mcU; \mathsf{d}_Y).
\end{align*}
It follows then by Theorem~\ref{thm:TalagrandUpperLowerGP} that
\begin{align*}
    \gamma_2(\mcU; \mathsf{d}_Y) 
    &\lesssim \E \insquare{ \, \sup_{u \in \mcU} \inp{Y, u} }\\
    &=\E  \insquare{\max_{i,j} \inp{Y, \pm \frac{1}{2} (e_i \pm e_j)}}\\
    &\le \E \insquare{\max_{j} \Bigl|\inp{Y, e_j}\Bigr|}\\
    &\lesssim \max_{j} \sqrt{\Sigma^X_{(j)} \log(j+1)},
\end{align*}
where the final inequality follows by Lemma~\ref{lem:DimFreeGaussianMax}. We have shown that with probability at least $1-e^{-t}$
\begin{align}
    \normn{\hatSigma^0 - \Sigma^X}_{\max}
    &\lesssim 
    \inparen{
    \sqrt{ \Sigma^X_{(1)} 
    \max_{j}
    \frac{ \Sigma^X_{(j)}\log(j+1)}{N}
    }
    \lor 
     \max_{j} 
    \frac{ \Sigma^X_{(j)} \log(j+1)}{N}
    \lor 
    \Sigma^X_{(1)} \sqrt{\frac{t}{N}}
    \lor 
    \Sigma^X_{(1)} \frac{t}{N}
    } \nonumber 
    \\
    &=\Sigma^X_{(1)}
     \inparen{
    \sqrt{\frac{r_\infty (\Sigma^X)}{N}}
    \lor 
    \frac{r_\infty(\Sigma^X)}{N}
    \lor 
    \sqrt{\frac{t}{N}}
    \lor 
    \frac{t}{N}
    } \label{eq:E1forMaxNormBd}.
\end{align}
In the un-centered case, taking $X_n = Z_n + \mu^X$, we have $\hatSigma^X = \hatSigma^0 - \barZ \barZ^\top$ and it follows that
    \begin{align*}
        \normn{\hatSigma^X - \Sigma^X}_{\max}
        \le 
        \normn{\hatSigma^0 - \Sigma^X}_{\max} + 
        \normn{\barZ \barZ^\top}_{\max}.
    \end{align*}
    By Lemma~\ref{lem:DimFreeGaussianMax}, with probability at least $1-ce^{-t}$
    \begin{align}
        \normn{\barZ \barZ^\top}_{\max} 
        \le \normn{\barZ}_{\max}^2 
        \le \frac{t}{N} \max_{j\le d} \Sigma^X_{(j)} \log(j+1)
        = t \Sigma^X_{(1)} \frac{ r_\infty(\Sigma^X)}{N}  \label{eq:E2forMaxNormBd}. 
    \end{align}
    Denote the set on which \eqref{eq:E1forMaxNormBd} occurs by $E_1$, and the set on which \eqref{eq:E2forMaxNormBd} occurs by $E_2$. Then the intersection $E=E_1 \cap E_2$ has probability at least $1-ce^{-t}$, and it holds on $E$ that 
    \begin{align*}
        \normn{\hatSigma^X - \Sigma^X}_{\max}
        &\lesssim 
        \Sigma^X_{(1)}
     \inparen{
    \sqrt{\frac{r_\infty (\Sigma^X)}{N}}
    \lor 
    \frac{r_\infty(\Sigma^X)}{N}
    \lor 
    \sqrt{\frac{t}{N}}
    \lor 
    \frac{t}{N}
    \lor 
    \frac{t r_\infty(\Sigma^X)}{N}
    }\\
    &=
    \Sigma^X_{(1)}
     \inparen{
    \sqrt{\frac{r_\infty (\Sigma^X)}{N}}
    \lor 
    \sqrt{\frac{t}{N}}
    \lor 
    \frac{t}{N}
    \lor 
    \frac{t r_\infty(\Sigma^X)}{N}
    }.
    \qedhere
    \end{align*} 
\end{proof}

\begin{lemma} \label{lem:pMomentEffectiveDim}
   Let $X_1,\dots,X_N$ be $d$-dimensional i.i.d. sub-Gaussian random vectors with $\E[X_1]=\mu^X$ and $\tvar[X_1] = \Sigma^X.$ Let $\hatSigma^X = (N-1)^{-1}\sum_{n=1}^N (X_n-\mu^X)(X_n-\mu^X)^\top$. Then, for any $p \ge 1,$ 
   \begin{align*}
       \insquare{ \E \normn{\hatSigma^X - \Sigma^X}^p_{\max}}^{1/p}
        \lesssim_p 
        \Sigma^X_{(1)}
            \inparen{
            \sqrt{\frac{r_{\infty}(\Sigma^X)}{N}}    
            \lor 
            \frac{r_{\infty}(\Sigma^X)}{N}
            }.
   \end{align*}
\end{lemma}
\begin{proof}
    To ease notation, let $B \equiv 
    \Sigma_{(1)}^X \inparen{
    \sqrt{\frac{r_\infty(\Sigma^X)}{N}}
    \lor 
    \frac{r_\infty(\Sigma^X)}{N}}$, then using that for positive $W$,  $\E [W^p] = p \int_0^\infty w^{p-1} \P(W>w) \, dw$ gives
    \begin{align*}
        \insquare{ \E\normn{\hatSigma - \Sigma}^p_{\max}}
        &= p 
        \int_0^{\infty} x^{p-1}
        \P(\normn{\hatSigma^X - \Sigma^X}_{\max} > x) \, dx \\
        &\le 
        p \int_0^B x^{p-1}dx 
        + 
        \int_B^\infty 
        x^{p-1}
        \P(\normn{\hatSigma^X - \Sigma^X}_{\max} > x) \, dx\\
        &\lesssim 
        B^p 
        + 
        p \int_0^{\infty} 
        x^{p-1}
        \exp 
        \inparen{-\min 
        \inparen{
        \frac{Nx^2}{ (\Sigma^X_{(1)})^2}, 
        \frac{Nx}{\Sigma^X_{(1)}},
        \frac{Nx}{r_{\infty}(\Sigma^X) \Sigma^X_{(1)} }
        }
        }
         \, dx\\
        &=
        B^p + p \max \inparen{
        \frac{\Gamma(p/2)}{2} 
        \inparen{\frac{ (\Sigma^X_{(1)})^2 }{N}}^{p/2}
        ,
        \Gamma(p)
        \inparen{\frac{ \Sigma^X_{(1)} }{N}}^{p}
        ,
        \Gamma(p)
        \inparen{\frac{ r_{\infty}(\Sigma^X) \Sigma^X_{(1)} }{N}}^{p}
        },
    \end{align*}
    where the last line follows by direct integration. We therefore have 
    \begin{align*}
        \insquare{ \E \normn{\hatSigma^X - \Sigma^X}^p_{\max}}^{1/p}
        &\lesssim 
        B + c(p) \max \inparen{
        \frac{ \Sigma^X_{(1)} }{\sqrt{N}},
        \frac{ \Sigma^X_{(1)} }{N},
        \frac{ r_{\infty}(\Sigma^X) \Sigma^X_{(1)} }{N}
        } \le 
        c(p) \Sigma^X_{(1)}
            \inparen{
            \sqrt{\frac{r_{\infty}(\Sigma^X)}{N}}    
            \lor 
            \frac{r_{\infty}(\Sigma^X)}{N}
            },
    \end{align*}
    where the final inequality holds due to the fact that $r_\infty(\Sigma^X) \gtrsim 1$.
\end{proof}

\begin{theorem}[Covariance Estimation with Localized Sample Covariance ---Operator-Norm Bound]\label{thm:SoftSparsityCovarianceBoundGeneral}
Let $X_1,\dots, X_N$ be $d$-dimensional i.i.d. sub-Gaussian random vectors with $\E [X_1] = \mu^X$ and $\tvar[X_1] = \Sigma^X$. Further, assume that $\Sigma^X \in \msU_d(q, R_q)$ for some $q \in [0,1)$ and $R_q > 0$. Let $\hatSigma^X = (N-1)^{-1}\sum_{n=1}^N(X_n - \barX) (X_n - \barX)^\top$ and, for any $t \ge 1$, set 
\begin{align*}
    \rho_N \asymp 
    \Sigma^X_{(1)}
    \inparen{
    \sqrt{\frac{r_\infty (\Sigma^X)}{N}}
    \lor 
    \sqrt{\frac{t}{N}}
    \lor 
    \frac{t}{N}
    \lor 
    \frac{t r_\infty(\Sigma^X)}{N}
    }
\end{align*}
and let $\hatSigma_{\rho_N}^X$ be the localized sample covariance estimator. There exists a constant $c>0$ such that, with probability at least $1-ce^{-t}$, it holds that
\begin{align*}
    \normn{\hatSigma_{\rho_N}^X - \Sigma^X } \lesssim R_q \rho_N^{1-q}.
\end{align*}
\end{theorem}
\begin{proof}
The localized sample covariance matrix has elements 
\begin{align*}
    [\hatSigma_{\rho_N}^X]_{ij}
    = \hatSigma^X_{ij} \indicator_{|\hatSigma^X_{ij}| \ge \rho_N}, \qquad 1 \le i,j \le d.
\end{align*}
By Theorem~\ref{thm:MaxNormEffectiveDim}, it holds with probability at least $1-ce^{-t}$ that 
\begin{align*}
    \normn{\hatSigma^X - \Sigma^X}_{\max} 
    \lesssim  
    \rho_N.
\end{align*}
The remainder of the analysis is carried out conditional on this event, following the approach taken in \cite[Theorem 6.27]{wainwright2019high}. Define the set of indices of the $i$-th row of $\Sigma^X$ that exceed $\rho_N/2$ by
\begin{align*}
    \mcI_i(\rho_N/2) \equiv 
    \inparen{
        j \in \inparen{1,\dots,d} :
        \abs{\Sigma^X_{ij}} \ge \rho_N/2
    }, \qquad i=1,\dots,d.
\end{align*}
We then have 
\begin{align*}
    \normn{\Sigma^X - \hatSigma_{\rho_N}^X}
    &\le \normn{\Sigma^X - \hatSigma_{\rho_N}^X}_\infty\\
    &= \max_{i = 1,\dots, d} 
    \sum_{j=1}^d \abs{\Sigma^X_{ij} - \hatSigma^X_{ij} \indicator_{|\hatSigma^X_{ij}| \ge \rho_N} }\\
    &= \max_{i = 1,\dots, d} \inparen{
    \sum_{j \in \mcI_i(\rho_N/2)} 
    \abs{\Sigma^X_{ij} - \hatSigma^X_{ij} \indicator_{|\hatSigma^X_{ij}| \ge \rho_N}} 
    +
    \sum_{j \notin \mcI_i(\rho_N/2)} 
    \abs{\Sigma^X_{ij} - \hatSigma^X_{ij} \indicator_{|\hatSigma^X_{ij}| \ge \rho_N} }
    },
\end{align*}
where $\hatSigma^X_{ij}$ is element $(i,j)$ of $\hatSigma^X$. For $j \in \mcI_i(\rho_N/2)$, it holds that $|\Sigma^X_{ij}| \ge \rho_N/2$ so that 
\begin{align*}
    \sum_{j \in \mcI_i(\rho_N/2)} 
    \abs{\Sigma^X_{ij} - \hatSigma^X_{ij} \indicator_{|\hatSigma^X_{ij}| \ge \rho_N}} 
    &\le 
    \sum_{j \in \mcI_i(\rho_N/2)} 
    \abs{\Sigma^X_{ij} - \hatSigma^X_{ij}}
    +
    \abs{\hatSigma^X_{ij} - \hatSigma^X_{ij} \indicator_{|\hatSigma^X_{ij}| \ge \rho_N}} 
    \\
    &\le 
    \sum_{j \in \mcI_i(\rho_N/2)} 
    \normn{\Sigma^X_{ij} - \hatSigma^X_{ij}}_{\max}
    +
    \abs{\hatSigma^X_{ij} - \hatSigma^X_{ij} \indicator_{|\hatSigma^X_{ij}| \ge \rho_N}} 
    \\
    &\le \sum_{j \in \mcI_i(\rho_N/2)} \inparen{\frac{\rho_N}{2} + \rho_N}\\
    &= |\mcI_i(\rho_N/2) | \frac{3\rho_N}{2},
\end{align*}
where we have used the fact that 
\begin{align*}
    \abs{ \hatSigma^X_{ij} - \hatSigma^X_{ij}\indicator_{|\hatSigma^X_{ij}| \ge \rho_N}}
    = 0 \times \indicator_{|\hatSigma^X_{ij}| \ge \rho_N}
    + \hatSigma^X_{ij} \times \indicator_{|\hatSigma^X_{ij}| \le \rho_N}
    \le \rho_N.
\end{align*}
Further, since 
\begin{align*}
    R_q \ge \sum_{j=1}^d |\Sigma^X_{ij}|^q \ge   |\mcI_i(\rho_N/2) | \inparen{\frac{\rho_N}{2}}^q,
\end{align*}
it follows that $|\mcI_i(\rho_N/2) | \le 2^{q} \rho_N^{-q} R_q$, and so 
\begin{align*}
    \sum_{j \in \mcI_i(\rho_N/2)} 
    \abs{\Sigma^X_{ij} - \hatSigma^X_{ij} \indicator_{|\hatSigma^X_{ij}| \ge \rho_N}} 
    \le |\mcI_i(\rho_N/2) | \frac{3\rho_N}{2}
    \le  \frac{3}{2} 2^{-q} \rho_N^{1-q} R_q.
\end{align*}

For $j \notin \mcI_i(\rho_N)$, then $|\Sigma^X_{ij}| \le \rho_N/2$ and so 
\begin{align*}
    |\hatSigma^X_{ij}| \le 
    |\hatSigma^X_{ij} - \Sigma^X_{ij}| + |\Sigma^X_{ij}|
    \le 
    \|\hatSigma^X - \Sigma^X \|_{\max} + |\Sigma^X_{ij}|
    \le \frac{\rho_N}{2} + \frac{\rho_N}{2} 
    = \rho_N.
\end{align*}
This implies that $\hatSigma^X_{ij}\indicator_{|\hatSigma^X_{ij}| \ge \rho_N} = 0$, and therefore for $ q \in [0,1)$, since $|\Sigma^X_{ij}|/(\rho_N/2) \le 1$, it holds that 
\begin{align*}
    \sum_{j \notin \mcI_i(\rho_N/2)} 
    \abs{\Sigma^X_{ij} - \hatSigma^X_{ij} \indicator_{|\hatSigma^X_{ij}| \ge \rho_N}}
    &\le \sum_{j \notin \mcI_i(\rho_N/2) } |\Sigma^X_{ij}| 
    = \frac{\rho_N}{2}\sum_{j \notin \mcI_i(\rho_N/2) } \frac{|\Sigma^X_{ij}|}{\frac{\rho_N}{2}}\\
    &\le \frac{\rho_N}{2}\sum_{j \notin \mcI_i(\rho_N/2) } \inparen {\frac{|\Sigma^X_{ij}|}{\rho_N/2} }^q 
    \le \rho_N^{1-q} R_q.
\end{align*}
Combining these two results gives 
\begin{equation*}
    \normn{\Sigma^X - \hatSigma_{\rho_N}^X}
    \le 4 \rho_N^{1-q} R_q.   \qedhere
\end{equation*}
\end{proof}

\begin{proof}[Proof of Theorem \ref{thm:SoftSparistyCovarianceBound}]
The result follows immediately by Theorem~\ref{thm:SoftSparsityCovarianceBoundGeneral}.
\end{proof}

\subsubsection{Cross-Covariance Estimation under Soft Sparsity}\label{sssec:proof4.3}
This subsection contains the proof of Theorem \ref{thm:SoftSparistyCrossCovarianceBound}. The presentation is parallel to that in Subsection \ref{sssec:proof4.2}.  We will use a max-norm cross-covariance estimation bound, analogous to Theorem 
\ref{thm:MaxNormEffectiveDim}. The proof relies on a high probability bound for product function classes that was shown in \cite[Theorem 1.13]{mendelson2016upper}. We present here a simplified version of that more general statement that suffices for our purposes. 

\begin{theorem} \label{thm:MendelsonProducts}
    Let $(\mcX, \mcA, \P)$ be a probability space and consider the random sample $X, X_1,\dots, X_N \iid \P$. Let $\mcF, \mcG$ be two classes of measurable functions on $(\mcX, \mcA)$ such that $0 \in \mcF$ and $0 \in \mcG$. There exist positive universal constants $c_1,c_2, c_3$ such that, for all $t \ge 1$, it holds with probability at least $1-c_1 e^{-c_2 t}$ that
    \begin{align*}
        &\sup_{f \in \mcF, g \in \mcG}\abs{ \frac{1}{N} \sum_{n=1}^N f(X_n)g(X_n) - \E \bigl[ f(X)g(X) \bigr]}
        \\
        & \hspace{1cm}
        \le
        c_3 \insquare{
        \inparen{\frac{t}{N} \lor \sqrt{\frac{t}{N}}} 
        \inparen{
        \sup_{f \in \mcF} \normn{f}_{\psi_2}  
        \gamma_2(\mcG, \psi_2)
        \lor 
        \sup_{g \in \mcG} \normn{g}_{\psi_2}  
        \gamma_2(\mcF, \psi_2)}
        \lor 
        \frac{\gamma_2(\mcF, \psi_2)\gamma_2(\mcG, \psi_2)}{N}
        }.
    \end{align*}
\end{theorem}

\begin{proof}
    For notational brevity, throughout this proof we write $\gamma_2(\mcF)$ instead of $\gamma_2(\mcF, \psi_2)$ and $\mathsf{d}_{\psi_2}(\mcF)$ instead of $\sup_{f \in \mcF} \normn{f}$ and similarly for the class $\mcG$. The result follows by an application of \cite[Theorem 1.13]{mendelson2016upper} and the ensuing remark, which deals with the case $\mcF= \mcG$, but is easily extended to the general case considered here. Together they imply that, for any $u \ge 1$, it holds with probability at least 
    $1-2\exp \inparen{-c
        u^2 \inparen{
        \frac{\gamma_2^2(\mcF)}{ \mathsf{d}^2_{\psi_2}(\mcF)}
        \land 
        \frac{\gamma_2^2(\mcG)}{ \mathsf{d}^2_{\psi_2}(\mcG)}}
        } $
     that 
    \begin{align} \label{eq:RHSProductBound}
        \sup_{f \in \mcF, g \in \mcG}\abs{ \frac{1}{N} \sum_{n=1}^N f(X_n)g(X_n) - \E \bigl[ f(X)g(X) \bigr]}
        \lesssim 
        \frac{u^2}{N} \gamma_2(\mcF) \gamma_2(\mcG)
        + 
        \frac{u}{\sqrt{N}} 
        \inparen{
        \gamma_2(\mcF) \mathsf{d}_{\psi_2}(\mcG) + 
        \gamma_2(\mcG) \mathsf{d}_{\psi_2}(\mcF)
        }.
    \end{align}
    We seek to rewrite \eqref{eq:RHSProductBound} so that all problem specific terms appear only in the upper bound. To this end, let     
    \begin{align*}
        t \equiv u^2 \inparen{
        \frac{\gamma_2^2(\mcF)}{ \mathsf{d}^2_{\psi_2}(\mcF)}
        \land 
        \frac{\gamma_2^2(\mcG)}{ \mathsf{d}^2_{\psi_2}(\mcG)}} 
        \implies 
        u = 
        \sqrt{t} \inparen{
        \frac{ \mathsf{d}_{\psi_2}(\mcF) }{\gamma_2(\mcF) }
        \lor
        \frac{ \mathsf{d}_{\psi_2}(\mcG)}{\gamma_2(\mcG)} 
        }
    \end{align*}
    and note that since $u \ge 1$, it must hold that $t \ge \inparen{\frac{\gamma_2^2(\mcF)}{ \mathsf{d}^2_{\psi_2}(\mcF)} \land \frac{\gamma_2^2(\mcG)}{ \mathsf{d}^2_{\psi_2}(\mcG)}}$. Therefore, for any $t \ge \inparen{\frac{\gamma_2^2(\mcF)}{ \mathsf{d}^2_{\psi_2}(\mcF)} \land \frac{\gamma_2^2(\mcG)}{ \mathsf{d}^2_{\psi_2}(\mcG)}}$, we have that with probability at least $1-2e^{-ct}$, the right-hand side of \eqref{eq:RHSProductBound} becomes 
    \begin{align*}
        \frac{t}{N} 
         \inparen{
        \frac{ \mathsf{d}^2_{\psi_2}(\mcF) }{\gamma^2_2(\mcF) }
        \lor
        \frac{ \mathsf{d}^2_{\psi_2}(\mcG)}{\gamma^2_2(\mcG)} 
        }
        \gamma_2(\mcF) \gamma_2(\mcG) 
        +
        \frac{\sqrt{t}}{\sqrt{N}} 
        \inparen{
        \frac{ \mathsf{d}_{\psi_2}(\mcF) }{\gamma_2(\mcF) }
        \lor
        \frac{ \mathsf{d}_{\psi_2}(\mcG)}{\gamma_2(\mcG)} 
        }
        \inparen{
        \gamma_2(\mcF) \mathsf{d}_{\psi_2}(\mcG) + 
        \gamma_2(\mcG) \mathsf{d}_{\psi_2}(\mcF)
        }.
    \end{align*}
    The above implies that, for any $t \ge 1$, it holds with probability at least 
    $$1-2\exp \inparen{-c \inparen{t \lor \inparen{\frac{\gamma_2^2(\mcF)}{ \mathsf{d}^2_{\psi_2}(\mcF)} \land \frac{\gamma_2^2(\mcG)}{ \mathsf{d}^2_{\psi_2}(\mcG)}}} } \ge 1-2e^{-ct},$$
    that
     \begin{align}
        &\frac{1}{N} 
        \inparen{t \lor \inparen{\frac{\gamma_2^2(\mcF)}{ \mathsf{d}^2_{\psi_2}(\mcF)} \land \frac{\gamma_2^2(\mcG)}{ \mathsf{d}^2_{\psi_2}(\mcG)}}}
         \inparen{
        \frac{ \mathsf{d}^2_{\psi_2}(\mcF) }{\gamma^2_2(\mcF) }
        \lor
        \frac{ \mathsf{d}^2_{\psi_2}(\mcG)}{\gamma^2_2(\mcG)} 
        }
        \gamma_2(\mcF) \gamma_2(\mcG) \label{eq:b1Prod}\\
        +& 
        \frac{1}{\sqrt{N}}
        \inparen{
        \sqrt{t} \lor 
        \inparen{\frac{\gamma_2(\mcF)}{ \mathsf{d}_{\psi_2}(\mcF)} \land \frac{\gamma_2(\mcG)}{ \mathsf{d}_{\psi_2}(\mcG)}}}
        \inparen{
        \frac{ \mathsf{d}_{\psi_2}(\mcF) }{\gamma_2(\mcF) }
        \lor
        \frac{ \mathsf{d}_{\psi_2}(\mcG)}{\gamma_2(\mcG)} 
        }
        \inparen{
        \gamma_2(\mcF) \mathsf{d}_{\psi_2}(\mcG) + 
        \gamma_2(\mcG) \mathsf{d}_{\psi_2}(\mcF)
        }\label{eq:b2Prod}.
    \end{align}
    Straightforward calculations then show that the first of the two terms, \eqref{eq:b1Prod}, is bounded above by
    \begin{align*}
        \frac{t}{N} \frac{ \mathsf{d}^2_{\psi_2}(\mcF)  \gamma_2(\mcG)}{ \gamma_2(\mcF)}
        \lor 
        \frac{t}{N} \frac{ \mathsf{d}^2_{\psi_2}(\mcG)  \gamma_2(\mcF)}{ \gamma_2(\mcG)}
        \lor 
        \frac{\gamma_2(\mcF)\gamma_2(\mcG)}{N},
    \end{align*}
    and \eqref{eq:b2Prod} is similarly bounded above by 
    \begin{align*}
        \sqrt{\frac{t}{N}} \frac{ \mathsf{d}^2_{\psi_2}(\mcF)  \gamma_2(\mcG)}{ \gamma_2(\mcF)}
        \lor 
        \sqrt{\frac{t}{N}} \frac{ \mathsf{d}^2_{\psi_2}(\mcG)  \gamma_2(\mcF)}{\gamma_2(\mcG)}
        \lor 
        \frac{\mathsf{d}_{\psi_2}(\mcG) \gamma_2(\mcF)}{\sqrt{N}}
        \lor 
        \frac{\mathsf{d}_{\psi_2}(\mcF) \gamma_2(\mcG)}{\sqrt{N}}.
    \end{align*} 
    Note then that since $0 \in \mcF$, 
    \begin{align*}
        \mathsf{d}_{\psi_2}(\mcF) 
        = \sup_{f \in \mcF} \normn{f}_{\psi_2}
        \le  \sup_{f_1, f_2 \in \mcF} \normn{f_1 - f_2}_{\psi_2}
        = \tdiam_{\psi_2}(\mcF) \le \gamma_2(\mcF),
    \end{align*}
    where the final equality holds since $\gamma_2(\mcF) = \inf \sup_{ f \in \mcF} \sum_{n=0}^{\infty} 2^{n/2} \tdiam_{\psi_2}(\Delta_n(f))$, and for $n=0$, $\Delta_0 = \mcF$. Similarly, $\mathsf{d}_{\psi_2}(\mcG) \le \gamma_2(\mcG)$, and so $\frac{ \mathsf{d}^2_{\psi_2}(\mcF)  \gamma_2(\mcG)}{ \gamma_2(\mcF)} \le d_{\psi_2}(\mcF)\gamma_2(\mcG)$ and $\frac{ \mathsf{d}^2_{\psi_2}(\mcG)  \gamma_2(\mcF)}{ \gamma_2(\mcG)} \le d_{\psi_2}(\mcG) \gamma_2(\mcF)$ which along with the fact that $t\ge 1$ completes the proof.
\end{proof}

\begin{theorem}[Cross-Covariance Estimation ---Max-Norm Bound] \label{thm:CrossCovarianceEffectiveDim}
    Let $X_1,\dots, X_N$ be $d$-dimensional i.i.d. sub-Gaussian random vectors with $\E [X_1] = \mu^X$ and $\tvar [X_1] = \Sigma^X$. Let $Y_1,\dots, Y_N$ be $k$-dimensional i.i.d. sub-Gaussian random vectors with $\E [Y_1] = \mu^Y$ and $\tvar[Y_1] = \Sigma^Y$. Define $\Sigma^{XY} = \E \bigl[(X-\mu^X)(Y-\mu^Y)^\top \bigr]$ and consider the cross-covariance estimator 
    \begin{align*}
        \hatSigma^{XY} = \frac{1}{N-1} \sum_{n=1}^N (X_n - \barX)(Y_n-\barY)^\top.
    \end{align*}
    Then there exist positive universal constants $c_1, c_2$ such that, for all $t \ge 1,$ it holds with probability at least $1-c_1e^{-c_2t}$ that
    \begin{align*}
        \normn{\hatSigma^{XY} - \Sigma^{XY}}_{\max}
        &\lesssim 
        (\Sigma_{(1)}^X \lor \Sigma_{(1)}^Y)
        \inparen{ 
        \inparen{\frac{t}{N} \lor \sqrt{\frac{t}{N}}}
        \inparen{
        \sqrt{r_\infty(\Sigma^X)}
        \lor 
        \sqrt{r_\infty(\Sigma^Y)}
        }
        \lor 
        \sqrt{    \frac{ r_{\infty}(\Sigma^X)}{N}   } 
        \sqrt{    \frac{ r_{\infty}(\Sigma^Y)}{N}   }
        }.
        \end{align*}
        
\end{theorem}
\begin{proof}
Assume first that $\mu^X=\mu^Y = 0$. Let $Z_1,\dots, Z_N$ be $d$-dimensional i.i.d. sub-Gaussian random vectors with zero mean and $\tvar[Z_1] = \Sigma^X$, and similarly let $V_1,\dots,V_N$ be $k$-dimensional i.i.d. sub-Gaussian random vectors with zero mean and $\tvar[V_1] = \Sigma^Y$. Further, let $W_n \equiv [Z_n^\top, V_n^\top]^\top$ for $n=1,\dots, N$.  We denote the distribution of $W_1$ by $\P$ and note that $\normn{\cdot}_{\psi_2}$ and $\normn{\cdot}_{L_2}$ are defined implicitly with respect to $\P$ throughout this proof. Define  $\hatSigma^{0} = N^{-1}\sum_{n=1}^N Z_n V_n^\top$. Define the dilation operator: $\mcH: \R^{d\times k} \to \R^{(d+k) \times (d+k)}$ by 
\begin{align*}
    \mcH(A) = \begin{bmatrix}
        O& A\\ A^\top& O
    \end{bmatrix},
\end{align*}
see for example \cite[Section 2.1.16]{tropp2015introduction}, and note that $\normn{A}_{\max} = \normn{\mcH(A)}_{\max}$. Let $\mcB^{m}$ be the space of standard basis vectors in $m$ dimensions, i.e. any $b \in \mcB^m$ is an $m$-dimensional vector with $1$ in a single coordinate and $0$ otherwise.  Then, for $e_i, e_j \in \mcB^{d+k}$, we have 

\begin{align*}
    \normn{\hatSigma^{0} - \Sigma^{XY}}_{\max}
    =
    \normn{\mcH(\hatSigma^{0}) - \mcH(\Sigma^{XY})}_{\max}
    &=
    \max_{1 \le i,j \le d+k} 
    \inp{(\mcH(\hatSigma^{0}) - \mcH(\Sigma^{XY})) e_i, e_j } \\
    &\le 2
    \sup_{ u \in \mcU } 
    \abs{\inp{(\mcH(\hatSigma^{0}) - \mcH(\Sigma^{XY}))u, u}},
\end{align*}
where 
\begin{align*}
    \mcU \equiv \biggl\{u \in \R^{d+k}: u= \pm \frac{1}{2}(e_i \pm e_j) \text{ and } e_i, e_j \in \mcB^{d+k} \biggr\}.
\end{align*}
Writing $u=[u_1^\top, u_2^\top]^\top$ where $u_1 \in \R^d$ and $u_2 \in \R^k$, we have 
\begin{align*}
    \inp{\mcH(\hatSigma^{0})u, u}
    = \frac{2}{N}\sum_{n=1}^N \inp{u_1,Z_n} \inp{u_2, V_n}
    = \frac{2}{N}\sum_{n=1}^N f_u(W_n),
\end{align*}
where $f_u(W_n) \equiv \inp{\msA_1 W_n, u_1} \inp{\msA_2 W_n, u_2}$ and where $\msA_1 \equiv [I_d, O_{d\times k}] \in \R^{d \times (d+k)}$ and $\msA_2\equiv [O_{k\times d}, I_{k}] \in \R^{k \times (d+k)}$ are the relevant selection matrices so that $\msA_1 W_n = Z_n$ and $\msA_2 W_n = V_n$. We define the class of functions 
\begin{align*}
    \mcF_{\mcU} \equiv 
    \biggl\{f_u(\cdot) = \inp{\msA_1 \cdot, u_1}\inp{\msA_2 \cdot, u_2} : u=[u_1^\top, u_2^{\top}]^{\top} \in \mcU \biggr\}.
\end{align*}
It is clear then that $\mcF_{\mcU} \subset \mcF_1 \cdot \mcF_2,$ where
\begin{align*}
    \mcU_1 &\equiv 
    \biggl\{u_1 \in \R^d: u_1= \pm \frac{1}{2}(e_i \pm e_j) \text{ and } e_i, e_j \in \mcB_d \biggr\}, 
    \quad
    \mcF_1 \equiv 
    \{f(\cdot) = \inp{\msA_1 \cdot, u_1}: u_1 \in \mcU_1 \}, \\
    \mcU_2 &\equiv 
    \biggl\{u_2 \in \R^k: u_2= \pm \frac{1}{2}(e_i \pm e_j)  \text{ and } e_i, e_j \in \mcB_k \biggr\},
    \quad 
    \mcF_2 \equiv 
    \{f(\cdot) = \inp{\msA_2 \cdot, u_2}: u_2 \in \mcU_2 \},
\end{align*}
and 
\begin{align*}
    \mcF_1 \cdot \mcF_2 \equiv \bigl\{f(\cdot) = f_1(\cdot)f_2(\cdot): f_1\in\mcF_1, f_2 \in \mcF_2 \bigr\}.
\end{align*}

We can then apply the product empirical process concentration bound of Theorem~\ref{thm:MendelsonProducts}, which implies that, with probability $1-c_1e^{-c_2t}$, 
\begin{align}
     \sup_{ u \in \mcU } 
    \abs{\inp{(\mcH(\hatSigma^{0}) - \mcH(\Sigma^{XY}))u, u}} \nonumber 
    =&  \sup_{f_u \in \mcF_{\mcU}} 
    \abs{
        \frac{1}{N} \sum_{n=1}^N f_u(W_n) - \E \bigl[f_u(W_n)\bigr]
    }\nonumber\\
    \lesssim &
        \inparen{\frac{t}{N} \lor \sqrt{\frac{t}{N}}} 
        \inparen{\mathsf{d}_{\psi_2}(\mcF_1)  \gamma_2(\mcF_2)
        \lor 
        \mathsf{d}_{\psi_2}(\mcF_2)  \gamma_2(\mcF_1)}
        \lor
        \frac{\gamma_2(\mcF_1)\gamma_2(\mcF_2)}{N}, \label{eq:CCT1}
\end{align}
where we use the notational shorthand $\gamma_2(\mcF_1) = \gamma_2(\mcF_1, \psi_2)$ and $\mathsf{d}_{\psi_2} (\mcF_1) = \sup_{f \in \mcF_1} \normn{f}_{\psi_2},$ and similarly for $\mcF_2$. Following a similar approach to the one taken in the proof of Theorem~\ref{thm:MaxNormEffectiveDim}, it follows by the equivalence of $\psi_2$ and $L_2$ norms for linear functionals that 
\begin{align*}
    d_{\psi_2}(\mcF_1) 
    = \sup_{f_1 \in \mcF_1} \normn{f_1}_{\psi_2}
    \lesssim 
    \sup_{f_1 \in \mcF_1} \normn{f_1}_{L_2}
    = \max_{u_1 \in \mcU_1} \sqrt{ \langle u_1, \Sigma^X u_1 \rangle} \le \sqrt{\Sigma^{X}_{(1)}},
\end{align*}
and similarly that $d_{\psi_2}(\mcF_2)  \le \sqrt{\Sigma^{Y}_{(1)}}$. Further, 
\begin{align*}
    \gamma_2(\mcF_1)
    = \gamma_2(\mcF_1, \psi_2)  
     \lesssim \gamma_2(\mcF_1, L_2) 
    =\gamma_2(\mcU_1, \mathsf{d}_X),
\end{align*}
 where 
\begin{align*}
    \mathsf{d}_X(u, v) &= \sqrt{\E \insquare{(\inp{g_X,u}- \inp{g_X,v})^2}}, 
    \qquad
    g_X \sim \mcN(0,\Sigma^{X}).
\end{align*}
By Theorem~\ref{thm:TalagrandUpperLowerGP} and Lemma~\ref{lem:DimFreeGaussianMax}, 
\begin{align*}
    \gamma_2(\mcU_1, \mathsf{d}_X)
    &\lesssim  \E \insquare{\sup_{u_1 \in \mcU_1} \inp{g_X, u_1}}
    = \E \insquare{\max_{i,j \le d} \inp{g_X, \pm\frac{1}{2}(e_i \pm e_j)}} \\
    &\le \E \insquare{ \max_{i \le d} \inp{g_X, e_i}}
    \lesssim \max_{i \le d} \sqrt{\Sigma^{X}_{(i)} \log(i+1)}.
\end{align*}
Similarly, $\gamma_2(\mcF_2) \lesssim  \max_{j \le k} \sqrt{\Sigma^{Y}_{(j)} \log(j+1)}.$ In summary, we have that 
\begin{align*}
    \normn{\hatSigma^{0} - \Sigma^{XY}}_{\max}
    & \lesssim 
        \inparen{\frac{t}{N} \lor \sqrt{\frac{t}{N}}} 
        \inparen{
        \sqrt{\Sigma^X_{(1)}}  \max_{j \le k} \sqrt{\Sigma^Y_{(j)} \log (j+1) }
        \lor 
         \sqrt{\Sigma^Y_{(1)}}  \max_{i \le d} \sqrt{\Sigma^X_{(i)} \log (i+1) }
         }
         \nonumber \\
        &
        \lor
        \frac{\max_{i \le d} \sqrt{\Sigma^X_{(i)} \log (i+1)}  \max_{j \le k} \sqrt{\Sigma^Y_{(j)} \log (j+1)}     }{N}\\
        &\lesssim 
        \inparen{\frac{t}{N} \lor \sqrt{\frac{t}{N}}}  
        \inparen{
        \Sigma_{(1)}^{X} \sqrt{r_\infty(\Sigma^X)}
        \lor 
        \Sigma_{(1)}^{Y} \sqrt{r_\infty(\Sigma^Y)}
        }
        \lor 
        \sqrt{
        \frac{ \Sigma_{(1)}^X r_{\infty}(\Sigma^X)}{N}
        } 
        \sqrt{
        \frac{\Sigma_{(1)}^Y r_{\infty}(\Sigma^Y)}{N}
        }\\
        &\lesssim 
        (\Sigma_{(1)}^X \lor \Sigma_{(1)}^Y)
        \inparen{
        \inparen{\frac{t}{N} \lor \sqrt{\frac{t}{N}}}  
        \inparen{
         \sqrt{r_\infty(\Sigma^X)}
        \lor 
        \sqrt{r_\infty(\Sigma^Y)}
        }
        \lor 
        \sqrt{
        \frac{ r_{\infty}(\Sigma^X)}{N}
        } 
        \sqrt{
        \frac{ r_{\infty}(\Sigma^Y)}{N}
        }}.
\end{align*}
 In the un-centered case, take $X_n = Z_n + \mu^X$ and $Y_n = V_n + \mu^Y$ for $n=1,\dots, N$, then $\hatSigma^{XY} = \hatSigma^{0} - \barX \barY^\top$, and so 
\begin{align*}
    \normn{\hatSigma^{XY} - \Sigma^{XY}}_{\max}
    &\le \normn{\hatSigma^{0} - \Sigma^{XY}}_{\max} +  \normn{\barX \barY^\top}_{\max}.
\end{align*}
The first term is controlled by appealing to the result in the centered case. For the second term, we note that by Lemma~\ref{lem:DimFreeGaussianMax} 
\begin{align*}
    \normn{\barX \barY^\top}_{\max} 
    \le \normn{\barX}_{\max} \normn{\barY}_{\max}
    &\lesssim 
    \frac{1}{N}
    \max_{i\le d} \sqrt{\Sigma^X_{(i)} \log(i+1)}
    \max_{j\le k} \sqrt{\Sigma^Y_{(j)} \log(j+1)}\\
    &\le 
    \sqrt{
        \frac{ \Sigma_{(1)}^X r_{\infty}(\Sigma^X)}{N}
        } 
        \sqrt{
        \frac{\Sigma_{(1)}^Y r_{\infty}(\Sigma^Y)}{N}
    }.  \qedhere
\end{align*}
\end{proof}

\begin{theorem}[Cross-Covariance Estimation with Localized Sample Cross-Covariance  ---Operator-Norm bound]\label{thm:SoftSparsityCrossCovarianceBoundGeneric}
    Let $X_1,\dots, X_N$ be $d$-dimensional i.i.d. sub-Gaussian random vectors with $\E [X_1]=\mu^X$ and $\tvar[X_1] = \Sigma^X$. Let $Y_1, \dots, Y_N$ be $k$-dimensional i.i.d. sub-Gaussian random vectors with $\E [Y_1]=\mu^Y$ and $\tvar[Y_1] = \Sigma^Y$. Define $\Sigma^{XY} = \E \bigl[(X-\mu^X)(Y-\mu^Y)^\top \bigr]$ and consider the estimator 
    \begin{align*}
        \hatSigma^{XY} = \frac{1}{N-1} \sum_{n=1}^N (X_n - \barX)(Y_n-\barY)^\top.
    \end{align*}
    Assume that $\Sigma^{XY} \in \mcU_{d, k}(q_1, R_{q_1})$ and  $\Sigma^{YX} \in \mcU_{k,d}(q_2, R_{q_2})$ where $q_1,q_2 \in [0,1)$ and $R_{q_1}, R_{q_2}$ are positive constants. For any $t\ge 1$, set 
    \begin{align*}
        \rho_N \asymp 
       (\Sigma_{(1)}^X \lor \Sigma_{(1)}^Y)
        \inparen{ 
        \inparen{\frac{t}{N} \lor \sqrt{\frac{t}{N}}}
        \inparen{
        \sqrt{r_\infty(\Sigma^X)}
        \lor 
        \sqrt{r_\infty(\Sigma^Y)}
        }
        \lor 
        \sqrt{    \frac{ r_{\infty}(\Sigma^X)}{N}   } 
        \sqrt{    \frac{ r_{\infty}(\Sigma^Y)}{N}   }
        },
    \end{align*}
    and let $\hatSigma^{XY}_{\rho_N}$ be the localized sample cross-covariance estimator. There exist positive universal constants $c_1,c_2$ such that, with probability at least $1-c_1e^{-c_2t}$,
    \begin{align*}
        \normn{\hatSigma^{XY}_{\rho_N} - \Sigma^{XY} } \lesssim R_{q_1}\rho_N^{1-q_1} \lor R_{q_2} \rho_N^{1-q_2}.
    \end{align*}
\end{theorem}
\begin{proof}[Proof of Theorem~\ref{thm:SoftSparsityCrossCovarianceBoundGeneric}]
    Let $E$ denote the event on which $\normn{\hatSigma^{XY} - \Sigma^{XY}}_{\max}=\normn{\hatSigma^{YX} - \Sigma^{YX}}_{\max} \lesssim \rho_N$. By Theorem~\ref{thm:CrossCovarianceEffectiveDim}, $E$ holds with probability at least $1-c_1e^{-c_2 t}$. Conditional on $E$, and following an analysis identical to the one in the proof of Theorem~\ref{thm:SoftSparsityCovarianceBoundGeneral} with $\hatSigma^{XY} (\hatSigma^{YX})$ and $\Sigma^{XY} (\Sigma^{YX})$ in place of $\hatSigma^X$ and $\Sigma^X$ respectively, it follows that
    \begin{align*}
        \normn{\hatSigma^{XY}_{\rho_N} - \Sigma^{XY}}_\infty &\lesssim R_{q_1} \rho_N^{1-q_1},
    \end{align*}
    and
     \begin{align*}
        \normn{\hatSigma^{YX}_{\rho_N} - \Sigma^{YX}}_\infty \lesssim R_{q_2} \rho_N^{1-q_2}.
    \end{align*}
    The result then follows by noting that
    \begin{align*}
        \normn{\hatSigma^{XY}_{\rho_N} - \Sigma^{XY} } 
        &= \normn{\mcH (\hatSigma^{XY}_{\rho_N} - \Sigma^{XY} )} 
        \le \normn{\mcH (\hatSigma^{XY}_{\rho_N} - \Sigma^{XY} )}_\infty \\
        &= \normn{\hatSigma^{XY}_{\rho_N} - \Sigma^{XY}}_\infty \lor \normn{\hatSigma^{YX}_{\rho_N} - \Sigma^{YX}}_\infty
        \lesssim R_{q_1}\rho_N^{1-q_1} \lor R_{q_2} \rho_N^{1-q_2},
    \end{align*} 
    where $\mcH$ is the dilation operator defined in the proof of Theorem~\ref{thm:CrossCovarianceEffectiveDim}.
\end{proof}

\begin{proof}[Proof of Theorem \ref{thm:SoftSparistyCrossCovarianceBound}]
    The proof follows immediately from Theorem~\ref{thm:SoftSparsityCrossCovarianceBoundGeneric}: since $\upr_1, \dots, \upr_N$ are i.i.d. Gaussian they are sub-Gaussian. Moreover, since $\mcG$ is Lipschitz, by \cite[Theorem 5.2.2]{vershynin2018high}, $\normn{\mcG(\upr_1) - \E[\mcG(\upr_1)]}_{\psi_2} \le \normn{\mcG}_{\text{Lip}} \normn{\Cpr}^{1/2} < \infty$, and so $\mcG(\upr_1), \dots, \mcG(\upr_N)$ are i.i.d. sub-Gaussian random vectors.
\end{proof}

\begin{lemma}[Stein's Lemma \cite{stein1972bound}]\label{lem:Stein} 
   Let $\upr \sim \mcN(\mpr, \Cpr)$ be a $d$-dimensional Gaussian vector. Let $h:\R^d \to \R$ such that $\partial_j h \equiv \partial h(u)/\partial u_j$ exists almost everywhere and $\E[|\partial_j h(\upr)|] < \infty$, $j=1,\dots, d$. Then 
    \begin{align*}
        \text{Cov}\bigl(\upr_j, h(\upr)\bigr) = \sum_{l=1}^d \Cpr_{jl} \E [\partial_l h(\upr)].
    \end{align*}
\end{lemma}

\begin{lemma}[Soft-Sparsity of Cross-Covariance ---Nonlinear Forward Map]\label{lem:CrossCovarianceStein} 
   Let $\upr$ be a $d$-dimensional Gaussian random vector with $\E[\upr]=\mpr$ and $\tvar[\upr] = \Cpr \in \msU_d(q, c)$. Consider the function $\mcG:\R^d \to \R^k$ with coordinate functions $\mcG_1,\dots, \mcG_k$. Assume that for each $i=1,\dots, d$ and $j=1,\dots,k$, $\mcG_j:\R^d \to \R$ for $j=1,\dots, k$, such that $\partial_i \mcG_j \equiv \partial \mcG_j(\upr)/\partial \upr_i$ exists almost everywhere, and $\E [|\partial_i \mcG_j|] < \infty$. Let $D \mcG \in \R^{k \times d}$ denote the Jacobian of $\mcG$, and assume that $\E \bigl[(D \mcG)^\top \bigr] \in \msU_{d, k}(q , a)$ for some $q \in [0,1)$ and $a > 0$. Then, 
   \begin{align*}
       \Cpr^{\,up} \in \msU_{d,k} \bigl(q, ac \normn{\E [ D \mcG]}_{\max}^{1-q} \normn{\Cpr}_{\max}^{1-q}\bigr).
   \end{align*}
\end{lemma}
\begin{proof}
    By Stein's Lemma (Lemma~\ref{lem:Stein}), the $i$-th row sum of $\Cpr^{\, up}$ is given by 
    \begin{align*}
    \sum_{j=1}^k C^{\,up}_{ij} 
    = 
    \sum_{j=1}^k \sum_{l=1}^d \Cpr_{il} \E \bigl[\partial_l \mcG_j(\upr)\bigr]
    &=
    \sum_{l=1}^d \Cpr_{il} \sum_{j=1}^k  \E \bigl[\partial_l \mcG_j(\upr)\bigr]\\
    &= \normn{\E[D \mcG]}_{\max} \sum_{l=1}^d \Cpr_{il} \sum_{j=1}^k  \frac{\E [\partial_l \mcG_j(\upr)]}{\normn{\E[D \mcG]}_{\max}}\\
    &\le \normn{\E [D \mcG]}_{\max}^{1-q} \sum_{l=1}^d \Cpr_{il} \sum_{j=1}^k \E [\partial_l \mcG_j(\upr)]^q\\
    &\le a \normn{\E [D \mcG]}_{\max}^{1-q} \sum_{l=1}^d \Cpr_{il}\\
    &\le ac \normn{\E [D \mcG]}_{\max}^{1-q} \normn{\Cpr}_{\max}^{1-q},
\end{align*}
where the first inequality holds since $q \in [0,1)$ and $\E [\partial_l \mcG_j(\upr)] \le \normn{\E[D \mcG]}_{\max}$.
\end{proof}

\begin{lemma}[Product of Two Soft-Sparse Matrices] \label{lem:BSsparsity}
Fix $q\in [0,1)$ and let $S \in \msU_d(q, s)$ and assume $S^\top =S$. Let $B \in \msU_{k,d}(q, b)$. Then $BS \in \msU_{k,d}(q,b s \normn{B}_{\max}^{1-q} \normn{S}_{\max}^{1-q})$.
\end{lemma}
\begin{proof}
    The $(i,j)$-th element of $BS$ is given by $[BS]_{ij} = \sum_{l=1}^d B_{il}S_{lj}$, and so the sum of the $i$-th row of $BS$ satisfies
    \begin{align*}
        \sum_{j=1}^d [BS]_{ij}
        = \sum_{j=1}^d \sum_{l=1}^d B_{il}S_{lj}
        =  \sum_{l=1}^d B_{il} \sum_{j=1}^d S_{lj}
        &= \normn{B}_{\max}\normn{S}_{\max} \sum_{l=1}^d \frac{B_{il}}{\normn{B}_{\max}} \sum_{j=1}^d \frac{S_{lj}}{\normn{S}_{\max}}\\
        &\le \normn{B}_{\max}\normn{S}_{\max} 
        \sum_{l=1}^d \inparen{\frac{B_{il}}{\normn{B}_{\max}}}^q \sum_{j=1}^d \inparen{\frac{S_{lj}}{\normn{S}_{\max}}}^q\\
        &\le \normn{B}_{\max}^{1-q} \normn{S}_{\max}^{1-q} b s,
    \end{align*}
    where the first inequality holds since $q \in [0,1)$, and the second follows by the symmetry of $S$.
\end{proof}

\begin{lemma}[Product of Three Soft-Sparse Matrices]\label{lem:BSBsparsity}
Fix $q\in [0,1)$ and let $S \in \msU_d(q, s)$ with $S^\top =S$. Let $B \in \msU_{k,d}(q, b_1)$ and $B^\top \in \msU_{d,k}(q,b_2)$, that is, $B$ is both row and column sparse.  Then $BSB^\top \in \msU_{k,d}(q, b_1 b_2 s \normn{B}^{2(1-q)}_{\max} \normn{S}_{\max}^{1-q})$.
\end{lemma}
\begin{proof}
    The $(i,j)$-th element of $BSB^\top$ is given by 
    \begin{align*}
        [BSB^\top]_{ij} 
        = \sum_{m=1}^d [BS]_{im} B^\top_{mj}
        = \sum_{m=1}^d [BS]_{im} B_{jm}
        = \sum_{m=1}^d \inparen{\sum_{l=1}^d B_{il}S_{lm} }  B_{jm}.
    \end{align*}
    Therefore, the sum of the $i$-th row of $BSB^\top$ satisfies
    \begin{align*}
        \sum_{j=1}^k [BSB^\top]_{ij}
        &= \sum_{j=1}^k \sum_{m=1}^d \sum_{l=1}^d B_{il}S_{lm}  B_{jm}
        =  \sum_{m=1}^d \sum_{l=1}^d B_{il}S_{lm} \sum_{j=1}^k B_{jm} \\
        &\le \normn{B}^{1-q}_{\max} b_2  \sum_{m=1}^d \sum_{l=1}^d B_{il}S_{lm}
        \le b_1 b_2 s \normn{B}^{2(1-q)}_{\max} \normn{S}_{\max}^{1-q},
    \end{align*}
    where the final inequality follows by Lemma~\ref{lem:BSsparsity}.
\end{proof}

\begin{lemma}[Sample Covariance Deviation] \label{lemma:SampleCovDeviationN}
Let $X_1,\dots,X_N$ be $d$-dimensional i.i.d. sub-Gaussian random vectors with $\E[X_1]=\mu^X$ and $\tvar[X_1]= \Sigma^X.$ Let $\hatSigma^X_N = (N-1)^{-1}\sum_{n=1}^N (X_n-\mu^X)(X_n-\mu^X)^\top$. Then 
\begin{align*}
    \hatSigma_N^X - \hatSigma_{N-1}^X
&\asymp 
\frac{1}{N} X_N X_N^\top
    -\frac{1}{N} \hatSigma^0_{N-1}
    -
    \frac{1}{N^2} X_N X_N^\top 
    -
    \inparen{\inparen{\frac{N-1}{N}}^2-1}
    \barX_{N-1}\barX_{N-1}^\top \\
    &-
    \inparen{\frac{N-1}{N^2}}
    \inparen{
    X_N \barX_{N-1}^\top
    +
    \barX_{N-1} X_N^\top
    },
\end{align*}
where $\hatSigma^0_N=N^{-1} \sum_{n=1}^N X_n X_n^\top$.
\end{lemma}
\begin{proof}
    We work with the biased sample covariance estimator $\frac{1}{N} \sum_{n=1}^n (X_n - \barX_N)(X_n - \barX_N)^\top,$ which is equivalent to the unbiased covariance estimator up to constants. Note then that 
\begin{align*}
    \hatSigma_N^X
    &\asymp \frac{1}{N} \sum_{n=1}^N (X_n - \barX_N)(X_n - \barX_N)^\top
    = \frac{1}{N} \sum_{n=1}^N X_nX_n^\top - \barX_N\barX_N^\top
    = \hatSigma_N^0 - \barX_N\barX_N^\top.
\end{align*}
We now seek to control the difference  $\hatSigma_N^X - \hatSigma_{N-1}^X.$ To that end, note that 
\begin{align*}
    \barX_N \barX_N^\top
    &=
    \inparen{\frac{1}{N} X_N + \frac{N-1}{N}\barX_{N-1}}
    \inparen{\frac{1}{N} X_N + \frac{N-1}{N}\barX_{N-1}}^\top\\
    &=
    \frac{1}{N^2} X_N X_N^\top 
    +
    \inparen{\frac{N-1}{N} }^2
    \barX_{N-1}\barX_{N-1}^\top 
    +
    \inparen{\frac{N-1}{N^2}}
    \inparen{
    X_N \barX_{N-1}^\top
    +
    \barX_{N-1} X_N^\top
    },
\end{align*}
and so
\begin{align}\label{eq:diffSampleMeans}
    \barX_N\barX_N^\top - \barX_{N-1}\barX_{N-1}^\top
    &=
    \frac{1}{N^2} X_N X_N^\top 
    +
    \inparen{\inparen{\frac{N-1}{N}}^2-1}
    \barX_{N-1}\barX_{N-1}^\top \nonumber\\
    &+
    \inparen{\frac{N-1}{N^2}}
    \inparen{
    X_N \barX_{N-1}^\top
    +
    \barX_{N-1} X_N^\top
    }.
\end{align}
Therefore, 
\begin{align} \label{eq:diffSampleCovars}
    \hatSigma_N^X - \hatSigma_{N-1}^X
    &\asymp 
    \inparen{\frac{1}{N} \sum_{n=1}^N X_nX_n^\top - \barX_N\barX_N^\top}
    -
    \inparen{\frac{1}{N-1} \sum_{n=1}^{N-1} X_n X_{n}^\top - \barX_{N-1}\barX_{N-1}^\top} \nonumber\\
    &= 
    \frac{1}{N} X_N X_N^\top
    +
    \inparen{ \inparen{\frac{1}{N}-\frac{1}{N-1}} \sum_{n=1}^{N-1} X_nX_n^\top }
    +\inparen{ \barX_{N-1}\barX_{N-1}^\top- \barX_N\barX_N^\top}\nonumber\\
    &= 
    \frac{1}{N} X_N X_N^\top
    -\frac{1}{N} \hatSigma^0_{N-1}
    -
    \frac{1}{N^2} X_N X_N^\top 
    -
    \inparen{\inparen{\frac{N-1}{N}}^2-1}
    \barX_{N-1}\barX_{N-1}^\top \nonumber\\
    &-
    \inparen{\frac{N-1}{N^2}}
    \inparen{
    X_N \barX_{N-1}^\top
    +
    \barX_{N-1} X_N^\top
    },
\end{align}
where the last equality follows by \eqref{eq:diffSampleMeans}.
\end{proof}

\begin{lemma}[Sample Cross-Covariance Deviation] \label{lemma:SampleCrossCovDeviationN}
Let $X_1,\dots,X_N$ be $d$-dimensional i.i.d. sub-Gaussian random vectors with $\E[X_1]=\mu^X$ and $\tvar[X_1]= \Sigma^X.$ Let $Y_1,\dots,Y_N$ be $k$-dimensional i.i.d. sub-Gaussian random vectors with $\E[Y_1]=\mu^Y$ and $\tvar[Y_1]= \Sigma^Y.$ Let $\hatSigma^{XY}_N = (N-1)^{-1}\sum_{n=1}^N (X_n-\mu^X)(Y_n-\mu^Y)^\top$. Then 
\begin{align*}
    \hatSigma_N^{XY} - \hatSigma_{N-1}^{XY}
&\asymp 
\frac{1}{N} X_N Y_N^\top
    -\frac{1}{N} \hatSigma^{0,XY}_{N-1}
    -
    \frac{1}{N^2} X_N Y_N^\top 
    -
    \inparen{\inparen{\frac{N-1}{N}}^2-1}
    \barX_{N-1}\barY_{N-1}^\top \\
    &-
    \inparen{\frac{N-1}{N^2}}
    \inparen{
    X_N \barY_{N-1}^\top
    +
    \barX_{N-1} Y_N^\top
    },
\end{align*}
where $\hatSigma^{0,XY}_{N-1} = N^{-1}\sum_{n=1}^N X_nY_n$.
\end{lemma}
\begin{proof}
    The result follows using the same approach utilized in the proof of Lemma~\ref{lemma:SampleCovDeviationN} and is omitted for brevity.
\end{proof}

\begin{lemma}[Covariance Estimation with Known Particle --- Operator-Norm Bound] \label{lemma:SampleCovDeviationNOpNorm}
Consider the set-up in Lemma~\ref{lemma:SampleCovDeviationN} and assume additionally that $X_n$ is known for some $n \in \{1,\dots, N\}$. Then with probability at least $1-ce^{-t}$
\begin{align*}
    \normn{\hatSigma^X_N - \Sigma^X}
&\lesssim
\frac{c(\normn{X_n}_2,\normn{\mu^X}_2)}{N}
    +
    \normn{\Sigma^X} \inparen{
    \sqrt{\frac{r_2(\Sigma^X)}{N}} 
    \lor \frac{r_2(\Sigma^X)}{N} 
    \lor \sqrt{\frac{t}{N}} 
    \lor \frac{t}{N}
    }.
\end{align*}
\end{lemma}
\begin{proof}
    By symmetry, we can assume without loss of generality that $n = N$. Let $E_1$ denote the event on which $\normn{\barX_{N-1} - \mu^X}_2 \lesssim \sqrt{\normn{\Sigma^X}\frac{r_2(\Sigma^X) \lor t}{N-1}} \asymp  \sqrt{\normn{\Sigma^X}\frac{r_2(\Sigma^X) \lor t}{N}}$. Then by Theorem~\ref{thm:SubGaussianConcentration}, $\P(E_1) \ge 1-e^{-t}$ and on $E_1$ it holds that 
\begin{align}\label{eq:diffsCovarsOpNorm}
    \normn{\hatSigma^X_N - \hatSigma^X_{N-1}}
    &\lesssim 
    \frac{1}{N} \normn{X_N X_N^\top}
    +\frac{1}{N} \normn{\hatSigma^0_{N-1}}
    +\inparen{\inparen{\frac{N-1}{N}}^2-1}
    \normn{\barX_{N-1} \barX_{N-1}^\top}
    \nonumber \\
    &
    +\frac{N-1}{N^2} \inparen{ 
    \normn{X_{N} \barX_{N-1}^\top}+ \normn{\barX_{N-1} X_N^\top} 
    }\nonumber\\
    &\lesssim 
    \frac{1}{N} \normn{X_N}^2_2
    +\frac{1}{N} \normn{\hatSigma^0_{N-1}}
    +\frac{1}{N}
    \normn{\barX_{N-1}}^2_2
    +\frac{1}{N} 
    \normn{X_{N}}_2 \normn{\barX_{N-1}}_2
    \nonumber\\
    &\le
    \frac{1}{N} \normn{X_N}^2_2
    +\frac{1}{N} \normn{\hatSigma^0_{N-1} - \Sigma^X}
    +\frac{1}{N} \normn{\Sigma^X}
    +\frac{1}{N}
    \normn{\barX_{N-1}-\mu^X}^2_2
    +\frac{1}{N} \normn{\mu^X}^2_2 \nonumber\\
    &+\frac{1}{N} 
    \normn{X_{N}}_2 \normn{\barX_{N-1} - \mu^X}_2
    +\frac{1}{N}\normn{X_{N}}_2 \normn{\mu^X}_2
    \nonumber\\
    &\lesssim
    \frac{1}{N} \normn{X_N}^2_2
    +\frac{1}{N} \normn{\hatSigma^0_{N-1} - \Sigma^X}
    +\frac{1}{N} \normn{\Sigma^X}
    +
    \normn{\Sigma^X}\frac{r_2(\Sigma^X) \lor t}{N^2}
    +\frac{1}{N} \normn{\mu^X}^2_2 \nonumber\\
    &+
    \normn{X_{N}}_2 
    \sqrt{\normn{\Sigma^X}}\frac{\sqrt{r_2(\Sigma^X) \lor t}}{N^{3/2}}
    +\frac{1}{N}\normn{X_{N}}_2 \normn{\mu^X}_2,
\end{align}
where the first line follows by Lemma~\ref{lemma:SampleCovDeviationN}. Let $E_2$ denote the event on which
\begin{align*}
    \normn{\hatSigma^X_{N-1} - \Sigma^X} 
    &\lesssim \normn{\Sigma^X} \inparen{
    \sqrt{\frac{r_2(\Sigma^X)}{N-1}} \lor \frac{r_2(\Sigma^X)}{N-1} \lor \sqrt{\frac{t}{N-1}} \lor \frac{t}{N-1}
    }\\
    &\asymp \normn{\Sigma^X} \inparen{
    \sqrt{\frac{r_2(\Sigma^X)}{N}} 
    \lor \frac{r_2(\Sigma^X)}{N} 
    \lor \sqrt{\frac{t}{N}} 
    \lor \frac{t}{N}
    }.
\end{align*}
Then by Proposition~\ref{thm:Koltchinski}, $\P(E_2) \ge 1-ce^{-t}$. It holds on $E_1\cap E_2$ that 
\begin{align*}
    \normn{\hatSigma^X_N - \Sigma^X}
    &\le 
    \normn{\hatSigma^X_{N} - \hatSigma^X_{N-1}}
    +
    \normn{\hatSigma^X_{N-1} - \Sigma^X}\\
    &\lesssim
    \frac{1}{N} \normn{X_N}^2_2
    +\frac{1}{N} \normn{\hatSigma^0_{N-1} - \Sigma^X}
    +\frac{1}{N} \normn{\Sigma^X}
    +
    \normn{\Sigma^X}\frac{r_2(\Sigma^X) \lor t}{N^2}
    +\frac{1}{N} \normn{\mu^X}^2_2 \nonumber\\
    &+
    \normn{X_{N}}_2 
    \sqrt{\normn{\Sigma^X}}\frac{\sqrt{r_2(\Sigma^X) \lor t}}{N^{3/2}}
    +\frac{1}{N}\normn{X_{N}}_2 \normn{\mu^X}_2 
    +
    \normn{\Sigma^X} \inparen{
    \sqrt{\frac{r_2(\Sigma^X)}{N}} 
    \lor \frac{r_2(\Sigma^X)}{N} 
    \lor \sqrt{\frac{t}{N}} 
    \lor \frac{t}{N}
    }\\
    &\lesssim
    \frac{c(\normn{X_N}_2,\normn{\mu^X}_2)}{N}
    +
    \normn{\Sigma^X} \inparen{
    \sqrt{\frac{r_2(\Sigma^X)}{N}} 
    \lor \frac{r_2(\Sigma^X)}{N} 
    \lor \sqrt{\frac{t}{N}} 
    \lor \frac{t}{N}
    },
\end{align*}
where the first inequality holds by $\eqref{eq:diffsCovarsOpNorm}$. The result follows by noting that $\P(E_1 \cap E_2) \ge 1-ce^{-t}.$
\end{proof}

\begin{lemma}[Covariance Estimation with Known Particle --- Maximum-Norm Bound] \label{lemma:SampleCovDeviationNMaxNorm}
Consider the set-up in Lemma~\ref{lemma:SampleCovDeviationN} and assume additionally that $X_n$ is known for some $n \in \{1,\dots, N\}$. Then with probability at least $1-ce^{-t}$
\begin{align*}
    \normn{\hatSigma^X_N - \Sigma^X}_{\max}
&\lesssim
\frac{c(\normn{X_n}_\infty,\normn{\mu^X}_\infty)}{N}
    +
   \Sigma^X_{(1)}
    \inparen{
    \sqrt{\frac{r_\infty (\Sigma^X)}{N}}
    \lor 
    \sqrt{\frac{t}{N}}
    \lor 
    \frac{t}{N}
    \lor 
    \frac{t r_\infty(\Sigma^X)}{N}
    }.
\end{align*}
\end{lemma}
\begin{proof}
 As in the proof of Lemma~\ref{lemma:SampleCovDeviationNOpNorm}, we may assume that $n=N$. Let $E_1$ denote the event on which $\normn{\barX_{N-1} - \mu^X}_\infty \lesssim \sqrt{t \Sigma^X_{(1)} \frac{r_\infty(\Sigma^X) }{N-1}} \asymp  \sqrt{t \Sigma^X_{(1)} \frac{r_\infty(\Sigma^X) }{N}}$. Then by Lemma~\ref{lem:DimFreeGaussianMax}, $\P(E_1) \ge 1-ce^{-ct}$ and on $E_1$, using similar calculations to those used to derive \eqref{eq:diffsCovarsOpNorm}, it holds that
\begin{align}
    \normn{\hatSigma^X_N - \hatSigma^X_{N-1}}_{\max}
    &\lesssim
    \frac{1}{N} \normn{X_N}^2_\infty
    +\frac{1}{N} \normn{\hatSigma^0_{N-1} - \Sigma^X}_{\max}
    +\frac{1}{N} \normn{\Sigma^X}_{\max}
    +
    t \Sigma^X_{(1)} \frac{r_\infty(\Sigma^X) }{N^2}
    +\frac{1}{N} \normn{\mu^X}^2_\infty \nonumber\\
    &+
    \normn{X_{N}}_\infty 
    \sqrt{t \Sigma^X_{(1)}} \frac{\sqrt{r_\infty(\Sigma^X)} }{N^{3/2}}
    +\frac{1}{N}\normn{X_{N}}_\infty \normn{\mu^X}_\infty. \label{eq:diffsCovarsMaxNorm}
\end{align}
Let $E_2$ be the event on which 
\begin{align*}
    \normn{\hatSigma^X_{N-1} - \Sigma^X}_{\max} 
    \lesssim
    \Sigma^X_{(1)}
    \inparen{
    \sqrt{\frac{r_\infty (\Sigma^X)}{N}}
    \lor 
    \sqrt{\frac{t}{N}}
    \lor 
    \frac{t}{N}
    \lor 
    \frac{t r_\infty(\Sigma^X)}{N}
    }. 
\end{align*}
By Theorem~\ref{thm:MaxNormEffectiveDim}, $\P(E_2) \ge 1-ce^{-t}$. Finally, note that the desired result holds on $E_1 \cap E_2$ and that  $\P(E_1 \cap E_2) \ge 1-ce^{-t},$ which completes the proof. 
\end{proof}

\begin{lemma}[Cross-Covariance Estimation with Known Particle --- Maximum-Norm Bound] \label{lemma:SampleCrossCovDeviationNMaxNorm}
Consider the set-up in Lemma~\ref{lemma:SampleCrossCovDeviationN} and assume additionally that $(X_n,Y_n)$ is known for some $n \in \{1,\dots, N\}$. Then with probability at least $1-ce^{-t}$
\begin{align*}
    \normn{\hatSigma^{XY}_N - \Sigma^{XY}}_{\max}
&\lesssim
\frac{c(\normn{X_n}_\infty,\normn{\mu^X}_\infty,\normn{Y_n}_\infty,\normn{\mu^Y}_\infty)}{N}\\
    &+
   (\Sigma_{(1)}^X \lor \Sigma_{(1)}^Y)
        \inparen{ 
        \inparen{\frac{t}{N} \lor \sqrt{\frac{t}{N}}}
        \inparen{
        \sqrt{r_\infty(\Sigma^X)}
        \lor 
        \sqrt{r_\infty(\Sigma^Y)}
        }
        \lor 
        \sqrt{    \frac{ r_{\infty}(\Sigma^X)}{N}   } 
        \sqrt{    \frac{ r_{\infty}(\Sigma^Y)}{N}   }
        }.
\end{align*}
\end{lemma}
\begin{proof}
    The result follows using the same approach utilized in the proof of Lemma~\ref{lemma:SampleCovDeviationN} and utilizing the statements of Lemma~\ref{lemma:SampleCrossCovDeviationN} and Theorem~\ref{thm:CrossCovarianceEffectiveDim}. We omit the details for brevity.
\end{proof}

\begin{lemma}[Covariance Estimation with Localized Sample Covariance and with Known Particle ---Operator-Norm Bound]\label{lem:SoftSparsityCovarianceBoundGeneralConditional}
Consider the set-up in Theorem~\ref{thm:SoftSparsityCovarianceBoundGeneral} and assume additionally that $X_n$ is known for some $n \in \{1,\dots,N\}$. For any $t \ge 1$, set 
\begin{align*}
    \rho_N \asymp 
    \frac{c(\normn{X_n}_\infty,\normn{\mu^X}_\infty)}{N}
    +
    \Sigma^X_{(1)}
    \inparen{
    \sqrt{\frac{r_\infty (\Sigma^X)}{N}}
    \lor 
    \sqrt{\frac{t}{N}}
    \lor 
    \frac{t}{N}
    \lor 
    \frac{t r_\infty(\Sigma^X)}{N}
    }
\end{align*}
and let $\hatSigma_{\rho_N}^X$ be the localized sample covariance estimator. There exists a constant $c>0$ such that, with probability at least $1-ce^{-t}$, it holds that
\begin{align*}
    \normn{\hatSigma_{\rho_N}^X - \Sigma^X } \lesssim R_q \rho_N^{1-q}.
\end{align*}
\end{lemma}
\begin{proof}
The proof follows in identical fashion to that of Theorem~\ref{thm:SoftSparsityCovarianceBoundGeneral}, except that we now use the max-norm bound established in Lemma~\ref{lemma:SampleCovDeviationNMaxNorm} in place of Theorem~\ref{thm:MaxNormEffectiveDim}.
\end{proof}

\begin{lemma}[Cross-Covariance Estimation with Localized Sample Covariance and with Known Particle ---Operator-Norm Bound]\label{lem:SoftSparsityCrossCovarianceBoundGeneralConditional}
    Consider the set-up in Theorem~\ref{thm:SoftSparsityCrossCovarianceBoundGeneric} and assume additionally that $(X_n,Y_n)$ is known for some $n \in \{1,\dots, N\}$. For any $t \ge 1$, set 
    \begin{align*}
        \rho_N &\asymp         \frac{c(\normn{X_n}_\infty,\normn{\mu^X}_\infty,\normn{Y_n}_\infty,\normn{\mu^Y}_\infty)}{N}\\
    &+
   (\Sigma_{(1)}^X \lor \Sigma_{(1)}^Y)
        \inparen{ 
        \inparen{\frac{t}{N} \lor \sqrt{\frac{t}{N}}}
        \inparen{
        \sqrt{r_\infty(\Sigma^X)}
        \lor 
        \sqrt{r_\infty(\Sigma^Y)}
        }
        \lor 
        \sqrt{    \frac{ r_{\infty}(\Sigma^X)}{N}   } 
        \sqrt{    \frac{ r_{\infty}(\Sigma^Y)}{N}   }
        }.
    \end{align*}
    There exists positive universal constants $c_1,c_2$ such that, with probability at least $1-c_1 e^{-c_2 t}$,
    \begin{align*}
        \normn{\hatSigma^{XY}_{\rho_N} - \Sigma^{XY} } \lesssim R_{q_1}\rho_N^{1-q_1} \lor R_{q_2} \rho_N^{1-q_2}.
    \end{align*}
\end{lemma}
\begin{proof}
The proof follows in identical fashion to that of Theorem~\ref{thm:SoftSparsityCovarianceBoundGeneral}, except that we now use the max-norm bound established in Lemma~\ref{lemma:SampleCrossCovDeviationNMaxNorm} in place of Theorem~\ref{thm:MaxNormEffectiveDim}.
\end{proof}

\nc

\subsection{Proof of Main Results in Section \ref{sec:withlocalization}} \label{sec:localizationProofs}

\begin{proof}[Proof of Theorem~\ref{thm:EKI}]
    First, we may write  
    \begin{align}\label{eq:EKI1}
        \normn{ \upost_n - \upost_n^* }_2
        & = \norm{ \bigl(y-\mcG(\upr_n) -\eta_n \bigr) \bigl(\msP(\hatCpr^{\,up}, \hatCpr^{\,pp}) - \msP(\Cpr^{\,up}, \Cpr^{\,pp})\bigr)}_2 \nonumber \\
        &\le  \normn{y-\mcG(\upr_n) -\eta_n}_2 \normn{\msP(\hatCpr^{\,up}, \hatCpr^{\,pp}) - \msP(\Cpr^{\,up}, \Cpr^{\,pp})}_2.
    \end{align}
    For the second term in \eqref{eq:EKI1}, it follows by Lemma~\ref{lem:EKIopCtyBdd} that
    \begin{align*}
        \normn{\msP(\hatCpr^{\,up}, \hatCpr^{\,pp}) - \msP(\Cpr^{\,up}, \Cpr^{\,pp}))}_2
        &\le 
        \normn{\Gamma^{-1}} 
        \normn{\hatCpr^{\,up}-C^{\,up}} 
        +
        \normn{\Gamma^{-1}}^2 \normn{\Cpr^{\,up}}
        \normn{\hatCpr^{\,pp}-\Cpr^{\,pp}}.
    \end{align*}
    In order to control the two deviation terms, we write  $W_i \equiv [\upr_i^\top, \mcG^\top(\upr_i)]^\top$ for $1 
    \le i   \le N$. Further, let
    \begin{align*}
        \hatCpr^W = \frac{1}{N-1} \sum_{i=1}^N (W_i - \barW_N)(W_i - \barW_N)^\top,
        \qquad 
        \Cpr^W = \begin{bmatrix}
        \Cpr & \Cpr^{\, up}\\
        \Cpr^{\, pu} & \Cpr^{\, pp}
        \end{bmatrix}.
    \end{align*}
    with $\barW_N=[\hatm^\top, \overline{\mcG}^\top]^\top$ and $\overline{\mcG}$ the sample mean of $\{\mcG(u_n)\}_{n=1}^N.$
    Since $\upr \sim \mcN(m, \Cpr)$ and $\mcG$ is Lipschitz, by Gaussian concentration \cite[Theorem 5.2.2]{vershynin2018high} it holds that $\normn{\mcG(\upr) - \E[ \mcG(\upr)]}_{\psi_2} \le \normn{\mcG}_{\text{Lip}} \normn{\Cpr}^{1/2}$ and we can apply Lemma~\ref{lemma:SampleCovDeviationNOpNorm}. Letting $E_1$ be the event on which 
    \begin{align*}
        \normn{\hatCpr^{\,up} - C^{\,up}} \lor \normn{\hatCpr^{\,pp} - \Cpr^{\,pp}}
        &\le 
        \normn{\hatCpr^W - C^W} \\
        &\lesssim 
        \frac{c(\normn{W_n}, \normn{\E[W_n]})}{N}
        +
        \normn{C^W} 
        \biggl( \sqrt{\frac{r_2(C^W)}{N}}
        \lor \frac{r_2(C^W)}{N} 
        \lor \sqrt{\frac{t}{N}}
        \lor \frac{t}{N} 
        \biggr),
    \end{align*}
    then
    Lemma~\ref{lemma:SampleCovDeviationNOpNorm} ensures that
    $\P(E_1) \ge 1-c_1e^{-c_2t}$. It follows that on the event $E_1$, we also have 
    \begin{align*}
        \normn{\msP(\hatCpr^{\, up}, \hatCpr^{\, pp}) - \msP(\Cpr^{\, up}, \Cpr^{\,pp}))}_2
        \lesssim 
             \normn{\Gamma^{-1}} (1 \lor \normn{\Cpr^{\, up}})
        \normn{C^W} \biggl( \sqrt{\frac{r_2(C^W)}{N}}\lor \frac{r_2(C^W)}{N} 
        \lor \sqrt{\frac{t}{N}}
        \lor \frac{t}{N} 
        \biggr).
    \end{align*}
    The expression can be simplified by noting that since $C^W \succeq 0$, $\normn{C^W} \le \normn{C}+\normn{\Cpr^{\,pp}}$ and further since $\ttrace(C^W) = \ttrace(C) + \ttrace(\Cpr^{\,pp})$
    \begin{align*}
        &\normn{C^W} \inparen{\sqrt{\frac{r_2(C^W)}{N}}
        \lor \frac{r_2(C^W)}{N}
        \lor \sqrt{\frac{t}{N}}
        \lor \frac{t}{N} }
        \\
        &\lesssim 
        (\normn{C} \lor \normn{\Cpr^{\,pp}})
        \inparen{
        \sqrt{\frac{\ttrace(C) + \ttrace(\Cpr^{\,pp})}{N(\normn{C} \lor \normn{\Cpr^{\,pp}})} } 
        \lor 
        \frac{\ttrace(C) + \ttrace(\Cpr^{\,pp})}{N(\normn{C} \lor \normn{\Cpr^{\,pp}}) }
        \lor \sqrt{\frac{t}{N}}
        \lor \frac{t}{N} } 
        \\
        &\lesssim 
        (\normn{C} \lor \normn{\Cpr^{\,pp}})
        \inparen{
        \sqrt{\frac{r_2(C)}{N}}
        \lor 
        \frac{r_2(C)}{N}
        \lor 
        \sqrt{\frac{r_2(\Cpr^{\,pp})}{N}}
        \lor 
        \frac{r_2(\Cpr^{\,pp})}{N}
        \lor \sqrt{\frac{t}{N}}
        \lor \frac{t}{N} 
        },
    \end{align*}
    where the last inequality follows by similar reasoning to that used in the proof of Lemma~\ref{lem:ForecastNoiseCovarianceBound}.
\end{proof}

\begin{proof}[Proof of Theorem~\ref{thm:LEKI}]
    As in the proof of Theorem~\ref{thm:EKI}, we have that 
    \begin{align*}
        \normn{ \upost_n^{\rho} - \upost_n^* }_2
        \le  \normn{y-\mcG(\upr_n) -\eta_n}_2 
        \normn{\msP(\hatCpr^{\,up}_{\rho_N}, \hatCpr^{\,pp}_{\rho_N}) - 
        \msP(\Cpr^{\,up}, \Cpr^{\,pp})}_2.
    \end{align*}
    Further, by Lemma~\ref{lem:EKIopCtyBdd},
    \begin{align*}
        \normn{\msP(\hatCpr^{\,up}_{\rho_N}, \hatCpr^{\,pp}_{\rho_N}) - \msP(\Cpr^{\,up}, \Cpr^{\,pp})}_2
        &\le 
        (\normn{\Gamma^{-1}} \lor \normn{\Gamma^{-1}}^2) 
        (1 \lor \normn{\Cpr^{\, up}})
        (\normn{\hatCpr^{\,up}_{\rho_N}-C^{\,up}} 
        +
        \normn{\hatCpr^{\,pp}_{\rho_N}-\Cpr^{\,pp}}). 
    \end{align*}
     Let $E_1$ denote the event on which 
    \begin{align*}
        \normn{\hatCpr^{\,up}_{\rho_N}-C^{\,up}} 
        \lesssim 
        R_{q_1} \rho_{N, 1}^{1-q_1}
        \lor
        R_{q_2} \rho_{N, 2}^{1-q_2}
        .
    \end{align*}
    By Lemma~\ref{lem:SoftSparsityCrossCovarianceBoundGeneralConditional} 
    , $E_1$ has probability at least $1-c_1 e^{-c_2t}$. Let $E_2$ be the event on which
    \begin{align*}
        \normn{\hatCpr^{\,pp}_{\rho_N}-C^{\,pp}} 
        \lesssim R_{q_3} \rho_{N, 3}^{1-q_3}.
    \end{align*}
     By Lemma~\ref{lem:SoftSparsityCovarianceBoundGeneralConditional}, $E_2$ has probability at least $1-c_1 e^{-c_2t}$. Therefore, $E = E_1 \cap E_2$ has probability at least $1-c_1e^{-c_2 t}$, and on $E$ it holds that 
    \begin{equation*}
         \normn{\msP(\hatCpr^{\,up}_{\rho_N}, \hatCpr^{\,pp}_{\rho_N}) - \msP(\Cpr^{\,up}, \Cpr^{\,pp})}_2
        \lesssim
        (\normn{\Gamma^{-1}} \lor \normn{\Gamma^{-1}}^2) 
        (1 \lor \normn{\Cpr^{\, up}}) 
        (
        R_{q_1} \rho_{N, 1}^{1-q_1} 
        + R_{q_2} \rho_{N, 2}^{1-q_2}
        + R_{q_3} \rho_{N, 3}^{1-q_3}
        ). \qedhere
    \end{equation*}
\end{proof}
\nc

 \section{Proofs:  Section 4}\label{sec:AuxillaryResultsAppendix}
This appendix contains the proofs of the auxiliary results discussed in Section~\ref{sec:conclusions}.

\begin{lemma} [Kalman Gain Deviation with Localization] \label{thm:KalmanGainDeviationLoc}
   Let $\upr_1,\dots, \upr_N$ be $d$-dimensional i.i.d. sub-Gaussian random vectors with $\E [\upr_1] = m$ and $\E \bigl[ (\upr_1-m)(\upr_1-m)^\top \bigr] = \Cpr$. Assume further that $\Cpr \in \msU_d(q, R_q)$ for some $q \in [0,1)$ and $R_q>0$. For any $t\ge 1$, set
    \begin{align*}
        \rho_N \asymp 
        \Cpr_{(1)}
         \inparen{
         \sqrt{\frac{r_{\infty}(\Cpr)}{N} } 
         \lor \sqrt{\frac{t}{N}}
         \lor \frac{t}{N}
         \lor 
         \frac{t r_{\infty}(\Cpr)}{N}}
    \end{align*}
    and let $\hatCpr_{\rho_N}$ be the localized sample covariance estimator. There exists a positive universal constant $c$ such that, with probability at least $1-ce^{-t}$, 
    \begin{align*}
        \normn{\msK(\hatCpr_{\rho_N}) - \msK(\hatCpr)} 
        \lesssim
        \normn{A} \normn{\Gamma^{-1}}
        R_q (1+\normn{A}^2 \normn{\Gamma^{-1}}\normn{\Cpr}) \rho_N^{1-q}.
    \end{align*}
\end{lemma}

\begin{proof}
    By Lemma~\ref{lem:KalmanOpCtyBdd} and Theorem~\ref{thm:SoftSparistyCovarianceBound}, it follows immediately that
    \begin{align*}
        \normn{\msK(\hatCpr_{\rho_N}) - \msK(\hatCpr)} 
        &\le
        \normn{A} \normn{\Gamma^{-1}}
        \normn{\hatCpr_{\rho_N} - \Cpr}
        (1+\normn{A}^2 \normn{\Gamma^{-1}}\normn{\Cpr})\\
        &\lesssim \normn{A} \normn{\Gamma^{-1}}
        R_q \rho_N^{1-q}
        (1+\normn{A}^2 \normn{\Gamma^{-1}}\normn{\Cpr}). \qedhere
    \end{align*} 
\end{proof}

\begin{theorem} [Square Root Ensemble Kalman Covariance Deviation with Localization] \label{th:SRcovwithloc}
 Consider the localized SR ensemble Kalman update given by \eqref{eq:locSRupdate}, leading to an estimate $\hatCpost$ of the posterior covariance $\Cpost$ defined in \eqref{eq:KFmeancovariance}. Assume that $\Cpr \in \msU_{d}(q, R_q)$ for $q \in [0,1)$ and $R_q>0$. For any $t\ge 1$, set
    \begin{align*}
        \rho_N \asymp 
        \Cpr_{(1)}
         \inparen{
         \sqrt{\frac{r_{\infty}(\Cpr)}{N} } 
         \lor \sqrt{\frac{t}{N}}
         \lor \frac{t}{N}
         \lor 
         \frac{t r_{\infty}(\Cpr)}{N}}.
    \end{align*}
    There exists a positive universal constant $c$ such that, with probability at least $1-ce^{-t}$,
    \begin{align*}
        \normn{\hatCpost - \Cpost} 
        &\lesssim 
        R_q\rho_N^{1-q}
        \inparen{ 1 + 
        \normn{A}^2 \normn{\Gamma^{-1}}
        \inparen{2\normn{\Cpr} + R_q\rho_N^{1-q}}
        + \normn{A}^4 \normn{\Gamma^{-1}}^2 \normn{\Cpr}(\normn{\Cpr}+R_q\rho_N^{1-q}) }.
    \end{align*}
\end{theorem}

\begin{proof}
    For the localized SR update we have $\hatCpost = \msC(\hatCpr_{\rho_N})$. From Lemma~\ref{lem:CovarOpCtyBdd}, the continuity of $\msC$ implies that 
    \begin{align*}
        \normn{\msC(\hatCpr_{\rho_N}) - \msC(\Cpr)} 
        &\le \normn{\hatCpr_{\rho_N} - \Cpr} 
        \Bigl( 1 + 
        \normn{A}^2 \normn{\Gamma^{-1}} \bigl( \normn{\hatCpr_{\rho_N}} + \normn{\Cpr} \bigr)
        + \normn{A}^4 \normn{\Gamma^{-1}}^2\normn{\hatCpr_{\rho_N}} \normn{\Cpr} \Bigr).
    \end{align*}
    Let $E$ denote the event on which 
    \begin{align*}
        \normn{\hatCpr_{\rho_N} - \Cpr} \lesssim R_q\rho_N^{1-q}.
    \end{align*}
    By Theorem~\ref{thm:SoftSparsityCrossCovarianceBoundGeneric}, $E$ has probability at least $1-ce^{-t}$. It also holds on $E$ that 
    \begin{align*}
        \normn{\hatCpr_{\rho_N}} 
        \le 
        \normn{\Cpr} + \normn{\hatCpr_{\rho_N} - \Cpr} 
        \lesssim \normn{\Cpr} + R_q\rho_N^{1-q}.
    \end{align*}
    Therefore, it holds on $E$ that
    \begin{align*}
        \normn{\hatCpost - \Cpost} 
        &\lesssim 
        R_q\rho_N^{1-q}
        \inparen{ 1 + 
        \normn{A}^2 \normn{\Gamma^{-1}}
        \inparen{2\normn{\Cpr} + R_q\rho_N^{1-q}}
        + \normn{A}^4 \normn{\Gamma^{-1}}^2 \normn{\Cpr}(\normn{\Cpr}+ R_q\rho_N^{1-q}) }.\qedhere 
    \end{align*}
\end{proof}

\end{appendix}

\end{document}